\documentclass[10pt]{article}

\usepackage{times}
\usepackage[margin=1in]{geometry}
\usepackage{graphicx,authblk} 
\usepackage{amsmath,amssymb,amsthm,flushend,multicol,fullpage,adjustbox}
\usepackage{natbib}
\usepackage{pifont,color}
\usepackage{placeins}
\usepackage{subcaption}
\usepackage{algorithm}
\usepackage{algorithmic,bm,multirow,array}

\usepackage{enumitem,scalerel}
\usepackage{xcolor}
\usepackage[colorlinks = true,
            linkcolor = red,
            urlcolor  = blue,
            citecolor = blue,
            anchorcolor = blue]{hyperref}
\graphicspath{{figures/newer-figs/}}
\title{Accelerating Stochastic Gradient Descent For Least Squares Regression\footnote{This paper appeared in the proceedings of Conference on Learning Theory (COLT), 2018 held in Stockholm, Sweden.}}

\author[1]{Prateek Jain}
\author[2]{Sham M. Kakade}
\author[2]{Rahul Kidambi}
\author[1]{Praneeth Netrapalli}
\author[3]{Aaron Sidford}
\affil[1]{Microsoft Research, Bangalore, India,    \url{{prajain,praneeth}@microsoft.com}}
\affil[2]{University of Washington, Seattle, WA, USA,  \url{sham@cs.washington.edu},\ \url{rkidambi@uw.edu}}
\affil[3]{Stanford University, Palo Alto, CA, USA, \url{sidford@stanford.edu}.}
\date{}

\newtheorem{theorem}{Theorem}
\newtheorem{lemma}[theorem]{Lemma}
\newtheorem{corollary}[theorem]{Corollary}

\theoremstyle{definition}

\newcommand{\Det}[1]{\left|#1\right|}
\newcommand{\infbound}{R^2}
\newcommand{\defeq}{\stackrel{\mathrm{def}}{=}}

\newcommand{\lamminH}{\lambda_{\textrm{min}}(\Cov)}

\newcommand{\norm}[1]{\left\| #1 \right\|}

\newcommand{\twonorm}[1]{\norm{#1}_2}
\newcommand{\Hinvnorm}[1]{\left\| #1 \right\|_{\Covinv}}
\renewcommand{\vec}[1]{\mathbf{#1}}
\newcommand{\mat}[1]{\mathbf{#1}}
\newcommand{\Sig}{\mat{\Sigma}}
\newcommand{\R}{\mathbb{R}}

\renewcommand{\a}{\vec{a}}
\newcommand{\x}{\vec{x}}
\renewcommand{\v}{\vec{v}}
\newcommand{\xtilde}{\widetilde{\x}}
\newcommand{\y}{\vec{y}}
\newcommand{\z}{\vec{z}}
\newcommand{\xplus}{\x^+}
\newcommand{\vplus}{\v^+}
\newcommand{\yplus}{\y^+}
\newcommand{\ai}[1][i]{\a_{#1}}
\newcommand{\xt}[1][j]{\x_{#1}}
\newcommand{\yt}[1][j]{\y_{#1}}
\newcommand{\vt}[1][t]{\v_{#1}}
\newcommand{\zt}[1][t]{\z_{#1}}
\renewcommand{\u}{\vec{u}}
\renewcommand{\b}{b}
\newcommand{\bi}[1][i]{b_{#1}}
\newcommand{\xs}{\x^*}
\newcommand{\vv}{\vec{v}}
\newcommand{\xhat}{\widehat{\x}}
\newcommand{\iprod}[2]{\left\langle #1, #2 \right\rangle}
\renewcommand{\ni}[1][i]{\epsilon_{#1}}
\newcommand{\n}{\epsilon}
\newcommand{\E}[1]{\mathbb{E}\left[#1\right]}
\newcommand{\Eover}[2]{\mathbb{E}_{#1}\left[#2\right]}
\newcommand{\D}{\mathcal{D}}
\newcommand{\B}{\mathcal{B}}
\newcommand{\T}{^{\top}}
\newcommand{\Cov}{\mat{H}}

\newcommand{\Hhat}{\widehat{\Cov}}
\newcommand{\g}{q}
\newcommand{\e}{e}
\newcommand{\f}{f}
\newcommand{\eplus}{\e^+}
\newcommand{\fplus}{\f^+}
\newcommand{\Covinv}{\Cov^{-1}}
\newcommand{\praneeth}[1]{{\color{red} praneeth: #1}}

\newcommand{\rahul}[1]{{\color{magenta} rahul: #1}}

\newcommand{\distr}{(\a,\b)\sim \D}
\newcommand{\M}{\mathcal{M}}
\renewcommand{\S}{\mat{S}}
\newcommand{\Id}{\mat{I}}
\newcommand{\eye}{\Id}
\newcommand{\zero}{0}
\newcommand{\eqdef}{\stackrel{\textrm{def}}{=}}
\newcommand{\cnS}{\widetilde{\kappa}}
\newcommand{\cnH}{{\kappa}}
\newcommand{\order}[1]{\mathcal{O}\left(#1\right)}

\newcommand{\thetat}[1][j]{\boldsymbol{\theta}_{#1}}
\newcommand{\A}{\mat{A}}

\newcommand{\AL}{\mathcal{A}_{\mathcal{L}}}
\newcommand{\AR}{\mathcal{A}_{\mathcal{R}}}
\newcommand{\Ahat}{\widehat{\A}}
\newcommand{\Ahatj}[1][j]{\Ahat_{#1}}
\newcommand{\zetat}[1][j]{\boldsymbol{\zeta}_{#1}}
\newcommand{\Sighat}{\widehat{\boldsymbol{\Sigma}}}
\newcommand{\thetavb}[1][t,n]{{\bar{\boldsymbol{\theta}}}_{#1}}
\newcommand{\inv}[1]{{#1}^{-1}}
\newcommand{\G}{\mat{G}}
\newcommand{\Gtilde}{\widetilde{\mat{G}}}
\newcommand{\U}{\mat{U}}
\newcommand{\V}{\mat{V}}
\newcommand{\Q}{\mat{Q}}
\newcommand{\Z}{\mat{Z}}
\newcommand{\Uhat}{\widehat{\U}}
\newcommand{\Rc}{\mathcal{R}}
\newcommand{\C}{\mat{C}}
\newcommand{\singmin}[1]{\sigma_{\textrm{min}}\left(#1\right)}
\newcommand{\singmax}[1]{\sigma_{\textrm{max}}\left(#1\right)}
\newcommand{\frob}[1]{\norm{#1}_F}
\newcommand{\abs}[1]{\left|{#1}\right|}

\newcommand{\Hinv}{\inv{\H}}
\renewcommand{\H}{\mat{H}}
\newcommand{\phiv}{\boldsymbol{\Phi}}
\newcommand{\phivi}{\boldsymbol{\Phi}_{\infty}}
\newcommand{\phivih}{\boldsymbol{\Phi}_{\infty}^{1/2}}

\newcommand{\phivb}{\boldsymbol{\bar{\Phi}}}
\newcommand{\tensor}[1]{\mathcal{#1}}
\newcommand{\BT}{\tensor{B}}
\newcommand{\DT}{\tensor{D}}
\newcommand{\RT}{\tensor{R}}
\newcommand{\ST}{\tensor{S}}
\newcommand{\cnHh}{\widetilde{\kappa}}
\newcommand{\eyeT}{\tensor{I}}
\newcommand\adanorm[1]{\left\lVert#1\right\rVert}
\newcommand{\cone}{c_1}
\newcommand{\ctwo}{c_2}
\newcommand{\cthree}{c_3}
\newcommand{\cfour}{c_4}
\newcommand{\trace}[1]{\textrm{Tr}\left(#1\right)}
\newcommand{\thetav}{\boldsymbol{\theta}}
\newcommand{\av}{\vec{a}}
\newcommand{\zetav}{\boldsymbol{\zeta}}
\newcommand{\etav}{\boldsymbol{\eta}}
\newcommand{\Ah}{\widehat{\mat{A}}}
\newcommand{\Sigh}{\mat{\widehat{\Sigma}}}
\newcommand{\Vh}{\hat{\mat{V}}}
\newcommand{\HL}{\tensor{H_L}}
\newcommand{\HR}{\tensor{H_R}}
\newcommand{\Y}{\tensor{E}}
\newcommand{\PM}{\Gtilde}
\newcommand{\UC}{C}

\DeclareMathOperator*{\Bigcdot}{\scalerel*{\cdot}{\bigodot}}

\allowdisplaybreaks

\begin{document} 
\maketitle

\begin{abstract}

There is widespread sentiment that fast gradient methods (\emph{e.g.}
Nesterov's acceleration, conjugate gradient, heavy ball) are not
effective for stochastic optimization due to their
instability and error accumulation. Numerous works have attempted to
quantify these instabilities in the face of either statistical or
non-statistical
errors~\citep{Paige71,Proakis74,Polyak87,Greenbaum89,DevolderGN14}.
This work considers these issues for the case of
stochastic approximation for the least squares regression problem, and
our main result refutes this conventional wisdom by showing that
acceleration can be made robust to statistical errors.  In
particular, this work introduces an accelerated stochastic gradient
method that provably achieves the minimax optimal statistical risk
faster than stochastic gradient descent.  Critical to the analysis is
a sharp characterization of accelerated stochastic gradient descent as
a stochastic process. We hope this characterization gives insights
towards the broader question of designing simple and effective
accelerated stochastic methods for general convex and non-convex
optimization problems.

\end{abstract} 
\section{Introduction}

Stochastic gradient descent (SGD) is the workhorse algorithm for
optimization in machine learning and stochastic approximation
problems; improving its runtime dependencies is a central issue in
large scale stochastic optimization that often arise in machine 
learning problems at scale~\citep{BottouB07}, where one can only resort to streaming algorithms. 


This work examines these broader runtime issues for the special case of stochastic approximation in the
following least squares regression problem:
\begin{align}
\label{eq:objFun}
\min_{\x \in \R^d} P(\x), \, \, \, 
\text{where, }P(\x)\defeq \tfrac{1}{2} \cdot\Eover{\distr}{(b-\iprod{\x}{\a})^2},
\end{align}

\noindent where we have access to a {\em stochastic first order oracle}, which, when provided with $\x$ as an input, returns a noisy unbiased stochastic gradient using a tuple $(\a,\b)$ sampled from $\D(\R^d\times \R)$, with $d$ being the dimension of the problem. A query to the stochastic first-order oracle at $\x$ produces: 
\begin{align}
\label{eq:stochFirstOrderOracle}
\widehat{\nabla}P(\x) = \ -(b-\iprod{\a}{\x})\cdot\a.
\end{align}
Note $\E{\widehat{\nabla}P(\x)}=\nabla P(\x)$ (i.e. eq\eqref{eq:stochFirstOrderOracle} is an unbiased estimate). Note that nearly all practical stochastic algorithms use sampled gradients of the specific form as in equation~\ref{eq:stochFirstOrderOracle}. We discuss differences to the more general stochastic first order oracle~\citep{NemirovskyY83} in section~\ref{sec:related}.

\begin{table*}[t]
	\begin{center}
  \begin{adjustbox}{max width=\textwidth}
  		\begin{tabular}{| c | c | c | c |}
			\hline
			Algorithm & Final error & Runtime & Memory\\ 
		\hline
			\begin{tabular}{@{}c@{}} Accelerated SVRG \\ \citep{Zhu16} \end{tabular}&  $\mathcal{O}\left(\frac{\sigma^2 d}{n}\right)$ & $({n+\sqrt{n\cnH}})d\log\bigg({\frac{P(\xt[0])-P(\xt[*])}{(\sigma^2d/n)}}\bigg)$ &$nd$\\
			\hline
			\begin{tabular}{@{}c@{}} Streaming SVRG \\ \citep{FrostigGKS15} \\ 	Iterate Averaged SGD \\ \citep{JainKKNS16} \end{tabular} & $\mathcal{O}\left(\exp\left(\frac{-n}{\cnH}\right)\cdot\big(P(\xt[0])-P(\xt[*])\big) + \frac{\sigma^2 d}{n}\right)$ & ${nd}$ &$\mathcal{O}(d)$\\
\hline
			\begin{tabular}{@{}c@{}} Accelerated Stochastic Gradient Descent \\ (this paper) \end{tabular}& $\mathcal{O}^*\left(\exp\left(\frac{-n}{\sqrt{\cnH\cnS}}\right) \big(P(\xt[0])-P(\xt[*])\big)\right) + \mathcal{O}\left(\frac{\sigma^2 d}{n}\right)$ & ${nd}$ &$\mathcal{O}(d)$\\
			\hline
		\end{tabular} 
		\end{adjustbox}
				\caption{Comparison of this work to the best known non-asymptotic results~\citep{FrostigGKS15,JainKKNS16} for the least squares stochastic approximation problem.\ Here, $d,n$ are the problem dimension, number of samples; $\cnH$, $\cnS$ denote the condition number and statistical condition number of the distribution; $\sigma^2$, $P(\xt[0])-P(\xt[*])$ denote the noise level and initial excess risk, $\mathcal{O}^*$ hides lower order terms in $d,\cnH,\cnS$ (see section~\ref{sec:prob} for definitions and a proof for $\cnS \leq \cnH$). Note that Accelerated SVRG~\citep{Zhu16} is not a streaming algorithm. 
                }
		\label{tab:comp}
	\end{center}
\end{table*}
Let $\xs \eqdef \arg\min_\x P(\x)$ be a population risk minimizer.  Given any estimation procedure which returns $\xhat_n $ using $n$ samples, define the {\em excess risk} (which we also refer to as the \emph{generalization error} or the \emph{error}) of $\xhat_n$ as $\E{P(\xhat_n)}-P(\xs)$.
Now, equipped a stochastic first-order oracle (equation~\eqref{eq:stochFirstOrderOracle}), our goal is to provide a computationally efficient (and streaming) estimation method whose excess risk is comparable to the optimal statistical minimax rate.  

In the limit of large $n$, this minimax rate is achieved by the {\em empirical risk minimizer} (ERM), which is defined as follows. Given $n$ i.i.d. samples $\ST_n=\{(\a_i,\b_i)\}_{i=1}^n$ drawn from $\D$, define 
\begin{align*} 
	\xhat_n^{\textrm{ERM}} \eqdef \arg\min_\x P_n(\x) , \textrm{ where } P_n(\x)\eqdef\frac{1}{n}\sum_{i=1}^{n} \tfrac{1}{2}\left(\b_i-\a_i\T \x \right)^2,
\end{align*}

\noindent where $\xhat_n^{\textrm{ERM}}$ denotes the ERM over the samples $\ST_n$. For the case of additive noise models (i.e. where $b=\a\T\xs+\n$, with $\n$ being independent of $\a$), the minimax estimation rate is $d\sigma^2/n$~\citep{KushnerClark,PolyakJ92,lehmann1998theory,Vaart00}, i.e.:
\begin{align} \label{eq:ERMVarianceAdditive}
\lim_{n\to\infty}\frac{\mathbb{E}_{\ST_n}[P(\xhat_n^{\textrm{ERM}})]-P(\xs)}{d\sigma^2/n} &= 1,
\end{align}

\noindent where $\sigma^2=\E{\n^2}$ is the variance of the additive noise and the expectation is over the samples $\ST_n$ drawn from $\D$.  The seminal works of~\cite{Ruppert88,PolyakJ92} proved that a certain averaged stochastic gradient method enjoys this minimax rate, in the limit.  The question we seek to address is: how fast (in a non-asymptotic sense) can we achieve the minimax rate of $d\sigma^2/n$?
\subsection{Review: Acceleration with Exact Gradients}\label{sec:background}
Let us review results in convex optimization in the exact first-order oracle model. Running $t-$steps of gradient descent~\citep{cauchy1847} with an exact first-order oracle yields the
following guarantee:
\begin{align*}
P(\x_t)-P(\xs)\leq \exp\big(-t/\cnH_o\big)\cdot\big(P(\x_0)-P(\xs)\big),
\end{align*}

\noindent where $\x_0$ is the starting iterate, $\cnH_o=\lambda_{\max}(\H)/\lambda_{\min}(\H)$ is the condition number of $P(.)$, where, $\lambda_{\max}(\H),\lambda_{\min}(\H)$ are the largest and smallest eigenvalue of the hessian $\H=\nabla^2P(\x)=\E{\a\a\T}$. Thus gradient descent requires $\mathcal{O}(\cnH_o)$ oracle calls to solve the problem to a given target accuracy, which is sub-optimal amongst the class of methods with access to an exact first-order oracle~\citep{Nesterov04}. This sub-optimality can be addressed through Nesterov's Accelerated Gradient Descent~\citep{Nesterov83}, which when run for t-steps, yields the following guarantee:
\begin{align*}
P(\x_t)-P(\xs)\leq \exp\big(-t/\sqrt{\cnH_o}\big)\cdot\big(P(\x_0)-P(\xs)\big),
\end{align*}

\noindent which implies that $\mathcal{O}(\sqrt{\cnH_o})$ oracle calls are sufficient to achieve a given target accuracy. This matches the oracle lower bounds~\citep{Nesterov04} that state that $\Theta(\sqrt{\cnH_o})$ calls to the exact first order oracle are necessary to achieve a given target accuracy. The conjugate gradient method~\citep{HestenesS52} and heavy ball method~\citep{Polyak64} are also known to obtain this convergence rate for solving a system of linear equations and for quadratic functions. These methods are termed fast gradient methods owing to the improvements offered by these methods over Gradient Descent.
\begin{figure}[t!]
	\begin{subfigure}{0.49\textwidth}
		\includegraphics[width=\linewidth]{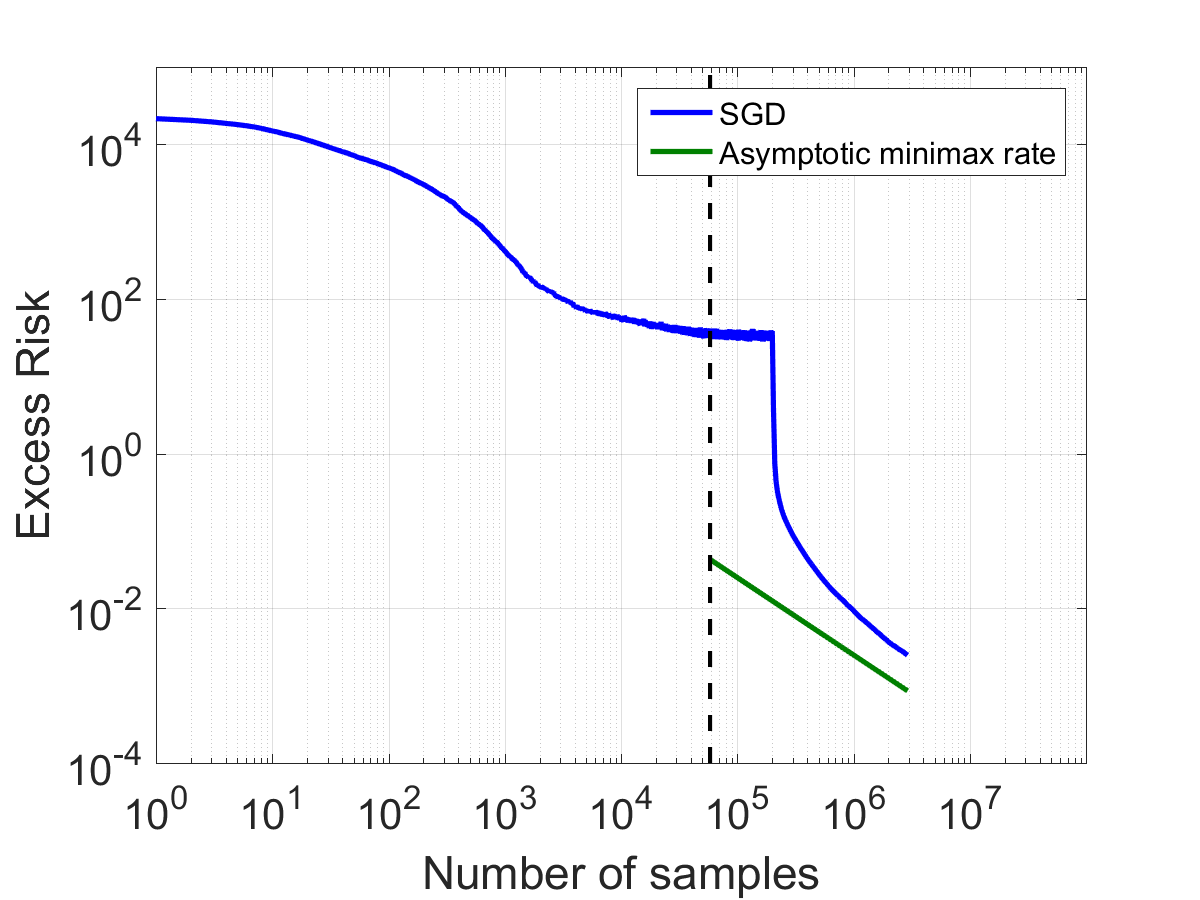}
		\caption{Discrete distribution}
	\end{subfigure}
	\begin{subfigure}{0.49\textwidth}
		\includegraphics[width=\linewidth]{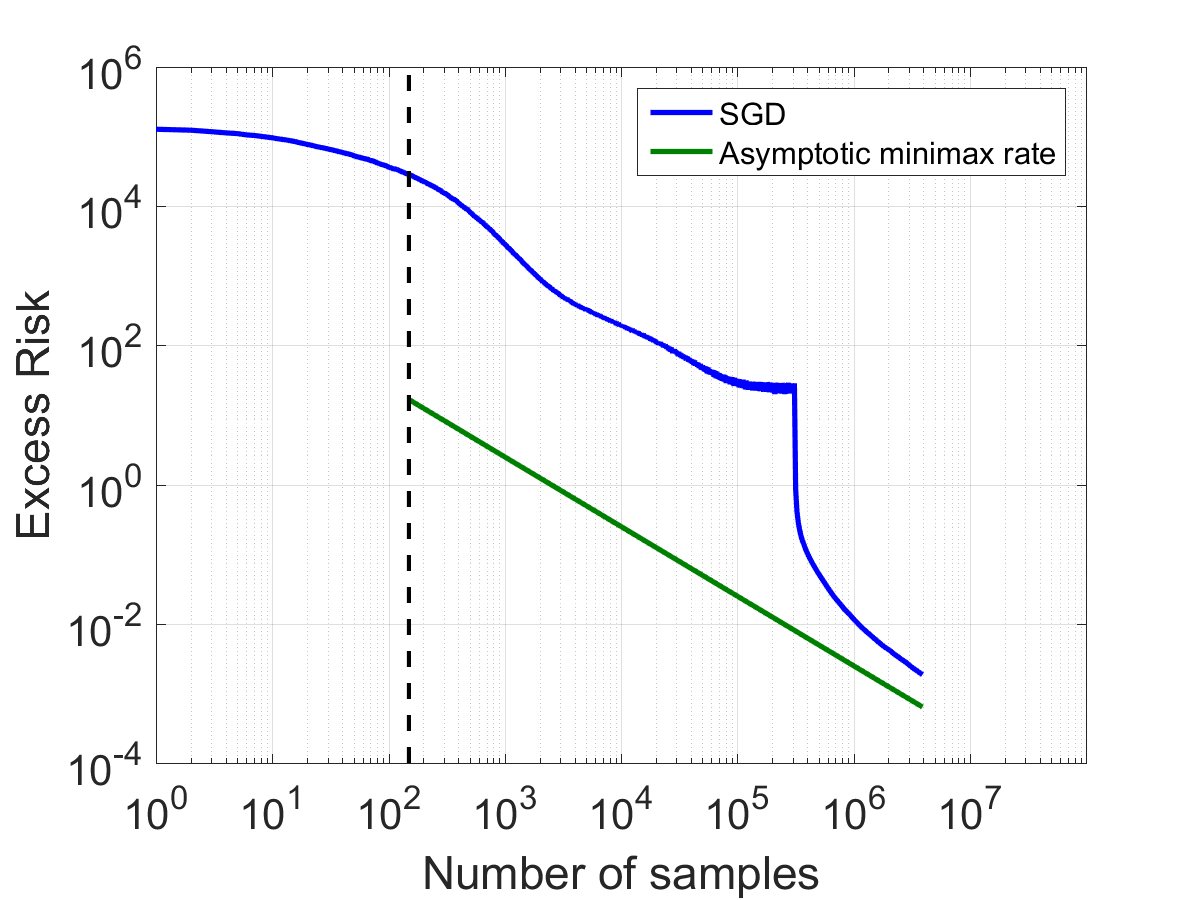}
		\caption{Gaussian distribution} 
	\end{subfigure}
	\caption{Plot of error vs number of samples for averaged
		SGD and the minimax risk for the discrete and
		Gaussian distributions with $d=50$, $\cnH\approx 10^5$ (see section~\ref{sec:exp} for
		details on the distribution). The kink in the SGD curve
		represents when the tail-averaging phase
		begins~\citep{JainKKNS16}; this point is chosen appropriately. 
		The green curves show the asymptotically optimal minimax rate of $d
		\sigma^2/n$. The vertical dashed line shows the sample size at which
		the empirical covariance, $\frac{1}{n}\sum_{i=1}^n \a_i\a_i\T$,
		becomes full rank, which is shown at $\frac{1}{\min_i p_{i}}$ in the
		discrete case and $d$ in the Gaussian case. With fewer samples than
		this (i.e. before the dashed line), it is information theoretically not possible to guarantee
		non-trivial risk (without further assumptions). For the Gaussian case,
		note how the behavior of SGD is far from the dotted line; it is this
		behavior that one might hope to improve upon.  See the text for a discussion.
	}  
	\label{fig:exp} 
\end{figure}
This paper seeks to address the question: ``Can we  accelerate
stochastic approximation in a manner similar to what has been achieved
with the exact first order oracle model?''
\subsection{A thought experiment: Is Accelerating Stochastic Approximation possible?}\label{sec:exp}
Let us recollect known results in stochastic approximation for the least squares regression problem (in equation~\ref{eq:objFun}). Running $n$-steps of tail-averaged SGD~\citep{JainKKNS16} (or, streaming SVRG~\citep{FrostigGKS15}\footnote{Streaming SVRG does not function in the stochastic first order oracle model~\citep{FrostigGKS15}}) provides an output $\xhat_n$ that satisfies the following excess risk bound:
\begin{align} \label{eq:sgd_rate}
\E{P(\xhat_n)}-P(\xs)\leq \exp(-n/\cnH) \cdot \big(P(\x_0)-P(\xs)\big) + 2\sigma^2d/n,
\end{align}

\noindent where $\cnH$ is the condition number
of the distribution, which can be upper bounded as $L/\lamminH$,
assuming that $\|\a\|\leq L$ with probability one (refer to
section~\ref{sec:prob} for a precise definition of $\cnH$).  Under appropriate
assumptions, these are the best known rates under the stochastic first
order oracle model (see section~\ref{sec:related} for further discussion).
  A natural implication of the bound implied by averaged SGD is that with $\widetilde{\mathcal{O}}(\cnH)$ oracle
calls~\citep{JainKKNS16} (where, $\widetilde{\mathcal{O}}(\cdot)$ hides $\log$ factors in $d,\cnH$), the excess risk attains (up to
constants) the (asymptotic) minimax statistical rate. Note that the excess
risk bounds in stochastic approximation consist of two terms:
(a) {\em bias}: which represents the dependence of the generalization
error on the initial excess risk $P(\x_0)-P(\xs)$, and (b) the {\em
  variance:} which represents the dependence of the generalization
error on the noise level $\sigma^2$ in the problem.

A precise question regarding accelerating stochastic
approximation is: ``is it possible to improve the rate of decay of the
bias term, while retaining (up to constants) the statistical minimax
rate?'' The key technical challenge in answering this question is in
sharply characterizing the error accumulation of fast gradient methods
in the stochastic approximation setting. Common folklore and prior
work suggest otherwise: several efforts have attempted to 
quantify instabilities in the face of statistical or
non-statistical
errors~\citep{Paige71,Proakis74,Polyak87,Greenbaum89,RoyS90,SharmaSB98,dAspremont08,DevolderGN13,DevolderGN14,YuanYS16}.
Refer to section~\ref{sec:related} for a discussion on robustness of acceleration to error accumulation.
\begin{figure}[t]
	\begin{subfigure}{0.49\textwidth}
		\includegraphics[width=\linewidth]{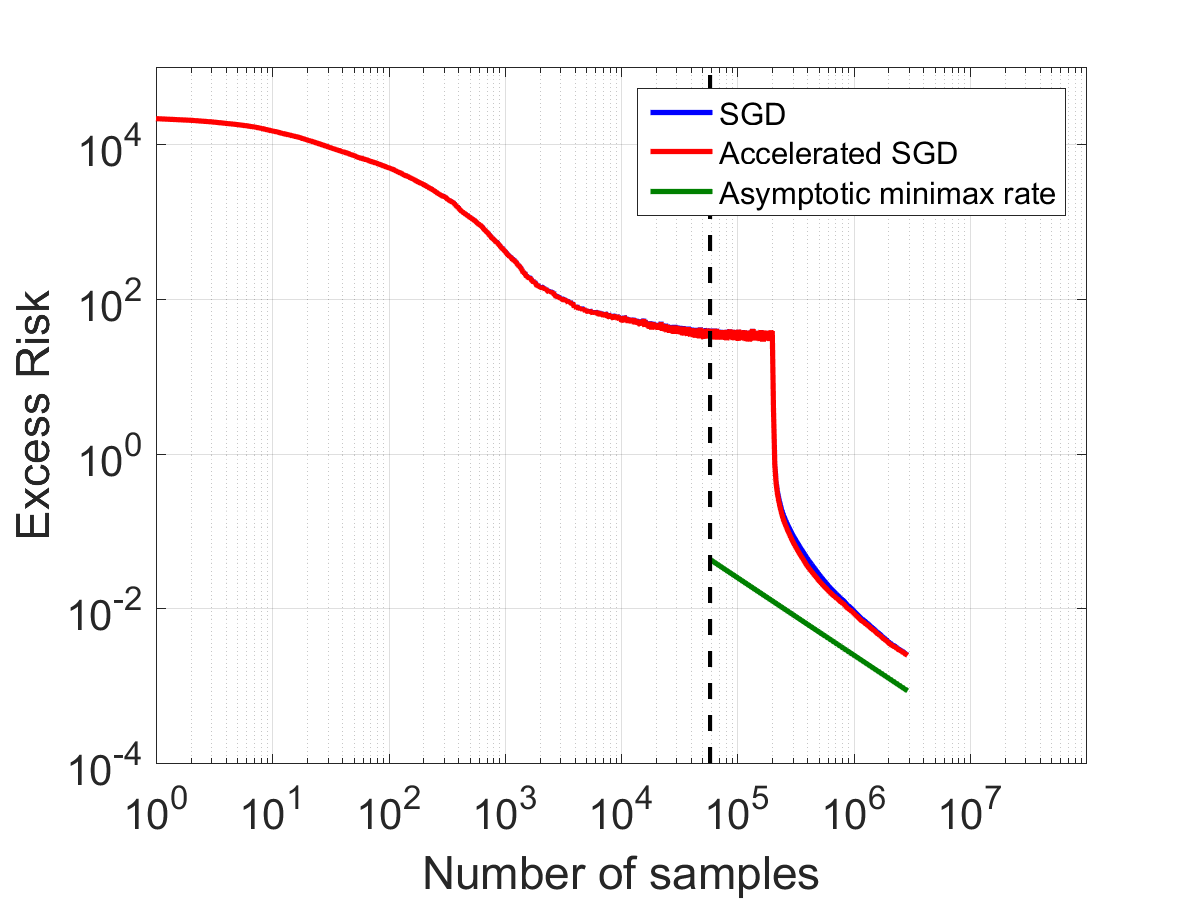}
		\caption{Discrete distribution}
	\end{subfigure}
	\begin{subfigure}{0.49\textwidth}
		\includegraphics[width=\linewidth]{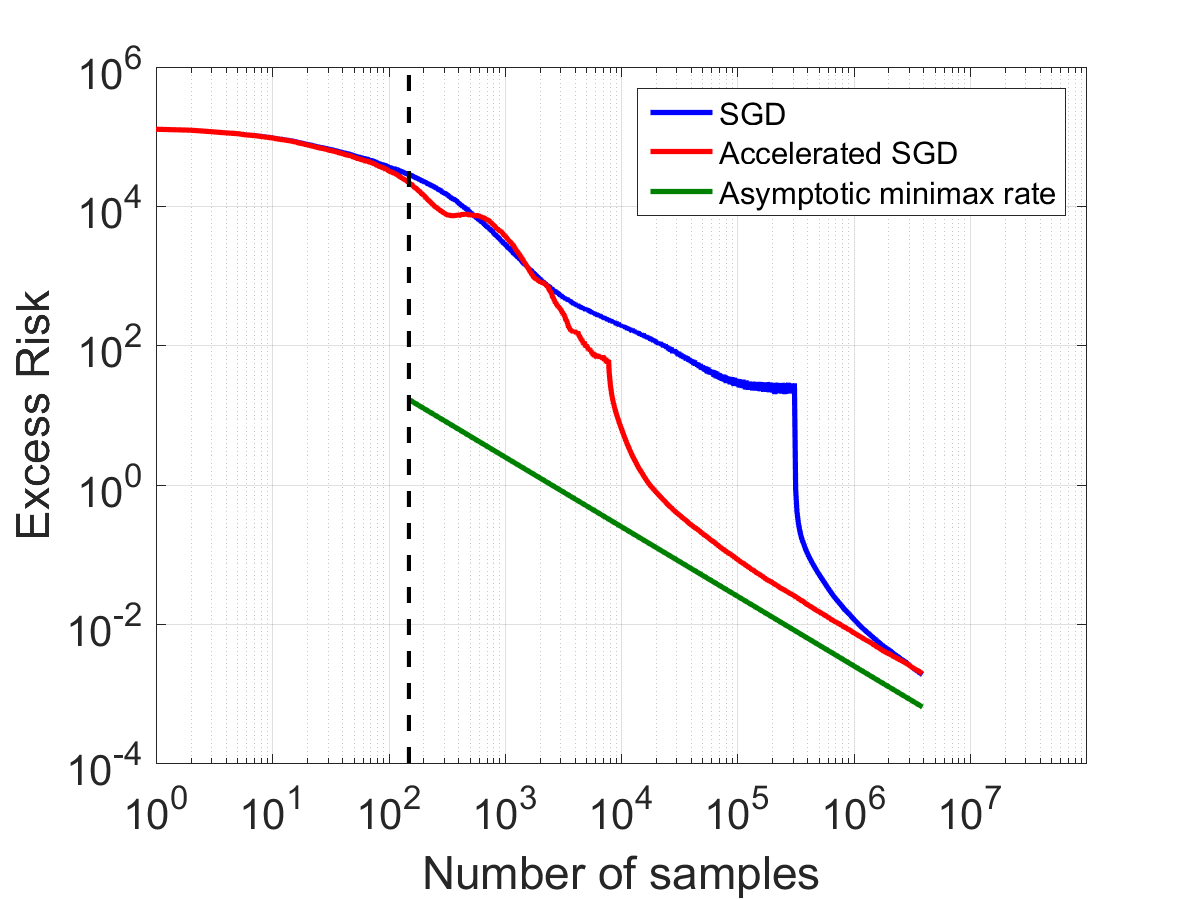}
		\caption{Gaussian distribution} 
	\end{subfigure}
	\caption{Plot of total error vs number of samples for averaged
		SGD, (this paper's) accelerated SGD method and the minimax risk for the discrete
		and Gaussian distributions with $d=50,\cnH\approx 10^5$ (see section~\ref{sec:exp} for
		details on the distribution). For the discrete case,
		accelerated SGD degenerates to SGD, which nearly matches the
		minimax risk (when it becomes well defined). For the
		Gaussian case, accelerated SGD significantly improves upon
		SGD. }\label{fig:res}
\end{figure}
Optimistically, as suggested by the gains enjoyed by accelerated methods in the exact first order oracle model, we may hope to replace the $\widetilde{\mathcal{O}}(\cnH)$ oracle calls achieved by averaged SGD to $\widetilde{\mathcal{O}}(\sqrt{\cnH})$.  We now provide a counter example, showing that such an improvement is not possible. Consider a (discrete) distribution $\D$ where the input $\a$ is the $i^{\textrm{th}}$ standard basis vector with probability $p_i$, $\forall\ i=1,2,...,d$. The covariance of $\a$ in this case is a diagonal matrix with diagonal entries $p_i$. The condition number of this distribution is $\cnH = \frac{1}{\min_i p_{i}}$. In this case, it is impossible to
make non-trivial reduction in error by observing fewer than $\cnH$ samples, since with constant probability, we would not have seen the vector corresponding to the smallest probability. 

On the other hand, consider a case where the distribution $\D$ is a
Gaussian with a large condition number $\cnH$. Matrix concentration
informs us that (with high probability and irrespective of how large $\cnH$ is) after observing
$n=\mathcal{O}(d)$ samples, the empirical covariance matrix will be a spectral approximation to the true covariance
matrix, i.e. for some constant $c>1$,
$\Cov/c \preceq \frac{1}{n}\sum_{i=1}^n \a_i\a_i\T \preceq c \Cov$.
Here, we may hope to achieve a faster convergence rate, as information
theoretically $\mathcal{O}(d)$ samples suffice to obtain a non-trivial
statistical estimate (see~\cite{HsuKZ14} for further discussion).

Figure~\ref{fig:exp} shows the behavior of SGD in these cases;
both are synthetic examples in $50-$dimensions, with a condition
number $\cnH\approx 10^5$ and noise level $\sigma^2=100$. See the figure caption for more details. 

These examples suggest that if acceleration is indeed possible, then the degree of 
improvement (say, over averaged SGD) must depend on distributional 
quantities that go beyond the condition number $\kappa$.
A natural conjecture is that this improvement must depend on 
the number of samples required to spectrally approximate 
the covariance matrix of the distribution; below this sample size it is 
not possible to obtain any non-trivial statistical estimate due 
to information theoretic reasons. This sample size is quantified by a
notion which we refer to as the {\em statistical condition number} $\cnS$ (see
section~\ref{sec:prob} for a precise definition and for further
discussion about $\cnS$). As we will see in section~\ref{sec:prob}, we have $\cnS\leq\cnH$, $\cnS$ is affine invariant, unlike $\cnH$ (i.e. $\cnS$ is invariant to linear transformations over $\a$).
\begin{figure}[t]
	\begin{subfigure}{0.49\textwidth}
		\includegraphics[width=\linewidth]{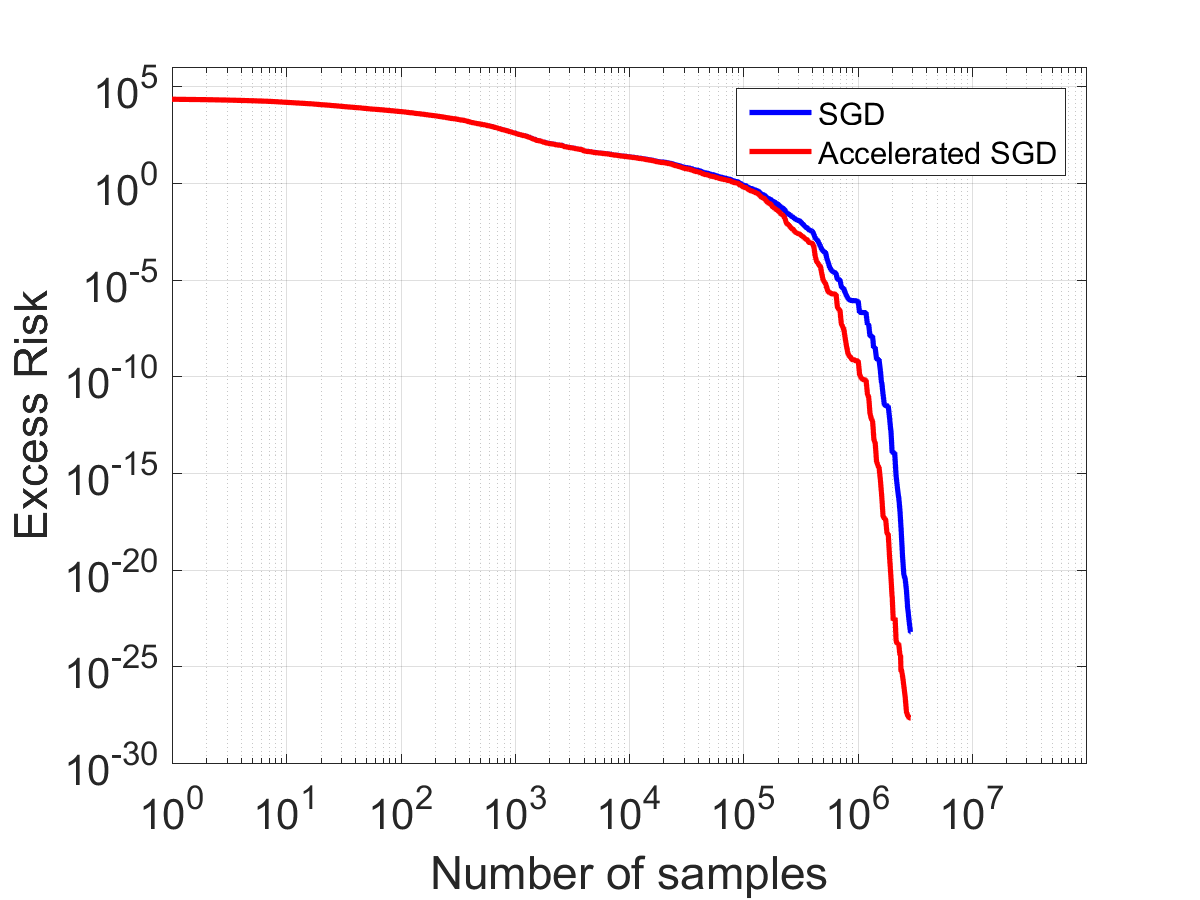}
		\caption{Discrete distribution}
	\end{subfigure}
	\begin{subfigure}{0.49\textwidth}
		\includegraphics[width=\linewidth]{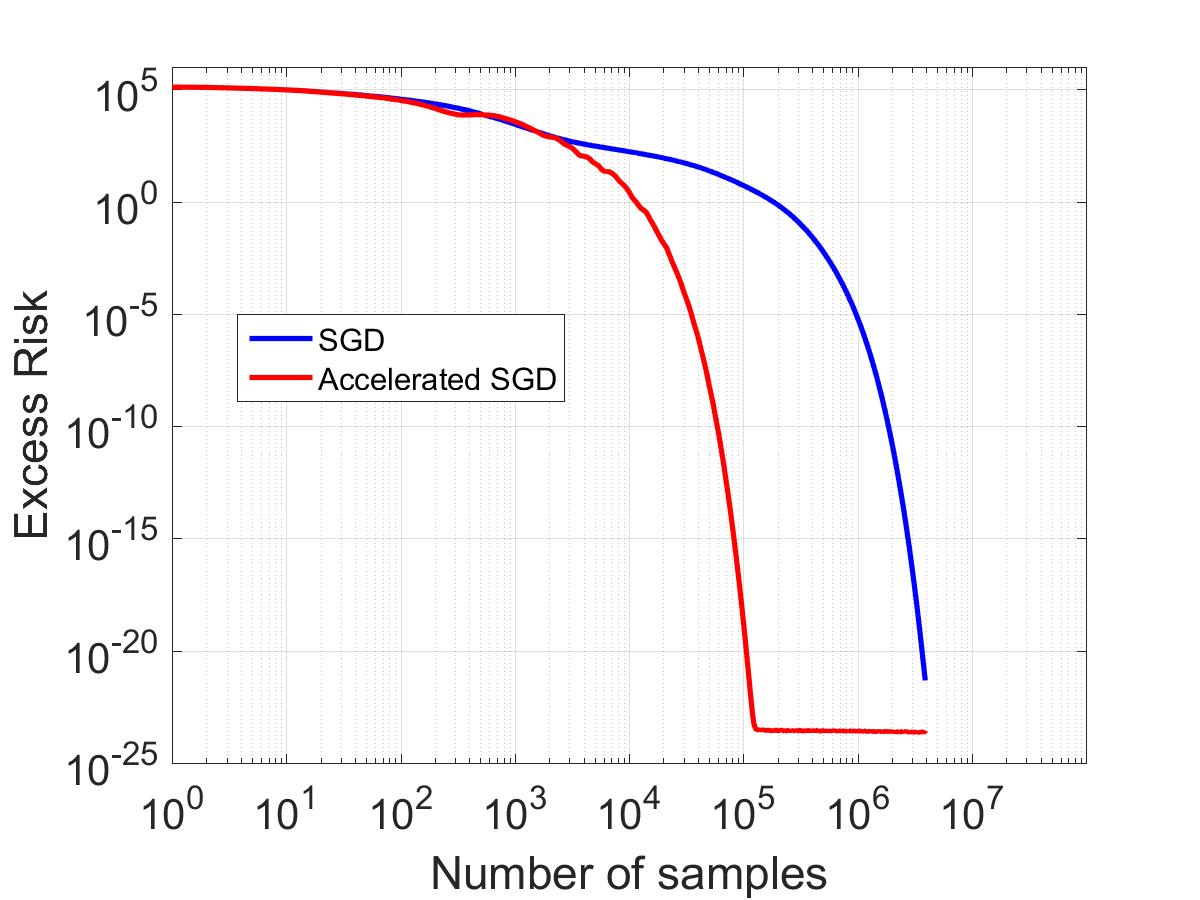}
		\caption{Gaussian distribution} 
	\end{subfigure}
	\caption{Comparison of averaged SGD with this paper's accelerated SGD
		in the absence of noise ($\sigma^2=0$) for the Gaussian and Discrete
		distribution with $d=50,\cnH\approx 10^5$. Acceleration yields
		substantial gains over averaged SGD for the Gaussian case, while
		degenerating to SGD's behavior for the discrete case.
		See section~\ref{sec:exp} for discussion.}
	\label{fig:bias}
\end{figure}
\subsection{Contributions}\label{sec:res}
This paper introduces an accelerated stochastic gradient descent
scheme, which can be viewed as a stochastic variant of Nesterov's accelerated gradient method~\citep{Nesterov12}. As pointed out in Section~\ref{sec:exp}, the excess risk of this algorithm can be decomposed into two parts namely, \emph{bias} and \emph{variance}. For the stochastic approximation problem of least squares regression, this paper establishes bias contraction at a geometric rate of $\mathcal{O}(1/\sqrt{\cnH\cnS})$, improving over prior results~\citep{FrostigGKS15,JainKKNS16},which prove a geometric rate of $\mathcal{O}(1/\cnH)$, while retaining statistical minimax rates (up to constants) for the variance. Here $\cnH$ is the condition number and $\cnS$ is the statistical condition number of the distribution, and a rate of $\mathcal{O}(1/\sqrt{\cnH\cnS})$ is an improvement over $\mathcal{O}(1/\cnH)$ since $\cnS \leq \cnH$ (see Section~\ref{sec:prob} for definitions and a short proof of $\cnS \leq \cnH$).

See Table~\ref{tab:comp} for a theoretical comparison. Figure~\ref{fig:res} provides an empirical comparison of the proposed (tail-averaged) accelerated algorithm to (tail-averaged) SGD~\citep{JainKKNS16} on our two running examples. Our result gives improvement over SGD even in the noiseless (i.e. realizable) case where $\sigma=0$; this case is equivalent to the setting where we have a distribution over a (possibly infinite) set of consistent linear equations. See Figure~\ref{fig:bias} for a comparison
on the case where $\sigma=0$.

On a more technical note, this paper introduces two new techniques in
order to analyze the proposed accelerated stochastic gradient method:
(a) the paper introduces a new potential function in order to
show faster rates of decaying the bias, and (b) the paper
provides a sharp understanding of the behavior of the
proposed accelerated stochastic gradient descent updates as a
stochastic process and utilizes this in providing a near-exact
estimate of the covariance of its iterates. This
viewpoint is critical in order to prove that the algorithm
achieves the statistical minimax rate. 

We use the operator viewpoint for analyzing stochastic
gradient methods, introduced in~\cite{DefossezB15}. This
viewpoint was also used in~\cite{DieuleveutB15,JainKKNS16}.

\subsection{Related Work}
\label{sec:related}
\paragraph{Non-asymptotic Stochastic Approximation:} Stochastic gradient descent (SGD) and its variants are by far the most
widely studied algorithms for the stochastic approximation problem.
While initial works~\citep{RobbinsM51} considered the final iterate of
SGD, later works~\citep{Ruppert88,PolyakJ92} demonstrated that averaged
SGD obtains statistically optimal estimation rates. Several 
works provide non-asymptotic analyses for averaged
SGD and variants~\citep{BachM11,Bach14,FrostigGKS15} for various
stochastic approximation problems. For stochastic approximation with
least squares regression~\citet{BachM13,DefossezB15,NeedellSW16,FrostigGKS15,JainKKNS16}
provide non-asymptotic analysis of the behavior of SGD and its
variants. \cite{DefossezB15,DieuleveutB15} provide non-asymptotic results which achieve the minimax rate on the variance (where the bias is lower order, not geometric).
\cite{NeedellSW16} achieves a geometric rate on the bias (and where the variance is not minimax). \cite{FrostigGKS15,JainKKNS16} obtain both the minimax rate on the variance and a geometric rate on the bias, as seen in equation~\ref{eq:sgd_rate}. 

\paragraph{Acceleration and Noise Stability:}  
While there have been several attempts at understanding if it is
possible to accelerate SGD , the results have been largely negative.
With regards to acceleration with adversarial (non-statistical) errors in the exact first order oracle model,~\cite{dAspremont08} provide negative results and \citet{DevolderGN13,DevolderGN14} provide
lower bounds showing that fast gradient methods do not improve upon
standard gradient methods. There is also a series of works considering
statistical errors.~\cite{Polyak87} suggests that the relative
merits of heavy ball (HB) method~\citep{Polyak64} in the noiseless case
vanish with noise unless strong assumptions on the
noise model are considered; an instance of this is when the
noise variance decays as the iterates approach the minimizer. The
Conjugate Gradient (CG) method~\citep{HestenesS52} is suggested to face
similar robustness issues in the face of statistical
errors~\citep{Polyak87}; this is in addition to the issues that CG is
known to suffer from owing to roundoff errors (due to finite precision
arithmetic)~\citep{Paige71,Greenbaum89}. In the signal processing
literature, where SGD goes by Least Mean Squares
(LMS)~\citep{WidrowS85}, there have been efforts that date to several
decades~\citep{Proakis74,RoyS90,SharmaSB98} which study accelerated LMS
methods (stochastic variants of CG/HB) in the
same oracle model as the one considered by this paper
(equation~\ref{eq:stochFirstOrderOracle}). These efforts consider the
final iterate (i.e. no iterate averaging) of accelerated LMS
methods with a fixed step-size and conclude that while it allows for a
faster decay of the initial error (bias) (which is unquantified), their steady state behavior (i.e. variance) is worse
compared to that of LMS. \citet{YuanYS16} considered a
constant step size accelerated scheme with no iterate averaging in the same oracle model as this paper, and conclude that these do not offer any improvement over standard SGD. More concretely, \citet{YuanYS16} show that the
variance of their accelerated SGD method with a sufficiently small constant step size is the same as that of SGD with a significantly larger step size. Note that none of the these efforts~\citep{Proakis74,RoyS90,SharmaSB98,YuanYS16} achieve minimax
error rates or quantify (any improvement whatsoever on the) rate of bias decay.

\paragraph{Oracle models and optimality:}With regards to notions of optimality, there are (at least) two lines
of thought: one is a statistical objective where the goal is (on every
problem instance) to match the rate of the statistically optimal
estimator~
\citep{anbar1971optimal,Fabian:1973:AES,KushnerClark,PolyakJ92};
another is on obtaining algorithms whose worst case upper bounds
(under various assumptions such as bounded noise) match the lower
bounds provided in ~\cite{NemirovskyY83}.  The work of~\cite{PolyakJ92} are in the former model, where
they show that the distribution of the averaged SGD
estimator matches, on \emph{every} problem, that of the statistically
optimal estimator, in the limit (under appropriate regularization
conditions standard in the statistics literature, where the optimal
estimator is often referred to as the maximum likelihood estimator/the
empirical risk minimizer/an
$M$-estimator~\citep{lehmann1998theory,Vaart00}). Along these lines,
non-asymptotic rates towards statistically optimal estimators are
given
by~\cite{BachM13,Bach14,DefossezB15,DieuleveutB15,NeedellSW16,FrostigGKS15,JainKKNS16}. This work can be seen as improving this non-asymptotic rate (to the
statistically optimal estimation rate) using an accelerated method. As to the latter (i.e. matching the worst-case lower bounds in
~\cite{NemirovskyY83}), there are a number of 
positive results on using accelerated stochastic
optimization procedures; the works
of~\cite{Lan08,HuKP09,ghadimi2012optimal,ghadimi2013optimal,DieuleveutFB16}
match the lower bounds provided in~\cite{NemirovskyY83}. We
compare these assumptions and works in more detail.

In stochastic first order oracle models (see ~\cite{KushnerClark,KushnerY03}), one typically has access to sampled gradients of the form:
\begin{align}
\widehat{\nabla}P(\x)  = \nabla P(\x) + \etav, \label{eq:stochFirstOrderOracle-diff}
\end{align}

\noindent where varying assumptions are made on the noise $\etav$. The
worst-case lower bounds in~\cite{NemirovskyY83} are based
on that $\etav$ is bounded; the accelerated methods in~\citet{Lan08,HuKP09,ghadimi2012optimal,ghadimi2013optimal,DieuleveutFB16}
which match these lower bounds in various cases, all assume either
bounded noise or, at least $\E{\|\etav\|^2}$ is finite. In
the least squares setting (such as the one often considered in
practice and also considered in~\citet{PolyakJ92,BachM13,DefossezB15,DieuleveutB15,FrostigGKS15,JainKKNS16}),
this assumption does not hold, since $\E{\|\etav\|^2}$ is
not bounded. To see this, $\etav$ in our oracle model (equation~\ref{eq:stochFirstOrderOracle}) is:
\begin{align}
  \label{eq:noiseModel2}
\etav=\widehat{\nabla}P(\x)-\nabla P(\x)=(\a\a\T-\Cov)(\x-\xs)-\epsilon\cdot\a
\end{align}

\noindent which implies that $\E{\|\etav\|^2}$ is not uniformly bounded (unless
additional assumptions are enforced to ensure that the algorithm's
iterates $\x$ lie within a compact set).  Hence, the assumptions made
in~\cite{HuKP09,ghadimi2012optimal,ghadimi2013optimal,DieuleveutFB16}
do not permit one to obtain finite $n$-sample bounds on the excess
risk.  Suppose we consider the case of $\epsilon=0$, i.e. where the
additive noise is zero and $\b=\a \T \xs$.  For this case, this
paper provides a geometric rate of convergence to the minimizer $\xs$,
while the results of~\cite{ghadimi2012optimal,ghadimi2013optimal,DieuleveutFB16}
at best indicate a $\mathcal{O}(1/n)$ rate. Finally, in contrast to all other existing work, our result is the first to provide finer distribution dependent characteristics of the improvements offered by accelerating SGD (e.g.  refer to the Gaussian and discrete examples in section~\ref{sec:exp}).

\paragraph{Acceleration and Finite Sums:} As a final remark, there have been
results~\citep{ShwartzZ14,FrostigGKS15b,LinMH15,LanZ15,Zhu16} that provide accelerated
rates for {\em offline} stochastic optimization which deal with
minimizing sums of convex functions; these results are almost
tight due to matching lower bounds~\citep{LanZ15,WoodworthS16}. These
results do not immediately translate into rates on the generalization
error.
Furthermore, these algorithms are not streaming, as they require making multiple passes over a
dataset stored in memory. Refer to~\citet{FrostigGKS15} for more details.

\section{Main Results}
\label{sec:prob}

We now provide our assumptions and main result, before which, we have some notation. For a vector $\x\in\R^d$ and a positive semi-definite matrix $\S\in\R^{d\times d}$ (i.e. $\S\succeq0$), denote $\|\x\|^2_\S \eqdef  \x\T \S \x$.\vspace*{-2mm}
\subsection{Assumptions and Definitions}
Let $\Cov$ denote the second moment matrix of the input, which is also the hessian $\nabla^2P(\x)$ of~\eqref{eq:objFun}:
\vspace{-0.2cm}
\begin{align*}
\Cov \eqdef \Eover{\distr}{\a\otimes\a} = \nabla^2P(\x).
\end{align*}
Furthermore, let the fourth moment tensor $\M$ of the inputs $\a\sim\D$ is defined as:
\begin{align*}\M =\Eover{\distr}{\a\otimes\a\otimes\a\otimes\a}.\end{align*}
\begin{enumerate}[leftmargin=*,label=$\mathbf{(\mathcal A\arabic*)}$]
\item \label{asmp:finiteness} \textbf{Finite second and fourth moment:} The second moment matrix $\H$ and the fourth moment tensor $\M$ exist and are finite.
\item \label{asmp:positiveDefinite}\textbf{Positive Definiteness}: The second moment matrix $\Cov$ is strictly positive definite, i.e. $\Cov\succ0$.
\end{enumerate}

\noindent We assume~\ref{asmp:finiteness} and~\ref{asmp:positiveDefinite}.~\ref{asmp:positiveDefinite} implies that $P(\x)$
is {\em strongly convex} and admits a unique minimizer
$\xs$. Denote the noise $\epsilon$ in a sample $(\a,b)\sim\D$ as: $\n \eqdef \b - \iprod{\a}{\xs}$. First order optimality conditions of $\xs$ imply $$\nabla P(\xs) = \E{\n \cdot \a} = 0.$$
Let $\Sig$ denote the covariance of gradient at optimum $\xs$ (or {\em noise covariance matrix}), $$\Sig \defeq\Eover{\distr}{\widehat{\nabla}P(\xs)\otimes\widehat{\nabla}P(\xs)}=\Eover{\distr}{\n^2\cdot\a\otimes\a}.$$

\noindent We define the \emph{noise level} $\sigma^2$, \emph{condition number} $\cnH$, \emph{statistical condition number} $\cnS$ below.\\
\noindent\textbf{Noise level}: The \emph{noise level} is defined to be the smallest positive number $\sigma^2$ such that $$\Sig\preceq \sigma^2\H.$$
The noise level $\sigma^2$ quantifies the amount of noise in the
stochastic gradient oracle and has been utilized in previous work (e.g., see~\cite{BachM11,BachM13}) for providing non-asymptotic bounds for the stochastic approximation problem. In the \emph{homoscedastic} (additive noise/well specified) case, where $\n$ is independent of the input $\a$, this condition is satisfied with equality, i.e. $\Sig = \sigma^2\ \Cov$ with $\sigma^2 = \E{\epsilon^2}$.\\
\noindent\textbf{Condition number}: Let $$\mu\defeq\lamminH.$$ $\mu>0$ by \ref{asmp:positiveDefinite}. Now, let $\infbound$ be the smallest positive number such that $$\E{\|\a\|^2\ \a\a\T} \preceq \infbound\ \Cov.$$. The \emph{condition number} $\cnH$ of the distribution $\D$~\citep{DefossezB15,JainKKNS16} is 
\begin{align*}
\cnH \defeq {\infbound}/{\mu}.
\end{align*}

\noindent \textbf{Statistical condition number}: The \emph{statistical condition number} $\cnS$ is defined as the smallest positive number such that \begin{align*}\E{\Hinvnorm{\a}^2\a\a\T}\preceq \cnS\ \Cov.\end{align*}

\noindent\textbf{Remarks on $\cnS$ and $\cnH$}:
Unlike $\cnH$, it is straightforward to see that $\cnS$ is affine invariant 
(i.e. unchanged with linear transformations over $\a$). 
Since $\E{\Hinvnorm{\a}^2\a\a\T}\preceq \frac{1}{\mu}
\E{\twonorm{\a}^2\a\a\T}\preceq \cnH \Cov$, we note $\cnS \leq \cnH$. For the discrete case
(from Section~\ref{sec:exp}), it is straightforward
to see that both $\cnH$ and $\cnS$ are equal to $1/\min_i p_{i}$. In
contrast, for the Gaussian case (from Section~\ref{sec:exp}), $\cnS$
 is $\mathcal{O}(d)$, while $\cnH$ is
 $\mathcal{O}(\textrm{Trace}(\Cov)/\mu)$ which may be arbitrarily large (based on
 choice of the coordinate system).

$\cnS$ governs how many samples $\ai$ require to be drawn from $\D$ so
 that the empirical covariance is spectrally close to $\Cov$,
 i.e. for some constant $c>1$, $\Cov/c \preceq \frac{1}{n}\sum_{i=1}^n \a_i\a_i\T \preceq c \Cov$.
 In comparison to the matrix Bernstein inequality where stronger (yet related) moment
 conditions are assumed in order to obtain high probability results,
 our results hold only in expectation (refer to~\citet{HsuKZ14} for this definition, wherein $\cnS$ is referred to as bounded statistical leverage in theorem $1$ and remark $1$).\vspace*{-2mm}

\begin{algorithm}[t]
	\caption{ (Tail-Averaged) \textbf{A}ccelerated \textbf{S}tochastic \textbf{G}radient
          \textbf{D}escent (ASGD)}
	\label{algo:TAASGD}
	\begin{algorithmic}[1]
		\INPUT $n$ oracle calls~\ref{eq:stochFirstOrderOracle}, initial point $\xt[0]=\vt[0]$, Unaveraged (burn-in) phase $t$, Step size parameters $\alpha, \beta, \gamma, \delta$
		\FOR{$j = 1, \cdots n$}
		\STATE $\yt[j-1] \leftarrow \alpha \xt[j-1] + (1-\alpha) \vt[j-1]$
		\STATE $\xt[j] \leftarrow \yt[j-1] - \delta \widehat{\nabla} P(\yt[j-1])$
		\STATE $\zt[j-1] \leftarrow \beta \yt[j-1] + (1-\beta) \vt[j-1]$
		\STATE $\vt[j] \leftarrow \zt[j-1] - \gamma \widehat{\nabla} P(\yt[j-1])$
		\ENDFOR
		\OUTPUT $\bar{\x}_{t,n} \leftarrow \frac{1}{n-t}\sum_{j=t+1}^{n} \xt[j]$
	\end{algorithmic}
\end{algorithm}
\subsection{Algorithm and Main Theorem}\label{sec:results}
Algorithm~\ref{algo:TAASGD} presents the pseudo code of the proposed algorithm. ASGD can be viewed as a variant of Nesterov's accelerated gradient method~\citep{Nesterov12}, working with a stochastic gradient oracle (equation~\ref{eq:stochFirstOrderOracle}) and with tail-averaging the final $n-t$ iterates. The main result now follows:
\begin{theorem}\label{thm:main}
Suppose ~\ref{asmp:finiteness} and
~\ref{asmp:positiveDefinite} hold. Set $\alpha =
\frac{3\sqrt{5}\cdot\sqrt{\cnH\cnS}}{1+3\sqrt{5}\cdot\sqrt{\cnH\cnS}},
\beta = \frac{1}{9\sqrt{\cnH\cnS}}, \gamma =  \frac{1}{3\sqrt{5}\cdot
  \mu \sqrt{\cnH\cnS}}, \delta = \frac{1}{5R^2}$. After $n$ calls to
the stochastic first order oracle
(equation~\ref{eq:stochFirstOrderOracle}), ASGD
outputs $\bar{\x}_{t,n}$ satisfying:
\vspace{-0.3cm}
	\begin{align*}
	&\E{P(\bar{\x}_{t,n})}-P(\xs) \leq \underbrace{\UC\cdot\frac{(\cnH\cnS)^{9/4}d\cnH}{(n-t)^2}\cdot\exp\bigg(\frac{-t}{9\sqrt{\cnH\cnS}}\bigg)\cdot\big(P(\x_0)-P(\xs)\big)}_{\text{Leading order bias error}}+\underbrace{5\frac{\sigma^2d}{n-t}}_{\text{Leading order variance error}} + \nonumber\\&\underbrace{\UC\cdot(\cnH\cnS)^{5/4}d\cnH\cdot\exp\left(\frac{-n }{9\sqrt{\cnH\cnS}}\right) \big(P(\x_0)-P(\xs)\big)}_{\text{Exponentially vanishing lower order bias term}} + \underbrace{\UC\cdot\frac{\sigma^2 d}{(n-t)^2} \sqrt{\cnH\cnS}}_{\text{Lower order variance error term}}+\nonumber\\ &{\small \underbrace{\UC\cdot\exp\bigg({-\frac{n}{9\sqrt{\cnH\cnS}}}\bigg)\cdot\bigg(\sigma^2d\cdot(\cnH\cnS)^{7/4}+\frac{\sigma^2d}{(n-t)^2}\cdot(\cnH\cnS)^{7/2}\cnS\bigg)+C\cdot\frac{\sigma^2d}{n-t}(\cnH\cnS)^{11/4}\exp\bigg({-\frac{(n-t-1)}{30\sqrt{\cnH\cnS}}}\bigg)}_{\text{Exponentially vanishing lower order variance error terms}}},
	\end{align*}

	\vspace{-0.3cm}
	\noindent 
	where $\UC$ is a universal constant, $\sigma^2$, $\cnH$ and $\cnS$ are the noise level, condition number and statistical condition number respectively.
\end{theorem}
The following corollary holds if the iterates are tail-averaged over the last $n/2$ samples and $n>\mathcal{O}(\sqrt{\cnH\cnS}\log(d\cnH\cnS))$. The second condition lets us absorb lower order terms into leading order terms. 
\begin{corollary}\label{cor:lowerOrder}
Assume the parameter settings of theorem~\ref{thm:main} and
let $t=\lfloor n/2 \rfloor$ and $n>\UC'\sqrt{\cnH\cnS}\log(d\cnH\cnS)$ (for
an appropriate universal constants $\UC,\UC'$).
We have that with $n$ calls to the stochastic first order oracle,
ASGD outputs a vector $\bar{\x}_{t,n}$
satisfying:\vspace*{-2mm}
	\begin{align*}
	&\E{P(\bar{\x}_{t,n})}-P(\xs) \leq \UC \cdot\exp\bigg(-\frac{n}{20\sqrt{\cnH\cnS}}\bigg)\cdot\big(P(\x_0)-P(\xs)\big)+11\frac{\sigma^2d}{n}.
	\end{align*}
\end{corollary}

A few remarks about the result of theorem~\ref{thm:main} are due: (i) ASGD decays the initial error at a geometric rate of $\mathcal{O}(1/\sqrt{\cnH\cnS})$ during the unaveraged phase of $t$ iterations, which presents the first improvement over the $\mathcal{O}\left(1/\cnH\right)$ rate offered by SGD~\citep{RobbinsM51}/averaged SGD~\citep{PolyakJ92,JainKKNS16} for the least squares stochastic approximation problem, (ii) the second term in the error bound indicates that ASGD obtains (up to constants) the minimax rate once $n>\mathcal{O}(\sqrt{\cnH\cnS}\log(d\cnH\cnS))$. Note that this implies that Theorem~\ref{thm:main} provides a sharp non-asymptotic analysis (up to $\log$ factors) of the behavior of Algorithm~\ref{algo:TAASGD}.

\subsection{Discussion and Open Problems}\label{sec:resultDiscussion}

A challenging problem in this context is in formalizing a finite sample size lower bound in the oracle model considered in this work.  Lower bounds in stochastic oracle models have been considered in the literature (see~\cite{NemirovskyY83,RaginskyR11,AgarwalBRW12}), though it is not evident these oracle models and lower bounds are sharp enough to imply statements in our setting (see section~\ref{sec:related} for a discussion of these oracles).

Let us now understand theorem~\ref{thm:main} in the broader context of stochastic approximation. Under certain regularity conditions, it is known that~\citep{lehmann1998theory,Vaart00} that the rate described in equation~\ref{eq:ERMVarianceAdditive} for the homoscedastic case holds for a broader set of misspecified models (i.e., heteroscedastic noise case), with an appropriate definition of the noise variance. By defining $\sigma^2_{\textrm{ERM}}\eqdef\E{\norm{\widehat{\nabla}P(\xs)}^2_{\inv{\H}}}$, the rate of the ERM is guaranteed to approach $\sigma^2_{\textrm{ERM}}/n$~\citep{lehmann1998theory,Vaart00} in the limit of large $n$, i.e.:
\begin{align}
\label{eq:ERMVariance}
	\lim_{n\to\infty}\frac{\mathbb{E}_{\ST_n}[P_n(\xhat_n^{\textrm{ERM}})]-P(\xs)}{\sigma^2_{\textrm{ERM}}/n} &= 1,
\end{align}

\noindent where $\xhat_n^{\textrm{ERM}}$ is the ERM over samples $\ST_n=\{\a_i,b_i\}_{i=1}^n$. Averaged SGD~\citep{JainKKNS16} and streaming SVRG~\citep{FrostigGKS15} are known to achieve these rates for the heteroscedastic case. Refer to~\cite{FrostigGKS15} for more details.Neglecting constants, Theorem~\ref{thm:main} is guaranteed to achieve the rate of the ERM for the {\em homoscedastic} case (where $\Sig=\sigma^2\H$) and is tight when the bound $\Sig\preceq\sigma^2\H$ is nearly tight (upto constants). We conjecture ASGD achieves the rate of the ERM in the heteroscedastic case by appealing to a more refined analysis as is the case for averaged SGD (see~\cite{JainKKNS16}). It is also an open question to understand acceleration for smooth stochastic approximation (beyond least squares), in situations where the rate represented by equation~\ref{eq:ERMVariance} holds~\citep{PolyakJ92}.

\section{Proof Outline}\label{sec:proofoutline}
We now present a brief outline of the proof of Theorem~\ref{thm:main}. Recall the variables in Algorithm~\ref{algo:TAASGD}. Before presenting the proof outline we require some definitions. We begin by defining the centered estimate $\thetat$ as:
\begin{align*}
\thetat \eqdef \left[\begin{array}{c} \xt - \xs \\ \yt - \xs \end{array}\right]\in\R^{2d}. 
\end{align*}

\noindent Recall that the stepsizes in Algorithm~\ref{algo:TAASGD} are $\alpha = \frac{3\sqrt{5}\cdot\sqrt{\cnH\cnS}}{1+3\sqrt{5}\cdot\sqrt{\cnH\cnS}}, \beta = \frac{1}{9\sqrt{\cnH\cnS}}, \gamma =  \frac{1}{3\sqrt{5}\cdot \mu \sqrt{\cnH\cnS}}, \delta = \frac{1}{5R^2}$. The accelerated SGD updates of Algorithm~\ref{algo:TAASGD} can be written in terms of $\thetat$ as:

{\small\begin{align*}
\thetat = \Ahatj \thetat[j-1] + \zetat, &\quad \text{where,}\\
\Ahatj \defeq \begin{bmatrix} 0 & (\eye-\delta\ai[j]\ai[j]\T)\\ -\alpha(1-\beta)\ \eye & (1+\alpha(1-\beta))\eye-(\alpha\delta+(1-\alpha)\gamma)\ai[j]\ai[j]\T \end{bmatrix}&,
\zetat[j] \eqdef  \left[\begin{array}{c} \delta \cdot \ni[j] \ai[j] \\ (\alpha \delta + (1-\alpha)\gamma) \cdot \ni[j] \ai[j] \end{array}\right],
\end{align*}}%

\noindent where $\epsilon_j=b_j-\iprod{\ai[j]}{\xs}$. The tail-averaged iterate $\bar{\x}_{t,n}$ is associated with its own centered estimate $\thetavb\defeq\frac{1}{n-t}\sum_{j=t+1}^n\thetav_j$. Let $\A \eqdef \E{\Ahatj|\mathcal{F}_{j-1}}$, where $\mathcal{F}_{j-1}$ is a filtration generated by $(\a_1,b_1),\cdots,(\a_{j-1},b_{j-1})$. Let $\B,\AL,\AR$ be linear operators acting on a matrix $\S\in\R^{2d\times2d}$ so that $\B\S\eqdef \E{\Ahatj \S \Ahatj\T|\mathcal{F}_{j-1}}$, $\AL\S\eqdef\A\S$, $\AR\S\eqdef\S\A$. Denote $\Sighat\eqdef \E{\zetat \zetat \T|\mathcal{F}_{j-1}}$ and matrices $\G,\Z,\PM$ as:
\begin{align*}
\G\defeq\PM \T \mat{Z}\PM , \text{where}, \PM\eqdef \begin{bmatrix} \Id &\zero \\ \frac{-\alpha}{1-\alpha}\Id & \frac{1}{1-\alpha}\Id \end{bmatrix},\ \ \mat{Z}\eqdef\begin{bmatrix} \Id &\zero \\ \zero & {\mu}\inv{\Cov}\end{bmatrix}.
\end{align*}

\noindent \underline{\textbf{Bias-variance decomposition}}: The proof of theorem~\ref{thm:main} employs the {\em bias-variance} decomposition, which is well known in the context of stochastic approximation (see~\cite{BachM11,FrostigGKS15,JainKKNS16}) and is re-derived in the appendix. The bias-variance decomposition allows for the generalization error to be upper-bounded by analyzing two sub-problems: (a) {\em bias}, analyzing the algorithm's behavior on the {\em noiseless} problem (i.e. $\zetav_j=0\ \forall\ j$ a.s.) while starting at $\thetav_0^{\textrm{bias}}=\thetav_0$ and (b) {\em variance}, analyzing the algorithm's behavior by starting at the solution (i.e. $\thetav_0^{\textrm{variance}}=0$) and allowing the noise $\zetav_{\Bigcdot}$ to drive the process. In a similar manner as $\thetavb$, the bias and variance sub-problems are associated with $\thetavb^{\textrm{bias}}$ and $\thetavb^{\textrm{variance}}$, and these are related as:
\begin{align}\label{eqn:bound-covar}
		\E{\thetavb \otimes \thetavb} \preceq 2\cdot\bigg( \E{\thetavb^{\text{bias}} \otimes \thetavb^{\text{bias}}} + \E{\thetavb^{\text{variance}} \otimes \thetavb^{\text{variance}}}\bigg).
\end{align}

\noindent Since we deal with the square loss, the generalization error of the output $\bar{\x}_{t,n}$ of algorithm~\ref{algo:TAASGD} is:
\begin{align}
\label{eq:genErrorExp}
\E{P(\bar{\x}_{t,n})}-P(\xs)=\frac{1}{2}\cdot\iprod{\begin{bmatrix}\Cov&0\\0&0\end{bmatrix}}{\E{\thetavb \otimes \thetavb}},
\end{align}

\noindent indicating that the generalization error can be bounded by analyzing the bias and variance sub-problem.
We now present the lemmas that bound the bias error.
\begin{lemma}\label{lem:average-covar-bias}
The covariance $\E{\thetavb^{\textrm{bias}} \otimes \thetavb^{\text{bias}}}$ of the bias part of averaged iterate $\thetavb^{\textrm{bias}}$ satisfies:
{\small\begin{align*}
\E{\thetavb^{\textrm{bias}} \otimes \thetavb^{\text{bias}}} &=  \frac{1}{(n-t)^2} \bigg( \eyeT + (\eyeT-\AL)^{-1}\AL + (\eyeT-\AR\T)^{-1}\AR\T\bigg) (\eyeT-\BT)^{-1}(\BT^{t+1}-\BT^{n+1}) \left(\thetat[0]\otimes \thetat[0]\right)\\	 &\quad -\frac{1}{(n-t)^2}\sum_{j=t+1}^n\bigg( (\eyeT-\AL)^{-1}\AL^{n+1-j} + (\eyeT-\AR\T)^{-1}(\AR\T)^{n+1-j} \bigg)\BT^j (\thetat[0] \otimes \thetat[0]).
\end{align*}}%
\end{lemma}
The quantity that needs to be bounded in the term above is $\BT^{t+1}\thetav_0\otimes\thetav_0$. Lemma~\ref{lem:main-bias} presents a result that can be applied recursively to bound $\BT^{t+1}\thetav_0\otimes\thetav_0$ ($=\BT^{t+1}\thetav_0^{\textrm{bias}}\otimes\thetav_0^{\textrm{bias}}$ since $\thetav_0^{\textrm{bias}}=\thetav_0$).
\begin{lemma}[Bias contraction] \label{lem:main-bias}
	For any two vectors $\x, \y \in \R^d$, let $\thetav \eqdef \begin{bmatrix}
	\x-\xs \\ \y-\xs
	\end{bmatrix} \in \R^{2d}$. We have:
	\begin{align*}
	\iprod{\G}{\B \left(\thetav \thetav \T\right)} \leq \bigg(1-\frac{1}{9\sqrt{\cnH\cnS}}\bigg) \iprod{\G}{\thetav \thetav \T}
	\end{align*}
\end{lemma}

\paragraph{Remarks:}(i) the matrices $\PM$ and $\PM\T$ appearing in $\G$ are due to the fact that we prove contraction using the variables $\x-\xs$ and $\v-\xs$ instead of $\x-\xs$ and $\y-\xs$, as used in defining $\thetav$. (ii) The key novelty in lemma~\ref{lem:main-bias} is that while standard analyses of accelerated gradient descent (in the exact first order oracle) use the potential function $\norm{\x-\xs}_{\Cov}^2 + \mu \twonorm{\v - \xs}^2$ (e.g.~\cite{WilsonRJ16}), we consider it crucial for employing the potential function $\twonorm{\x-\xs}^2 + \mu \norm{\v - \xs}_{\inv{\Cov}}^2$ (this potential function is captured using the matrix $\mat{Z}$) to prove accelerated rates (of $\order{1/\sqrt{\cnH\cnS}}$) for bias decay. 

We now present the lemmas associated with bounding the variance error:
\begin{lemma}\label{lem:average-covar-var}
	The covariance $\E{\thetavb^{\textrm{variance}} \otimes \thetavb^{\text{variance}}}$ of the variance error $\thetavb^{\textrm{variance}}$ satisfies:
	{\small
	\begin{align*}
	&\E{\thetavb^{\textrm{variance}} \otimes \thetavb^{\text{variance}}} = \frac{1}{n-t}\big(\eyeT + (\eyeT-\AL)^{-1}\AL + (\eyeT-\AR\T)^{-1}\AR\T\big)(\eyeT-\BT)^{-1}\Sigh	\nonumber\\
	&\quad -\frac{1}{(n-t)^2}\big((\eyeT-\AL)^{-2}(\AL-\AL^{n+1-t})+(\eyeT-\AR\T)^{-2}(\AR\T-(\AR\T)^{n+1-t})\big)(\eyeT-\BT)^{-1}\Sigh\nonumber\\
	&\quad -\frac{1}{(n-t)^2}\big(\eyeT + (\eyeT-\AL)^{-1}\AL + (\eyeT-\AR\T)^{-1}\AR\T\big)(\eyeT-\BT)^{-2}(\BT^{t+1}-\BT^{n+1})\Sigh\nonumber\\
	&\quad +\frac{1}{(n-t)^2}\sum_{j=t+1}^n\big((\eyeT-\AL)^{-1}\AL^{n+1-j}+(\eyeT-\AR\T)^{-1}(\AR\T)^{n+1-j}\big)(\eyeT-\BT)^{-1}\BT^j\Sigh.
	\end{align*}}
\end{lemma}
The covariance of the stationary distribution of the iterates i.e., $\lim_{j\to\infty}\thetav_j^{\textrm{variance}}$ requires a precise bound to obtain statistically optimal error rates. Lemma~\ref{lem:main-variance} presents a bound on this quantity. 
\begin{lemma}[Stationary covariance]\label{lem:main-variance}
The covariance of limiting distribution of $\thetav^{\textrm{variance}}$ satisfies:
	\begin{align*}	\E{\thetav_{\infty}^{\textrm{variance}}\otimes\thetav_{\infty}^{\textrm{variance}}}=\inv{\left(\Id - \B\right)} \Sighat &\preceq 5\sigma^2\bigg((2/3)\cdot \big(\frac{1}{\cnS}\Hinv\big)+ (5/6) \cdot(\delta\eye)\bigg)\otimes\begin{bmatrix} 1 & 0 \\ 0 & 1\end{bmatrix}.
	\end{align*}
\end{lemma}
A crucial implication of this lemma is that the limiting final iterate $\thetav_{\infty}^{\textrm{variance}}$ has an excess risk $\mathcal{O}(\sigma^2)$. This result naturally lends itself to the (tail-)averaged iterate achieving the minimax optimal rate of $\mathcal{O}(d\sigma^2/n)$. Refer to the appendix~\ref{sec:varianceContraction} and lemma~\ref{lem:var-main-1} for more details in this regard.
\section{Conclusion}\label{sec:conclusion}
This paper introduces an accelerated stochastic gradient method, which presents the first improvement in achieving minimax rates faster than averaged SGD~\citep{RobbinsM51,Ruppert88,PolyakJ92,JainKKNS16}/Streaming SVRG~\citep{FrostigGKS15} for the stochastic approximation problem of least squares regression. To obtain this result, the paper presented the need to rethink what acceleration has to offer when working with a stochastic gradient oracle: these thought experiments indicated a need to consider a quantity that captured more fine grained problem characteristics. The statistical condition number (an affine invariant distributional quantity) is shown to characterize the improvements that acceleration offers in the stochastic first order oracle model. 

In essence, this paper presents the first known provable analysis of the claim that fast gradient methods are stable when dealing with statistical errors, in contrast to negative results in statistical and non-statistical settings~\citep{Paige71,Proakis74,Polyak87,Greenbaum89,RoyS90,SharmaSB98,dAspremont08,DevolderGN13,DevolderGN14,YuanYS16}. We hope that this paper provides insights towards developing simple and effective accelerated stochastic gradient schemes for general convex and non-convex optimization problems.

\paragraph{Acknowledgments:} Sham Kakade acknowledges funding from Washington Research Foundation Fund for Innovation in Data-Intensive Discovery and the NSF through awards CCF-$1637360$, CCF-$1703574$ and CCF-$1740551$. 
 

\clearpage

\clearpage
\appendix
\section{Appendix setup}
\label{sec:setup}
We will first provide a note on the organization of the appendix and follow that up with introducing the notations.

\subsection{Organization}
\label{ssec:org}
\begin{itemize}
\item In subsection~\ref{ssec:notations}, we will recall notation from the main paper and introduce some new notation that will be used across the appendix.
\item In section~\ref{sec:tailAverageIterateCovariance}, we will write out expressions that characterize the generalization error of the proposed accelerated SGD method. In order to bound the generalization error, we require developing an understanding of two terms namely the bias error and the variance error.
\item In section~\ref{sec:commonLemmas}, we prove lemmas that will be used in subsequent sections to prove bounds on the bias and variance error.
\item In section~\ref{sec:biasContraction}, we will bound the bias error of the proposed accelerated stochastic gradient method. In particular, lemma~\ref{lem:main-bias} is the key lemma that provides a new potential function with which this paper achieves acceleration. Further, lemma~\ref{lem:bound-bias} is the lemma that bounds all the terms of the bias error.
\item In section~\ref{sec:varianceContraction}, we will bound the variance error of the proposed accelerated stochastic gradient method. In particular, lemma~\ref{lem:main-variance} is the key lemma that considers a stochastic process view of the proposed accelerated stochastic gradient method and provides a sharp bound on the covariance of the stationary distribution of the iterates. Furthermore, lemma~\ref{lem:bound-variance} bounds all terms of the variance error.
\item Section~\ref{sec:proofMainTheorem} presents the proof of Theorem~\ref{thm:main}. In particular, this section aggregates the result of lemma~\ref{lem:bound-bias} (which bounds all terms of the bias error) and lemma~\ref{lem:bound-variance} (which bounds all terms of the variance error) to present the guarantees of Algorithm~\ref{algo:TAASGD}.
\end{itemize}

\subsection{Notations}
\label{ssec:notations}

We begin by introducing $\M$, which is the fourth moment tensor of the input $\a\sim\D$, i.e.:
\begin{align*}
\M\defeq\Eover{\distr}{\a\otimes\a\otimes\a\otimes\a}
\end{align*}
Applying the fourth moment tensor $\M$ to any matrix $\S\in\R^{d\times d}$ produces another matrix in $\R^{d\times d}$ that is expressed as: 
\begin{align*}
\M\S \defeq \E{(\a\T\S\a)\a\a\T}. 
\end{align*}
With this definition in place, we recall $\infbound$ as the smallest number, such that $\M$ applied to the identity matrix $\eye$ satisfies:
\begin{align*}
\M\eye=\E{\twonorm{\a}^2\a\a\T}\preceq\infbound\ \Cov
\end{align*}
Moreover, we recall that the condition number of the distribution $\cnH = \infbound/\mu$, where $\mu$ is the smallest eigenvalue of $\Cov$. Furthermore, the definition of the statistical condition number $\cnS$ of the distribution follows by applying the fourth moment tensor $\M$ to $\Covinv$, i.e.:
\begin{align*}
\M\Covinv&=\E{(\a\T\Covinv\a)\cdot\a\a\T}\preceq\cnS\ \H
\end{align*}

We denote by $\tensor{A}_{\mathcal{L}}$ and $\tensor{A}_{\mathcal{R}}$ the left and right multiplication operator of any matrix $\A\in\R^{d\times d}$, i.e. for any matrix $\S\in\R^{d\times d}$, $\tensor{A}_{\mathcal{L}}\S=\A\S$ and $\tensor{A}_{\mathcal{R}}\S=\S\A$.

\underline{\bf Parameter choices:} In all of appendix we choose the parameters in Algorithm~\ref{algo:TAASGD} as
\begin{align*}
\alpha = \frac{\sqrt{\cnH\cnHh}}{\ctwo\sqrt{2\cone-\cone^2}+\sqrt{\cnH\cnHh}},\ \  \beta = \cthree\frac{\ctwo\sqrt{2\cone-\cone^2}}{\sqrt{\cnH\cnHh}},\ \  \gamma = \ctwo\frac{\sqrt{2\cone-\cone^2}}{\mu\sqrt{\cnH\cnHh}}, \ \ \delta=\frac{\cone}{\infbound}
\end{align*}
where $\cone$ is an arbitrary constant satisfying $0 < \cone < \frac{1}{2}$. Furthermore, we note that $\cthree=\frac{\ctwo\sqrt{2\cone-\cone^2}}{\cone}$, $\ctwo^2=\frac{\cfour}{2-\cone}$ and $\cfour< 1/6$. 
Note that we recover Theorem~\ref{thm:main} by choosing $\cone = 1/5, \ctwo = \sqrt{5}/9, \cthree=\sqrt{5}/3, \cfour = 1/9$. We denote 
\begin{align*}
c\defeq\alpha(1-\beta) \text{ and, } \g\defeq\alpha\delta+(1-\alpha)\gamma.
\end{align*}

Recall that $\xs$ denotes unique minimizer of $P(\x)$, i.e. $\xs=\arg\min_{\x \in \R^d} \Eover{\distr}{(b-\iprod{\x}{\a})^2}$. We track $\thetav_k=\begin{bmatrix}\x_k-\xs\\ \yt[k]-\xs\end{bmatrix}$. The following equation captures the updates of Algorithm~\ref{algo:TAASGD}:
\begin{align}
\label{eq:mainRec}
\thetav_{k+1}&=\begin{bmatrix}0&\eye-\delta\widehat{\H}_{k+1}\\-c\cdot\eye&(1+c)\cdot\eye-\g\cdot\widehat{\H}_{k+1}\end{bmatrix}\thetav_k
+\begin{bmatrix}\delta\cdot\epsilon_{k+1}\av_{k+1}\\\g\cdot\epsilon_{k+1}\av_{k+1}\end{bmatrix}\nonumber\\
&\defeq\Ah_{k+1}\thetav_{k}+\zetav_{k+1},
\end{align}
where, $\widehat{\H}_{k+1} \defeq \av_{k+1} \av_{k+1}^\top$, $\widehat{\A}_{k+1} \defeq \begin{bmatrix}0&\eye-\delta\widehat{\H}_{k+1}\\-c\cdot\eye&(1+c)\cdot\eye-\g\cdot\widehat{\H}_{k+1}\end{bmatrix}$ 
and $\zetav_{k+1} \defeq \begin{bmatrix}\delta\cdot\epsilon_{k+1}\av_{k+1}\\\g\cdot\epsilon_{k+1}\av_{k+1}\end{bmatrix}$.

\noindent Furthermore, we denote by $\phiv_k$ the expected covariance of $\thetav_k$, i.e.:
\begin{align*}
\phiv_k\defeq\E{\thetav_k\otimes\thetav_k}.
\end{align*}
Next, let $\mathcal{F}_k$ denote the filtration generated by samples $\{(\a_1,b_1),\cdots, (\a_k,b_k)\}$. Then,
\begin{align*}
\A&\eqdef \E{\Ah_{k+1}|\mathcal{F}_{k}}=\begin{bmatrix}
\zero & \eye - \delta \Cov \\ -c\eye & (1+c)\eye - \g\Cov
\end{bmatrix}.
\end{align*}
By iterated conditioning, we also have
\begin{align}\label{eqn:theta-det}
	\E{\thetav_{k+1}\middle \vert \mathcal{F}_{k}} = \A \thetav_k.
\end{align}
Without loss of generality, we assume that $\Cov$ is a diagonal matrix. We now note that we can rearrange the coordinates through an eigenvalue decomposition so that $\A$ becomes a block-diagonal matrix with $2\times2$ blocks. We denote the $j^{\textrm{th}}$ block by $\A_j$:
\begin{align*}
\A_j \eqdef \begin{bmatrix}
0 & 1 - \delta \lambda_j \\ -c & 1+c - \g \lambda_j
\end{bmatrix},
\end{align*}
where $\lambda_j$ denotes the $j^{\textrm{th}}$ eigenvalue of $\Cov$.
Next, 
\begin{align*}
\BT&\eqdef \E{\Ah_{k+1}\otimes\Ah_{k+1}|\mathcal{F}_{k}}, \mbox{ and }\\
\Sigh&\eqdef \E{\zetav_{k+1}\otimes\zetav_{k+1}|\mathcal{F}_k} = \begin{bmatrix}\delta^2&\delta\cdot \g\\\delta\cdot \g&\g^2\end{bmatrix}\otimes\Sig\preceq\sigma^2\cdot\begin{bmatrix}\delta^2&\delta\cdot \g\\\delta\cdot \g&\g^2\end{bmatrix}\otimes\H.
\end{align*}
Finally, we observe the following:
\begin{align*}
\E{(\A-\Ah_{k+1})\otimes(\A-\Ah_{k+1})|\mathcal{F}_k}&=\A\otimes\A-\E{\Ah_{k+1}\otimes\A|\mathcal{F}_k}\\&\quad\quad-\E{\Ah_{k+1}\otimes\A|\mathcal{F}_k}+\E{\Ah_{k+1}\otimes\Ah_{k+1}|\mathcal{F}_k}\\
&=-\A\otimes\A+\E{\Ah_{k+1}\otimes\Ah_{k+1}|\mathcal{F}_k}\\
\implies \E{\Ah_{k+1}\otimes\Ah_{k+1}|\mathcal{F}_k}&=\E{(\A-\Ah_{k+1})\otimes(\A-\Ah_{k+1})|\mathcal{F}_k}+\A\otimes\A
\end{align*}
We now define:
\begin{align*}
\RT&\defeq\E{(\A-\Ah_{k+1})\otimes(\A-\Ah_{k+1})|\mathcal{F}_k}, \mbox{ and }\\
\DT&\defeq\A\otimes\A.
\end{align*}
Thus implying the following relation between the operators $\BT,\DT$ and $\RT$:
\begin{align*}
\BT=\DT+\RT.
\end{align*}

\section{The Tail-Average Iterate: Covariance and bias-variance decomposition}
\label{sec:tailAverageIterateCovariance}
We begin by considering the first-order Markovian recursion as defined by equation~\ref{eq:mainRec}:
\begin{align*}
\thetav_j&=\Ah_j\thetav_{j-1}+\zetav_j.
\end{align*}
We refer by $\phiv_j$ the covariance of the $j^{\text{th}}$ iterate, i.e.:
\begin{align}
\label{eq:finalIterateCovariance}
\phiv_j \defeq \E{\thetav_j\otimes\thetav_j}
\end{align}
Consider a decomposition of $\thetat$ as $\thetat = \thetat^{\textrm{bias}} + \thetat^{\textrm{variance}}$, where $\thetat^{\textrm{bias}}$ and $\thetat^{\textrm{variance}}$ are defined as follows:
\begin{align}
	\thetat^{\textrm{bias}} \eqdef \Ahatj \thetat[j-1]^{\textrm{bias}} &;\qquad  \thetat[0]^{\textrm{bias}} \eqdef \thetat[0], \mbox{ and } \label{eq:biasRec}\\
	\thetat^{\textrm{variance}} \eqdef \Ahatj \thetat[j-1]^{\textrm{variance}} +\zetat &;\qquad  \thetat[0]^{\textrm{variance}} \eqdef \zero. \label{eq:varianceRec}
\end{align}
We note that 
\begin{align}
\E{\thetav_j^{\textrm{bias}}}=\A\E{\thetav_{j-1}^{\textrm{bias}}}\label{eq:condExpBias},\\
\E{\thetav_j^{\textrm{variance}}}=\A\E{\thetav_{j-1}^{\textrm{variance}}}\label{eq:condExpVar}.
\end{align}
Note equation~\ref{eq:condExpVar} follows using a conditional expectation argument with the fact that $\E{\zetav_k}=0\ \forall\ k$ owing to first order optimality conditions.

Before we prove the decomposition holds using an inductive argument, let us understand what the bias and variance sub-problem intuitively mean. 

Note that the {\em bias} sub-problem (defined by equation~\ref{eq:biasRec}) refers to running algorithm on the noiseless problem (i.e., where, $\zetav_{\Bigcdot}=0$ a.s.) by starting it at $\thetav_0^{\textrm{bias}}=\thetav_0$. The bias essentially measures the dependence of the generalization error on the excess risk of the initial point $\thetav_0$ and bears similarities to convergence rates studied in the context of offline optimization. 

The {\em variance} sub-problem (defined by equation~\ref{eq:varianceRec}) measures the dependence of the generalization error on the noise introduced during the course of optimization, and this is associated with the statistical aspects of the optimization problem. The variance can be understood as starting the algorithm at the solution ($\thetav_0^{\text{variance}}=0$) and running the optimization driven solely by noise. Note that the variance is associated with sharp statistical lower bounds which dictate its rate of decay as a function of the number of oracle calls $n$.

Now, we will prove that the decomposition $\thetat = \thetat^{\textrm{bias}} + \thetat^{\textrm{variance}}$ captures the recursion expressed in equation~\ref{eq:mainRec} through induction. For the base case $j=1$, we see that 
\begin{align*}
\thetav_1&=\Ah_1\thetav_0+\zetav_1\\
&=\underbrace{\Ah_1\thetav_0^{\textrm{bias}}}_{\because\ \thetav_0^{\textrm{bias}}=\thetav_0}+\underbrace{\Ah_1\thetav_0^{\textrm{variance}}}_{=0 ,\ \because\ \thetav_0^{\textrm{variance}}=0}+\zetav_1\\
&=\thetav_1^{\textrm{bias}}+\thetav_1^{\textrm{variance}}
\end{align*}
Now, for the inductive step, let us assume that the decomposition holds in the $j-1^{st}$ iteration, i.e. we assume $\thetav_{j-1}=\thetav_{j-1}^{\textrm{bias}}+\thetav_{j-1}^{\textrm{variance}}$. We will then prove that this relation holds in the $j^{th}$ iteration. Towards this, we will write the recursion:
\begin{align*}
\thetav_j&=\Ah_j\thetav_{j-1}+\zetav_j\\
&=\Ah_j(\thetav_{j-1}^{\textrm{bias}}+\thetav_{j-1}^{\textrm{variance}})+\zetav_j\quad\text{(using the inductive hypothesis)}\\
&=\Ah_j\thetav_{j-1}^{\textrm{bias}}+\Ah_j\thetav_{j-1}^{\textrm{variance}}+\zetav_j\\
&=\thetav_j^{\textrm{bias}}+\thetav_j^{\textrm{variance}}.
\end{align*}
This proves the decomposition holds through a straight forward inductive argument.

In a similar manner as $\thetat$, the tail-averaged iterate $\thetavb \eqdef \frac{1}{n-t} \sum_{j=t+1}^n \thetat[j]$ can also be written as $\thetavb = \thetavb^{\textrm{bias}} + \thetavb^{\textrm{variance}}$, where $\thetavb^{\textrm{bias}} \eqdef \frac{1}{n-t} \sum_{j=t+1}^n \thetat[j]^{\textrm{bias}}$ and $\thetavb^{\textrm{variance}} \eqdef \frac{1}{n-t} \sum_{j=t+1}^n \thetat[j]^{\textrm{variance}}$. Furthermore, the tail-averaged iterate $\thetavb$ and its bias and variance counterparts $\thetavb^{\text{bias}},\thetavb^{\text{variance}}$ are associated with their corresponding covariance matrices $\phivb_{t,n},\phivb_{t,n}^{\text{bias}},\phivb_{t,n}^{\text{variance}}$ respectively. Note that $\phivb_{t,n}$ can be upper bounded using Cauchy-Shwartz inequality as:
\begin{align}\label{eq:tailAvgCovarBound}
		\E{\thetavb \otimes \thetavb} &\preceq 2\cdot\bigg( \E{\thetavb^{\text{bias}} \otimes \thetavb^{\text{bias}}} + \E{\thetavb^{\text{variance}} \otimes \thetavb^{\text{variance}}}\bigg)\nonumber\\
\implies\phivb_{t,n}&\preceq 2\cdot(\phivb_{t,n}^{\text{bias}}+\phivb_{t,n}^{\text{variance}}).
\end{align}
The above inequality is referred to as the {\em bias-variance} decomposition and is well known from previous work~\cite{BachM13,FrostigGKS15,JainKKNS16}, and we re-derive this decomposition for the sake of completeness.
We will now derive an expression for the covariance of the tail-averaged iterate and apply it to obtain the covariance of the bias ($\phivb_{t,n}^{\text{bias}}$) and variance ($\phivb_{t,n}^{\text{variance}}$) error of the tail-averaged iterate.
\subsection{The tail-averaged iterate and its covariance}
We begin by writing out an expression for the tail-averaged iterate $\thetavb$ as: 
\begin{align*}
\thetavb=\frac{1}{n-t}\sum_{j=t+1}^n\thetav_j
\end{align*}
To get the excess risk of the tail-averaged iterate $\thetavb$, we track its covariance $\phivb_{t,n}$:
\begin{align}
\label{eq:tailAvgCovar}
\phivb_{t,n}&=\E{\thetavb\otimes\thetavb}\nonumber\\
&=\frac{1}{(n-t)^2}\sum_{j,l=t+1}^n\E{\thetav_j\otimes\thetav_l}\nonumber\\
&=\frac{1}{(n-t)^2}\sum_j\left(\sum_{l=t+1}^{j-1}\E{\thetav_j\otimes\thetav_l}+\E{\thetav_j\otimes\thetav_j}+\sum_{l=j+1}^n\E{\thetav_j\otimes\thetav_l}\right)\nonumber\\
&=\frac{1}{(n-t)^2}\sum_j\left(\sum_{l=t+1}^{j-1}\A^{j-l}\E{\thetav_l\otimes\thetav_l}+\E{\thetav_j\otimes\thetav_j}+\sum_{l=j+1}^n\E{\thetav_j\otimes\thetav_j}(\A\T)^{l-j}\right) \left(\mbox{ from~\eqref{eqn:theta-det}}\right) \nonumber\\
&=\frac{1}{(n-t)^2}\bigg(\sum_{l=t+1}^n\sum_{j=l+1}^n\A^{j-l}\E{\thetav_l\otimes\thetav_l} + \sum_{j=t+1}^n\E{\thetav_j\otimes\thetav_j}+\sum_{j=t+1}^n\sum_{l=j+1}^n\E{\thetav_j\otimes\thetav_j}(\A\T)^{l-j}\bigg)\nonumber\\
&=\frac{1}{(n-t)^2}\bigg(\sum_{j=t+1}^n\sum_{l=j+1}^n\A^{l-j}\E{\thetav_j\otimes\thetav_j} + \sum_{j=t+1}^n\E{\thetav_j\otimes\thetav_j}+\sum_{j=t+1}^n\sum_{l=j+1}^n\E{\thetav_j\otimes\thetav_j}(\A\T)^{l-j}\bigg)\nonumber\\
&=\frac{1}{(n-t)^2}\bigg(\sum_{j=t+1}^n(\eye-\A)^{-1}(\A-\A^{n+1-j})\E{\thetav_j\otimes\thetav_j} + \sum_{j=t+1}^n\E{\thetav_j\otimes\thetav_j}\nonumber\\&\quad\quad\quad\quad\quad\quad +\sum_{j=t+1}^n\E{\thetav_j\otimes\thetav_j}(\eye-\A\T)^{-1}(\A\T-(\A\T)^{n+1-j})\bigg)\nonumber\\
&=\frac{1}{(n-t)^2}\sum_{j=t+1}^n\bigg( \eyeT + (\eyeT-\AL)^{-1}(\AL-\AL^{n+1-j}) + (\eyeT-\AR\T)^{-1}(\AR\T-(\AR\T)^{n+1-j}) \bigg)\E{\thetav_j\otimes\thetav_j}\nonumber\\
&=\frac{1}{(n-t)^2}\sum_{j=t+1}^n\bigg( \eyeT + (\eyeT-\AL)^{-1}(\AL-\AL^{n+1-j}) + (\eyeT-\AR\T)^{-1}(\AR\T-(\AR\T)^{n+1-j}) \bigg)\phiv_j.
\end{align}
Note that the above recursion can be applied to obtain the covariance of the tail-averaged iterate for the bias ($\phivb_{t,n}^{\text{bias}}$) and variance ($\phivb_{t,n}^{\text{variance}}$) error, since the conditional expectation arguments employed in obtaining equation~\ref{eq:tailAvgCovar} are satisfied by both the recursion used in tracking the bias error (i.e. equation~\ref{eq:biasRec}) and the variance error (i.e. equation~\ref{eq:varianceRec}). This implies that,
\begin{align}
\label{eq:biasTailAvg1}
\phivb_{t,n}^{\text{bias}}&\defeq\frac{1}{(n-t)^2}\sum_{j=t+1}^n\bigg( \eyeT + (\eyeT-\AL)^{-1}(\AL-\AL^{n+1-j}) + (\eyeT-\AR\T)^{-1}(\AR\T-(\AR\T)^{n+1-j}) \bigg)\phiv_j^{\text{bias}}
\end{align}
\begin{align}
\label{eq:varianceTailAvg1}
\phivb_{t,n}^{\text{variance}}&\defeq\frac{1}{(n-t)^2}\sum_{j=t+1}^n\bigg( \eyeT + (\eyeT-\AL)^{-1}(\AL-\AL^{n+1-j}) + (\eyeT-\AR\T)^{-1}(\AR\T-(\AR\T)^{n+1-j}) \bigg)\phiv_j^{\text{variance}}
\end{align}
\subsection{Covariance of Bias error of the tail-averaged iterate}
\begin{proof}[Proof of Lemma~\ref{lem:average-covar-bias}]
To obtain the covariance of the bias error of the tail-averaged iterate, we first need to obtain $\phiv_j^{\text{bias}}$, which we will by unrolling the recursion of equation~\ref{eq:biasRec}:
\begin{align}
\label{eq:biasLP}
\thetav_k^{\text{bias}}&= \Ah_k\thetav_{k-1}^{\text{bias}}\nonumber\\
\implies \phiv_{k}^{\text{bias}} &=\E{\thetav_k^{\text{bias}}\otimes\thetav_k^{\text{bias}}}\nonumber\\
&=\E{\E{\thetav_k^{\text{bias}}\otimes\thetav_k^{\text{bias}}|\mathcal{F}_{k-1}}}\nonumber\\
&=\E{\E{\Ah_k\thetav_{k-1}^{\text{bias}}\otimes\thetav_{k-1}^{\text{bias}}\Ah_k\T|\mathcal{F}_{k-1}}}\nonumber\\
&=\BT\ \E{\thetav_{k-1}^{\text{bias}}\otimes\thetav_{k-1}^{\text{bias}}}=\BT\ \phiv_{k-1}^{\text{bias}}\nonumber\\
\implies\phiv_{k}^{\text{bias}}&=\BT^k\ \phiv_{0}^{\text{bias}}
\end{align}
Next, we recount the equation for the covariance of the bias of the tail-averaged iterate from equation~\ref{eq:biasTailAvg1}:
\begin{align*}
\phivb_{t,n}^{\text{bias}}&=\frac{1}{(n-t)^2}\sum_{j=t+1}^n\bigg( \eyeT + (\eyeT-\AL)^{-1}(\AL-\AL^{n+1-j}) + (\eyeT-\AR\T)^{-1}(\AR\T-(\AR\T)^{n+1-j}) \bigg)\phiv_j^{\text{bias}}
\end{align*}
Now, we substitute $\phiv_j^{\text{bias}}$ from equation~\ref{eq:biasLP}:
\begin{align}
\label{eq:biasTA}
\phivb_{t,n}^{\text{bias}}&=\frac{1}{(n-t)^2}\sum_{j=t+1}^n\bigg( \eyeT + (\eyeT-\AL)^{-1}(\AL-\AL^{n+1-j}) + (\eyeT-\AR\T)^{-1}(\AR\T-(\AR\T)^{n+1-j}) \bigg)\BT^j\phiv_0\nonumber\\
&=\frac{1}{(n-t)^2}\sum_{j=t+1}^n\bigg( \eyeT + (\eyeT-\AL)^{-1}\AL + (\eyeT-\AR\T)^{-1}\AR\T\bigg)\BT^j\phiv_0\nonumber\\
&\qquad-\frac{1}{(n-t)^2}\sum_{j=t+1}^n\bigg( (\eyeT-\AL)^{-1}\AL^{n+1-j} + (\eyeT-\AR\T)^{-1}(\AR\T)^{n+1-j} \bigg)\BT^j\phiv_0\nonumber\\
&=\underbrace{\frac{1}{(n-t)^2}\bigg( \eyeT + (\eyeT-\AL)^{-1}\AL + (\eyeT-\AR\T)^{-1}\AR\T\bigg)(\eyeT-\BT)^{-1}(\BT^{t+1}-\BT^{n+1})\phiv_0}_{\text{Leading order term}}\nonumber\\
&\qquad-\frac{1}{(n-t)^2}\sum_{j=t+1}^n\bigg( (\eyeT-\AL)^{-1}\AL^{n+1-j} + (\eyeT-\AR\T)^{-1}(\AR\T)^{n+1-j} \bigg)\BT^j\phiv_0.
\end{align}
There are two points to note here: (a) The second line consists of terms that constitute the lower-order terms of the bias. We will bound the summation by taking a supremum over $j$. (b) Note that the burn-in phase consisting of $t$ unaveraged iterations allows for a geometric decay of the bias, followed by the tail-averaged phase that allows for a sublinear rate of bias decay. 
\end{proof}
\subsection{Covariance of Variance error of the tail-averaged iterate}
\begin{proof}[Proof of Lemma~\ref{lem:average-covar-var}]
Before obtaining the covariance of the tail-averaged iterate, we note that  $\E{\thetav_j^{\text{variance}}}=0\ \forall\ j$. This can be easily seen since $\thetav_0^{\textrm{variance}}=0$ and $\E{\thetav_k^{\textrm{variance}}}=\A\E{\thetav_{k-1}^{\textrm{variance}}}$ (from equation~\ref{eq:condExpVar}).

Next, in order to obtain the covariance of the variance of the tail-averaged iterate, we first need to obtain $\phiv_j^{\text{variance}}$, and we will obtain this by unrolling the recursion of equation~\ref{eq:varianceRec}:
\begin{align}
\label{eq:varianceLP}
\thetav_k^{\text{variance}}&= \Ah_k\thetav_{k-1}^{\text{variance}}+\zetav_k\nonumber\\
\implies \phiv_{k}^{\text{variance}} &=\E{\thetav_k^{\text{variance}}\otimes\thetav_k^{\text{variance}}}\nonumber\\
&=\E{\E{\thetav_k^{\text{variance}}\otimes\thetav_k^{\text{variance}}|\mathcal{F}_{k-1}}}\nonumber\\
&=\E{\E{\Ah_k\thetav_{k-1}^{\text{variance}}\otimes\thetav_{k-1}^{\text{variance}}\Ah_k\T+\zetav_k\otimes\zetav_k|\mathcal{F}_{k-1}}}\nonumber\\
&=\BT\ \E{\thetav_{k-1}^{\text{variance}}\otimes\thetav_{k-1}^{\text{variance}}}+\Sigh=\BT\ \phiv_{k-1}^{\text{variance}}+\Sigh\nonumber\\
\implies\phiv_{k}^{\text{variance}}&=\sum_{j=0}^{k-1}\BT^j\ \Sigh\nonumber\\
&=(\eye-\BT)^{-1}(\eyeT-\BT^k)\Sigh
\end{align}
Note that the cross terms in the outer product computations vanish owing to the fact that $\E{\thetav_{k-1}^{\textrm{variance}}}=0\ \forall\ k$. We then recount the expression for the covariance of the variance error from equation~\ref{eq:varianceTailAvg1}:
\begin{align*}
\phivb_{t,n}^{\text{variance}}&=\frac{1}{(n-t)^2}\sum_{j=t+1}^n\bigg( \eyeT + (\eyeT-\AL)^{-1}(\AL-\AL^{n+1-j}) + (\eyeT-\AR\T)^{-1}(\AR\T-(\AR\T)^{n+1-j}) \bigg)\phiv_j^{\text{variance}}
\end{align*}
We will substitute the expression for $\phiv_j^{\text{variance}}$ from equation~\ref{eq:varianceLP}.
{\small
\begin{align}
\phivb_{t,n}^{\text{variance}}&=\frac{1}{(n-t)^2}\sum_{j=t+1}^n\bigg( \eyeT + (\eyeT-\AL)^{-1}(\AL-\AL^{n+1-j}) + (\eyeT-\AR\T)^{-1}(\AR\T-(\AR\T)^{n+1-j}) \bigg)(\eyeT-\BT)^{-1}(\eyeT-\BT^j)\Sigh\nonumber
\end{align}
}
Evaluating this summation, we have:
\begin{align}
\label{eq:varianceTA}
\phivb_{t,n}^{\text{variance}}&=\underbrace{\frac{1}{n-t}\big(\eyeT + (\eyeT-\AL)^{-1}\AL + (\eyeT-\AR\T)^{-1}\AR\T\big)(\eyeT-\BT)^{-1}\Sigh}_{\text{Leading order term}}\nonumber\\&-\frac{1}{(n-t)^2}\big((\eyeT-\AL)^{-2}(\AL-\AL^{n+1-t})+(\eyeT-\AR\T)^{-2}(\AR\T-(\AR\T)^{n+1-t})\big)(\eyeT-\BT)^{-1}\Sigh\nonumber\\&-\frac{1}{(n-t)^2}\big(\eyeT + (\eyeT-\AL)^{-1}\AL + (\eyeT-\AR\T)^{-1}\AR\T\big)(\eyeT-\BT)^{-2}(\BT^{t+1}-\BT^{n+1})\Sigh\nonumber\\&+\frac{1}{(n-t)^2}\sum_{j=t+1}^n\big((\eyeT-\AL)^{-1}\AL^{n+1-j}+(\eyeT-\AR\T)^{-1}(\AR\T)^{n+1-j}\big)(\eyeT-\BT)^{-1}\BT^j\Sigh
\end{align}

\end{proof}

Equations~\ref{eq:tailAvgCovarBound},~\ref{eq:biasTA},~\ref{eq:varianceTA} wrap up the proof of lemmas~\ref{lem:average-covar-bias},~\ref{lem:average-covar-var}.

The parameter error of the (tail-)averaged iterate can be obtained using a trace operator 
$\iprod{{\Bigcdot}}{{\Bigcdot}}$ to the tail-averaged iterate's covariance $\phivb_{t,n}$ with the matrix $\begin{bmatrix}\eye&0\\0&0\end{bmatrix}$, i.e. 
\begin{align*}
\|\bar{\x}_{t,n}-\xs\|_2^2=\iprod{\begin{bmatrix}\eye & 0\\ 0 & 0\end{bmatrix}}{\phivb_{t,n}}
\end{align*}
In order to obtain the function error, we note the following taylor expansion of the function $P(\Bigcdot)$ around the minimizer $\xs$:
\begin{align*}
P(\x)&=P(\xs) + \frac{1}{2}\ \|\x-\xs\|_{\nabla^2P(\xs)}^2\\
&=P(\xs) + \frac{1}{2}\ \|\x-\xs\|_{\Cov}^2
\end{align*}
This implies the excess risk can be obtained as:
\begin{align*}
P(\bar{\x}_{t,n})-P(\xs)&=\frac{1}{2}\cdot\iprod{\begin{bmatrix}\H & 0\\ 0 & 0\end{bmatrix}}{\phivb_{t,n}}\\
&\leq\iprod{\begin{bmatrix}\H & 0\\ 0 & 0\end{bmatrix}}{\phivb^{\text{bias}}_{t,n}}+\iprod{\begin{bmatrix}\H & 0\\ 0 & 0\end{bmatrix}}{\phivb^{\text{variance}}_{t,n}}
\end{align*}

\section{Useful lemmas}
\label{sec:commonLemmas}
In this section, we will state and prove some useful lemmas that will be helpful in the later sections.
\begin{lemma}\label{lem:com3}
	\begin{align*}
	{\left(\Id - \A\T\right)}^{-1} \begin{bmatrix}
	\Cov & \zero \\ \zero & \zero
	\end{bmatrix} = \frac{1}{\g-c\delta}\begin{bmatrix}-(c\eye-\g\Cov)&0\\(\eye-\delta\Cov)&0\end{bmatrix}
	\end{align*}
\end{lemma}
\begin{proof}
Since we assumed that $\Cov$ is a diagonal matrix (with out loss of generality), we note that $\A$ is a block diagonal matrix after a rearrangement of the co-ordinates (via an eigenvalue decomposition).

In particular, by considering the $j^{\text{th}}$ block (denoted by $\A_j$ corresponding to the $j^{\textrm{th}}$ eigenvalue $\lambda_j$ of $\Cov$), we have:
\begin{align*}
\eye-\A_j\T=\begin{bmatrix} 1 & c \\ -(1-\delta\lambda_j) & -(c-\g\lambda_j)\end{bmatrix}
\end{align*}
Implying that the determinant $\Det{\eye-\A_j\T}=(\g-c\delta)\lambda_j$, using which:
\begin{align}
\label{eq:ATInv}
(\eye-\A_j\T)^{-1}&=\frac{1}{(\g-c\delta)\lambda_j}\begin{bmatrix}-(c-\g\lambda_j)&-c\\1-\delta\lambda_j&1\end{bmatrix}
\end{align}
Thus, 
\begin{align*}
(\eye-\A_j\T)^{-1}\begin{bmatrix}\lambda_j&0\\0&0\end{bmatrix}&=\frac{1}{\g-c\delta}\begin{bmatrix}-(c-\g\lambda_j)&0\\(1-\delta\lambda_j)&0 \end{bmatrix}
\end{align*}
Accumulating the results of each of the blocks and by rearranging the co-ordinates, the result follows.
\end{proof}

\begin{lemma}\label{lem:com1}
	\begin{align*}
			\inv{\left(\Id - \A\T\right)} \begin{bmatrix}
		\Cov & \zero \\ \zero & \zero
		\end{bmatrix} \inv{\left(\Id - \A \right)} = \frac{1}{(\g-c\delta)^2}\bigg(\otimes_2\begin{bmatrix} -(c\eye-\g\Cov)\Cov^{-1/2}\\(\eye-\delta\Cov)\Cov^{-1/2}\end{bmatrix}\bigg)
	\end{align*}
\end{lemma}
\begin{proof}

In a similar manner as in lemma~\ref{lem:com3}, we decompose the computation into each of the eigen-directions and subsequently re-arrange the results. In particular, we note:
\begin{align*}
(\eye-\A_j)^{-1}=\frac{1}{(\g-c\delta)\lambda_j}\begin{bmatrix}-(c-\g\lambda_j)&(1-\delta\lambda_j)\\-c&1\end{bmatrix}
\end{align*}
Multiplying the above with the result of lemma~\ref{lem:com3}, we have:
\begin{align*}
(\eye-\A_j\T)^{-1}\begin{bmatrix}\lambda_j&0\\0&0\end{bmatrix}(\eye-\A_j)^{-1}=\frac{1}{(\g-c\delta)^2}\bigg(\otimes_{2}\begin{bmatrix}-(c-\g\lambda_j)\lambda_j^{-1/2}\\(1-\delta\lambda_j)\lambda_j^{-1/2}\end{bmatrix}\bigg)
\end{align*}
From which the statement of the lemma follows through a simple re-arrangement.

\end{proof}

\begin{lemma}\label{lem:com2}
	\begin{align*}
	{\left(\Id - \A\T\right)}^{-2} \A\T \begin{bmatrix}
	\Cov & \zero \\ \zero & \zero
	\end{bmatrix} = \frac{1}{(\g-c\delta)^2}\begin{bmatrix}\Cov^{-1}(-c(1-c)\eye-c\g\Cov)(\eye-\delta\Cov)&0\\\Cov^{-1}((1-c)\eye-c\delta\Cov)(\eye-\delta\Cov)&0\end{bmatrix}
	\end{align*}
\end{lemma}
\begin{proof}
In a similar argument as in previous two lemmas, we analyze the expression in each eigendirection of $\Cov$ through a rearrangement of the co-ordinates. Utilizing the expression of $\eye-\A_j\T$ from equation~\ref{eq:ATInv}, we get:
\begin{align}
\label{eq:intermediateEqn}
(\eye-\A_j\T)^{-1}\A_j\T\begin{bmatrix}\lambda_j&0\\0&0\end{bmatrix}
=\frac{1}{(\g-c\delta)}\begin{bmatrix}-c(1-\delta\lambda_j)&0\\(1-\delta\lambda_j)&0\end{bmatrix}
\end{align}
thus implying:
\begin{align*}
(\eye-\A_j\T)^{-2}\A_j\T\begin{bmatrix}\lambda_j&0\\0&0\end{bmatrix}
=\frac{(1-\delta\lambda_j)}{(\g-c\delta)^2\lambda_j}\begin{bmatrix}-c(1-c)-c\g\lambda_j&0\\(1-c)-c\delta\lambda_j&0\end{bmatrix}
\end{align*}
Rearranging the co-ordinates, the statement of the lemma follows.
\end{proof}

\begin{lemma}\label{lem:eig-A}
	The matrix $\A$ satisfies the following properties:
	\begin{enumerate}
		\item	Eigenvalues $q$ of $\A$ satisfy $\abs{q} \leq \sqrt{\alpha}$, and
		\item	$ \twonorm{\A^k}\leq 3\sqrt{2} \cdot k \cdot \alpha^{\frac{k-1}{2}} \; \forall \; k \geq 1$.
	\end{enumerate}
\end{lemma}
\begin{proof}
	Since the matrix is block-diagonal with $2\times2$ blocks, after a rearranging the coordinates, we will restrict ourselves to bounding the eigenvalues and eigenvectors of each of these $2\times 2$ blocks. Combining the results for different blocks then proves the lemma. Recall that $\A_j = \begin{bmatrix}
	0 & 1 - \delta \lambda_j \\ -c & 1+c - \g \lambda_j
	\end{bmatrix}$.
	
	\textbf{Part I}: Let us first prove the statement about the eigenvalues of $\A$. There are two scenarios here:
	\begin{enumerate}
		\item \emph{Complex eigenvalues}: In this case, both eigenvalues of $\A_j$ have the same magnitude which is given by $\sqrt{\det(\A_j)} = \sqrt{c(1-\delta \lambda_j)}\leq \sqrt{c} \leq \sqrt{\alpha}$.
		\item	\emph{Real eigenvalues}: Let $q_1$ and $q_2$ be the two real eigenvalues of $\A_j$. We know that $q_1+q_2 = \trace{\A_j} = 1 + c - \g\lambda_j > 0$ and $q_1 \cdot q_2 = \det(\A_j) > 0$. This means that $q_1 > 0$ and $q_2 > 0$.
		Now, consider the matrix $\G_j \eqdef (1-\beta) \eye - \A_j = \begin{bmatrix}
		 (1-\beta)  & - 1 + \delta \lambda_j \\ c & -1+{ (1-\beta) (1-\alpha)} + \g \lambda_j
		\end{bmatrix}$. We see that $( (1-\beta) -q_1)( (1-\beta) -q_2) = \det(\G_j) =  (1-\beta) (1-\alpha)\left( (1-\beta) -1\right) +  (1-\beta)  \left(\g-\alpha \delta\right)\lambda_j = (1-\beta)\left(1-\alpha\right)\left(\gamma \lambda_j - \beta\right) \geq 0$. This means that there are two possibilities: either $q_1, q_2 \geq  (1-\beta) $ or $q_1, q_2 \leq  (1-\beta) $. If the second condition is true, then we are done. If not, if $q_1, q_2 \geq  (1-\beta) $, then $\max_i q_i = \frac{\det(\A_j)}{\min_i q_i} \leq \frac{c(1-\delta \lambda_j)}{ (1-\beta) }\leq \alpha (1-\delta \lambda_j)$. Since $\sqrt{\alpha} \geq \alpha \geq 1-\beta$, this proves the first part of the lemma.
	\end{enumerate}

	\textbf{Part II}: Let $\A_j = \V \Q \V\T$ be the Schur decomposition of $\A_j$ where $\Q = \begin{bmatrix}
	q_1 & q \\ 0 & q_2
	\end{bmatrix}$ is an upper triangular matrix with eigenvalues $q_1$ and $q_2$ of $\A_j$ on the diagonal and $\V$ is a unitary matrix i.e., $\V \V\T = \V\T \V = \Id$. We first observe that $\abs{q} \leq \twonorm{\Q} \stackrel{(\zeta_1)}{=} \twonorm{\A_j} \leq \frob{\A_j} \leq \sqrt{6}$, where $(\zeta_1)$ follows from the fact that $\V$ is a unitary matrix. $\V$ being unitary also implies that $\A_j^k = \V \Q^k \V\T$. On the other hand, a simple proof via induction tells us that
	\begin{align*}
		\Q^k = \begin{bmatrix}
		q_1^k & q \left(\sum_{\ell = 1}^{k-1} q_1^{\ell}q_2^{k-\ell}\right) \\ 0 & q_2^k
		\end{bmatrix}.
	\end{align*}
	So, we have $\twonorm{\A_j^k} = \twonorm{\Q^k} \leq \frob{\Q^k} \leq \sqrt{3}k \abs{q} \max\left(\abs{q_1}^{k-1}, \abs{q_2}^{k-1}\right) \leq 3\sqrt{2} \cdot k \cdot \alpha^{\frac{k-1}{2}}$, where we used $\abs{q} \leq \sqrt{6}$ and $\max\left(\abs{q_1},\abs{q_2}\right) \leq \sqrt{\alpha}$.
\end{proof}
Finally, we state and prove the following lemma which is a relation between left and right multiplication operators.
\begin{lemma}
	\label{lem:lhs-psd-lemma}
	Let $\A$ be any matrix with $\AL=\A\otimes\eye$ and $\AR=\eye\otimes\A$ representing its left and right multiplication operators. Then, the following expression holds:
	\begin{align*}
	\bigg(\eyeT + (\eyeT-\AL)^{-1}\AL + (\eyeT-\AR\T)^{-1}\AR\T\bigg)(\eyeT-\AL\AR\T)^{-1}&=(\eyeT-\AL)^{-1}(\eyeT-\AR\T)^{-1}
	\end{align*}
\end{lemma}
\begin{proof}
	Let us assume that $\A$ can be written in terms of its eigen decomposition as $\A = \V\Lambda\V^{-1}$.
	Then the first claim is that $\eyeT,\AL,\AR$ are diagonalized by the same basis consisting of the eigenvectors of $\A$, i.e. in particular, the matrix of eigenvectors of $\eyeT,\AL,\AR$ can be written as $\V\otimes\V$. In particular, this implies, $\forall\ i,j\in\{1,2,...,d\}\times\{1,2,...,d\}$, we have, applying $\vv_i\otimes\vv_j$ to the LHS, we have:
	\begin{align*}
	&\bigg(\eyeT + (\eyeT-\AL)^{-1}\AL + (\eyeT-\AR\T)^{-1}\AR\T\bigg)(\eyeT-\AL\AR\T)^{-1} \vv_i\otimes\vv_j\\
	&=(1-\lambda_i\lambda_j)^{-1}\bigg(\eyeT + (\eyeT-\AL)^{-1}\AL + (\eyeT-\AR\T)^{-1}\AR\T\bigg)\vv_i\otimes\vv_j\\
	&=(1+\lambda_i(1-\lambda_i)^{-1}+\lambda_j(1-\lambda_j)^{-1})\cdot(1-\lambda_i\lambda_j)^{-1}\vv_i\otimes\vv_j
	\end{align*}
	Applying $\vv_i\otimes\vv_j$ to the RHS, we have:
	\begin{align*}
	&(\eyeT-\AL)^{-1}(\eyeT-\AR\T)^{-1}\vv_i\otimes\vv_j\\
	&=(1-\lambda_i)^{-1}(1-\lambda_j)^{-1}\vv_i\otimes\vv_j
	\end{align*}
	The next claim is that for any scalars (real/complex) $x,y~\ne 1$, the following statement holds implying the statement of the lemma:
	\begin{align*}
	(1+(1-x)^{-1}x+(1-y)^{-1}y)\cdot(1-xy)^{-1}=(1-x)^{-1}(1-y)^{-1}
	\end{align*}
\end{proof}

\begin{lemma}\label{lem:G-bound}
	Recall the matrix $\G$ defined as $\G \eqdef \begin{bmatrix} \Id & \frac{-\alpha}{1-\alpha}\Id \\ \zero & \frac{1}{1-\alpha}\Id \end{bmatrix} \begin{bmatrix} \Id &\zero \\ \zero & {\mu}\inv{\Cov} \end{bmatrix} \begin{bmatrix} \Id &\zero \\ \frac{-\alpha}{1-\alpha}\Id & \frac{1}{1-\alpha}\Id \end{bmatrix}$. The condition number of $\G$, $\kappa(\G)$ satisfies $\kappa(\G)\leq \frac{4\cnH}{\sqrt{1-\alpha^2}}$.
\end{lemma}
\begin{proof}
	Since the above matrix is block-diagonal after a rearrangement of coordinates, it suffices to compute the smallest and largest singular values of each block. Let $\lambda_i$ be the $i^{\textrm{th}}$ eigenvalue of $\Cov$. Let $\C \eqdef \begin{bmatrix} 1 & 0 \\ \frac{-\alpha}{1-\alpha} & \frac{1}{1-\alpha} \end{bmatrix}$ and consider the matrix $\G_i \eqdef \C \begin{bmatrix} 1 & 0 \\ 0 & \frac{\mu}{\lambda_i} \end{bmatrix} \C \T$. The largest eigenvalue of $\G_i$ is at most $\singmax{\C}^2$, while the smallest eigenvalue, $\singmin{\G_i}$ is at least $\frac{\mu}{\lambda_i} \cdot \singmin{\C}^2$. We obtain the following bounds on $\singmin{\C}$ and $\singmax{\C}$.
	\begin{align*}
		\singmax{\C} &\leq \frob{\C} \leq \frac{2}{\sqrt{1-\alpha^2}} \quad (\because \; \alpha \leq 1 ) \\
		\singmin{\C} &\geq \frac{\sqrt{\det\left(\C \C\T\right)}}{\frob{\C}} \geq \frac{1}{2}, \\ &\qquad \left(\because \det\left(\C\C\T\right) = \singmax{\C}^2 \singmin{\C}^2\right)
	\end{align*}
	where we used the computation that $\det\left(\C \C\T\right) = \frac{1}{1-\alpha}$. This means that $\singmin{\G_i} \geq \frac{\mu}{2\lambda_i}$ and $\singmax{\G_i} \leq \frac{2}{\sqrt{1-\alpha^2}}$. Combining all the blocks, we see that the condition number of $\G$ is at most $\frac{4\cnH}{\sqrt{1-\alpha^2}}$, proving the lemma.
\end{proof}

\section{Lemmas and proofs for bias contraction}
\label{sec:biasContraction}
\begin{proof}[Proof of Lemma~\ref{lem:main-bias}]
Let $\v \eqdef \frac{1}{1-\alpha}\left(\y - \alpha \x\right)$ and consider the following update rules corresponding to the noiseless versions of the updates in Algorithm~\ref{algo:TAASGD}:
\begin{align*}
	\xplus &= \y - \delta \Hhat (\y-\xs) \\
	\z &= \beta \y + (1-\beta) \v \\
	\vplus &= \z - \gamma \Hhat (\y-\xs) \\
	\yplus &= \alpha \xplus + (1-\alpha) \vplus,
\end{align*}
where $\Hhat \eqdef \a \a\T$ where $\a$ is sampled from the marginal on $\distr$. We first note that 
\begin{align*}
\E{\otimes_2 \begin{bmatrix} \xplus-\xs \\ \yplus-\xs \end{bmatrix}} &= \E{\Ah \bigg(\otimes_2 \begin{bmatrix} \x-\xs \\ \y-\xs \end{bmatrix}\bigg)\Ah\T} \\
&=\BT\bigg(\otimes_2 \begin{bmatrix} \x-\xs \\ \y-\xs \end{bmatrix}\bigg) 
\end{align*}
Letting $\Gtilde \eqdef \begin{bmatrix} \Id &\zero \\ \frac{-\alpha}{1-\alpha}\Id & \frac{1}{1-\alpha}\Id \end{bmatrix}$, we can verify that $\begin{bmatrix} \x-\xs \\ \v-\xs \end{bmatrix} = \Gtilde \begin{bmatrix} \x-\xs \\ \y-\xs \end{bmatrix}$, similarly $\begin{bmatrix} \xplus-\xs \\ \vplus-\xs \end{bmatrix} = \Gtilde \begin{bmatrix} \xplus-\xs \\ \yplus-\xs \end{bmatrix}$. Recall that $\G \defeq \Gtilde \T \begin{bmatrix} \Id &\zero \\ \zero & {\mu}\inv{\Cov} \end{bmatrix} \Gtilde$. With this notation in place, we prove the statement below, and substitute the values of $\cone,\ctwo,\cthree$ to obtain the statement of the lemma:
\begin{align}
\label{eq:biasContraction}
\iprod{\begin{bmatrix}\eye&0\\0&\mu\cdot\Hinv\end{bmatrix}}{\otimes_{2}\bigg(\begin{bmatrix}\xplus-\xs\\\vplus-\xs\end{bmatrix}\bigg)}\leq\left(1-\cthree\frac{\ctwo\sqrt{2\cone-\cone^2}}{\sqrt{\cnH\cnHh}}\right)\cdot\iprod{\begin{bmatrix}\eye&0\\0&\mu\cdot\Hinv\end{bmatrix}}{\otimes_{2}\bigg(\begin{bmatrix}\x-\xs\\\v-\xs\end{bmatrix}\bigg)}
\end{align}

To establish this result, let us define two quantities: $\e \eqdef {\twonorm{\x-\xs}^2}$, $\f \eqdef {\norm{\v-\xs}^2_{\Hinv}}$ and similarly, $\eplus \eqdef {\twonorm{\xplus-\xs}^2}$ and $\fplus \eqdef {\norm{\vplus-\xs}^2_{\Hinv}}$. The potential function we consider is $\e + \mu\cdot \f$. Recall that the parameters are chosen as:
\begin{align*}
\alpha = \frac{\sqrt{\cnH\cnHh}}{\ctwo\sqrt{2\cone-\cone^2}+\sqrt{\cnH\cnHh}},\ \  \beta = \cthree\frac{\ctwo\sqrt{2\cone-\cone^2}}{\sqrt{\cnH\cnHh}},\ \  \gamma = \ctwo\frac{\sqrt{2\cone-\cone^2}}{\mu\sqrt{\cnH\cnHh}}, \ \ \delta=\frac{\cone}{\infbound}
\end{align*}
with $\cone<1/2$, $\cthree=\frac{\ctwo\sqrt{2\cone-\cone^2}}{\cone}$, $\ctwo^2=\frac{\cfour}{2-\cone}$.
Consider $\eplus$ and employ the simple gradient descent bound:
\begin{align}
\label{eq:gd}
\eplus=\E{\twonorm{\xplus-\xs}^2}&=\E{\twonorm{\y-\delta\cdot\Hhat(\y-\xs)-\xs}^2}\nonumber\\
&=\E{\twonorm{\y-\xs}^2}-2\delta\cdot\E{\norm{\y-\xs}^2_{\H}}+\delta^2\E{\norm{\y-\xs}^2_{\M\eye}}\nonumber\\
&\leq\E{\twonorm{\y-\xs}^2}-2\delta\cdot\E{\norm{\y-\xs}^2_{\H}}+\infbound\delta^2\E{\norm{\y-\xs}^2_{\H}}\nonumber\\
&=\E{\twonorm{\y-\xs}^2}-\frac{2\cone-\cone^2}{\infbound}\E{\norm{\y-\xs}^2_{\H}}
\end{align}
Next, consider $\fplus$:
\begin{align}
\label{eq:1}
\fplus=\E{\norm{\vplus-\xs}^2_{\Hinv}}&=\E{\norm{\z-\gamma\Hhat(\y-\xs)-\xs}^2_{\Hinv}}\nonumber\\
&=\E{\norm{\z-\xs}^2_{\Hinv}}+\gamma^2\E{\norm{\y-\xs}^2_{\M\Hinv}}-2\gamma\E{\iprod{\z-\xs}{\y-\xs}}\nonumber\\
&\leq\E{\norm{\z-\xs}^2_{\Hinv}}+\gamma^2\cnHh\cdot\E{\norm{\y-\xs}^2_{\H}}-2\gamma\cdot\E{\iprod{\z-\xs}{\y-\xs}}
\end{align}
Where, we use the fact that $\M\Hinv\preceq\cnHh\H$, where $\cnHh$ is the {\em statistical} condition number.\linebreak
Consider $\E{\norm{\z-\xs}^2_{\Hinv}}$ and use convexity of the weighted $2-$norm to get:
\begin{align}
\label{eq:3}
\E{\norm{\z-\xs}^2_{\Hinv}}&\leq\beta\E{\norm{\y-\xs}^2_{\Hinv}}+(1-\beta)\E{\norm{\v-\xs}^2_{\Hinv}}\nonumber\\
&\leq\frac{\beta}{\mu} \E{\twonorm{\y-\xs}^2}+(1-\beta)\cdot \f
\end{align}
Next, consider $\E{\iprod{\z-\xs}{\y-\xs}}$, and first write $\z$ in terms of $\x$ and $\y$. This can be seen as two steps:
\begin{itemize}
\item $\v = \frac{1}{1-\alpha}\cdot\y-\frac{\alpha}{1-\alpha}\cdot\x$
\item $\z = \beta\y + (1-\beta)\v=\y + (1-\beta)(\v-\y)$. Then substituting $\v$ in terms of $\x$ and $\y$ as in the equation above, we get: $\z = \y + \left(\frac{\alpha\cdot(1-\beta)}{1-\alpha}\right)(\y-\x)$
\end{itemize}
Then, $\E{\iprod{\z-\xs}{\y-\xs}}$ can be written as:
\begin{align}
\label{eq:2}
\E{\iprod{\z-\xs}{\y-\xs}}&=\E{\twonorm{\y-\xs}^2}+\left(\frac{\alpha(1-\beta)}{1-\alpha}\right)\E{\iprod{\y-\x}{\y-\xs}}
\end{align}
Then, we note:
\begin{align*}
\E{\iprod{\y-\x}{\y-\xs}} &= \E{\twonorm{\y-\xs}^2}-\E{\iprod{\x-\xs}{\y-\xs}}\\
&\geq \E{\twonorm{\y-\xs}^2}-\frac{1}{2}\cdot\left(\E{\twonorm{\y-\xs}^2}+\E{\twonorm{\x-\xs}^2}\right) \\
&=\frac{1}{2}\cdot\left(\E{\twonorm{\y-\xs}^2}-\E{\twonorm{\x-\xs}^2}\right)
\end{align*}
Re-substituting in equation~\ref{eq:2}:
\begin{align}
\label{eq:4}
\E{\iprod{\z-\xs}{\y-\xs}}&\geq\left(1+\frac{1}{2}\cdot\frac{\alpha(1-\beta)}{1-\alpha}\right)\E{\twonorm{\y-\xs}^2}-\frac{1}{2}\cdot\frac{\alpha(1-\beta)}{1-\alpha}\E{\twonorm{\x-\xs}^2}\nonumber\\
&=\left(1+\frac{1}{2}\cdot\frac{\alpha(1-\beta)}{1-\alpha}\right)\E{\twonorm{\y-\xs}^2}-\frac{1}{2}\cdot\frac{\alpha(1-\beta)}{1-\alpha}\cdot \e
\end{align}
Substituting equations~\ref{eq:3},~\ref{eq:4} into equation~\ref{eq:1}, we get:
\begin{align*}
\mu\cdot \fplus&\leq \left(\beta-2\gamma\mu-\frac{\gamma\mu\alpha(1-\beta)}{1-\alpha}\right)\E{\twonorm{\y-\xs}^2}+\mu(1-\beta)\cdot \f \nonumber\\&+ \frac{\gamma\mu\alpha(1-\beta)}{1-\alpha}\cdot \e+\mu\gamma^2\cnHh\cdot\E{\norm{\y-\xs}^2_{\H}}
\end{align*}
Rewriting the guarantee on $\eplus$ as in equation~\ref{eq:gd}:
\begin{align*}
\eplus \leq \E{\twonorm{\y-\xs}^2}-\frac{2\cone-\cone^2}{\infbound}\cdot\E{\norm{\y-\xs}^2_{\H}}
\end{align*}
By considering $\eplus+\mu\cdot \fplus$, we see the following:
\begin{itemize}
\item The coefficient of $\E{\norm{\y-\xs}_{\H}^2}\leq 0$ by setting $\gamma = \ctwo\frac{\sqrt{2\cone-\cone^2}}{\mu\sqrt{\cnH\cnHh}}$, where, $0<\ctwo\leq1$, $\cnH = \frac{\infbound}{\mu}$.
\item Set $\frac{\gamma\mu\alpha}{1-\alpha}=1$ implying $\alpha = \frac{1}{1+\gamma\mu} = \frac{\sqrt{\cnH\cnHh}}{\ctwo\sqrt{2\cone-\cone^2}+\sqrt{\cnH\cnHh}}$
\end{itemize}
With these in place, we have the final result:
\begin{align*}
\eplus+\mu\cdot \fplus \leq (2\beta-2\gamma\mu)\E{\twonorm{\y-\xs}^2} + (1-\beta)\cdot(\e+\mu\cdot \f)
\end{align*}
In particular, setting $\beta=\cthree\gamma\mu=\cthree\frac{\ctwo\sqrt{2\cone-\cone^2}}{\sqrt{\cnH\cnHh}}$, we have a per-step contraction of $1-\beta$ which is precisely $1-\cthree\frac{\ctwo\sqrt{2\cone-\cone^2}}{\sqrt{\cnH\cnHh}}$, from which the claimed result naturally follows by substituting the values of $\cone,\ctwo,\cthree$.
\end{proof}

\begin{lemma}\label{lem:B-contraction}
	For any psd matrix $\Q \succeq 0$, we have:
	\begin{align*}
		\twonorm{\B^k \Q} \leq \frac{4\cnH}{\sqrt{1-\alpha^2}}\bigg(1-\left(\frac{ \ctwo \cthree \sqrt{2\cone-\cone^2} }{\sqrt{\cnH\cnS}}\right)\bigg)^k \twonorm{\Q}.
	\end{align*}
\end{lemma}
\begin{proof}
	From Lemma~\ref{lem:main-bias}, we conclude that $\iprod{\G}{\B^k \Q} \leq \bigg(1-\left(\frac{ \ctwo \cthree \sqrt{2\cone-\cone^2} }{\sqrt{\cnH\cnS}}\right)\bigg)^k\iprod{\G}{\Q}$. This implies that $\twonorm{\B^k \Q} \leq \bigg(1-\left(\frac{ \ctwo \cthree \sqrt{2\cone-\cone^2} }{\sqrt{\cnH\cnS}}\right)\bigg)^k \twonorm{\Q} \kappa(\G)$. Plugging the bound on $\kappa(\G)$ from Lemma~\ref{lem:G-bound} proves the lemma.
\end{proof}
\begin{lemma}\label{lem:bias-bound}
	We have: 
	\begin{align*}
		&\left(\Id - \D\right) \inv{\left(\Id - \B \right)} \B^{t+1} \left(\Id - \B^{n-t}\right) \thetat[0]\thetat[0]\T \\ &\;\preceq \frac{4\cnH}{\sqrt{1-\alpha^2}} \exp\left(- t \ctwo \cthree \sqrt{2\cone - \cone^2}/\sqrt{\cnH\cnS}\right) \norm{\thetat[0]}^2 \left(\eye +  \frac{\sqrt{\cnH\cnS}}{\ctwo\cthree\sqrt{2\cone-\cone^2}} (\infbound/\sigma^2) \Sighat \right).
	\end{align*}
\end{lemma}
\begin{proof}
	The proof follows from Lemma~\ref{lem:main-bias}. Since $\B = \D + \Rc$, we have $\left(\eyeT - \D\right) \inv{\left(\eyeT - \B \right)} = \eyeT + \Rc \inv{\left(\eyeT - \B \right)}$. Since $\Rc, \B$ and $\inv{(\eyeT - \B)}$ are all PSD operators, we have
	\begin{align*}
		&\left(\eyeT - \D\right) \inv{\left(\eyeT - \B \right)} \B^{t+1} \left(\eyeT - \B^{n-t}\right) \thetat[0]\thetat[0]\T \\
		&= \left(\eyeT + \Rc \inv{\left(\eyeT - \B \right)}\right) \B^{t+1} \left(\eyeT - \B^{n-t}\right) \thetat[0]\thetat[0]\T \\
		&\preceq \underbrace{\B^{t+1} \thetat[0]\thetat[0]\T}_{\S_1\eqdef } + \underbrace{\Rc \inv{\left(\eyeT - \B \right)} \B^{t+1} \thetat[0]\thetat[0]\T}_{\S_2\eqdef }.
	\end{align*}
	Applying Lemma~\ref{lem:B-contraction} with $\Q=\thetav_0\thetav_0\T$ tells us that $\S_1 \preceq \frac{4\cnH}{\sqrt{1-\alpha^2}} \exp\left(-t \ctwo \cthree \sqrt{2\cone - \cone^2}/\sqrt{\cnH\cnS}\right) \twonorm{\thetat[0]}^2 \Id$.
	For $\S_2$, we have
	\begin{align*}
		&\iprod{\G}{\inv{\left(\eyeT - \B \right)} \B^{t+1} \thetat[0]\thetat[0]\T} = \iprod{\G}{\sum_{j=t+1}^{\infty} \B^j \thetat[0]\thetat[0]\T} \\
		&\quad \leq \sum_{j=t+1}^{\infty} \bigg(1-\left(\frac{ \ctwo \cthree \sqrt{2\cone-\cone^2} }{\sqrt{\cnH\cnS}}\right)\bigg)^j \iprod{\G}{\thetat[0]\thetat[0]\T} \\
		&\quad \leq \frac{\sqrt{\cnH\cnS}}{\ctwo\cthree\sqrt{2\cone-\cone^2}} \exp\left(-t \ctwo \cthree \sqrt{2\cone - \cone^2}/\sqrt{4 \cnH\cnS}\right) \iprod{\G}{\thetat[0]\thetat[0]\T}.
	\end{align*}
	This implies
	\begin{align*}
	\inv{\left(\eyeT - \B \right)} \B^{t+1} \thetat[0]\thetat[0]\T \preceq  \kappa(\G) (\sqrt{\cnH\cnS}/(\ctwo\cthree\sqrt{2\cone-\cone^2}))  \exp\left(-t \ctwo \cthree \sqrt{2\cone - \cone^2}/\sqrt{4 \cnH\cnS}\right) \norm{\thetat[0]}^2 \eye,
	\end{align*}
	which tells us that 
	\begin{align*}
	\S_2 \preceq \kappa(\G) (\sqrt{\cnH\cnS}/(\ctwo\cthree\sqrt{2\cone-\cone^2}))  \exp\left(-t \ctwo \cthree \sqrt{2\cone - \cone^2}/\sqrt{4 \cnH\cnS}\right) \norm{\thetat[0]}^2 (\infbound/\sigma^2) \Sighat
	\end{align*} 
	Combining the bounds on $\S_1$ and $\S_2$, we obtain
	\begin{align*}
		&\left(\eyeT - \D\right) \inv{\left(\eyeT - \B \right)} \B^{t+1} \left(\eyeT - \B^{n-t}\right) \thetat[0]\thetat[0]\T \nonumber\\&\preceq \kappa(\G) \exp\left(- t \ctwo \cthree \sqrt{2\cone - \cone^2}/\sqrt{4 \cnH\cnS}\right) \norm{\thetat[0]}^2 \left(\eye + \frac{\sqrt{\cnH\cnS}}{\ctwo\cthree\sqrt{2\cone-\cone^2}} (\infbound/\sigma^2) \Sighat \right).
	\end{align*}
	Plugging the bound for $\kappa(\G)$ from Lemma~\ref{lem:G-bound} finishes the proof.
\end{proof}
\begin{corollary}\label{cor:bias-tail1}
For any psd matrix $\Q\succeq0$, we have:
	\begin{align*}
		\norm{\A^{n+1-j} \B^j \Q} &\leq \frac{12\sqrt{2}(n+1-j)\cnH}{\sqrt{1-\alpha^2}} \alpha^{\frac{n-j}{2}} \left(1-\frac{ \ctwo \cthree \sqrt{2\cone-\cone^2} }{\sqrt{\cnH\cnS}}\right)^j \twonorm{\Q}\\
		&\leq\frac{12\sqrt{2}(n+1-j)\cnH}{\sqrt{1-\alpha^2}} \alpha^{\frac{n-j}{2}} \exp\left(\frac{-j \ctwo \cthree \sqrt{2\cone-\cone^2} }{\sqrt{\cnH\cnS}}\right) \twonorm{\Q}.
	\end{align*}
\end{corollary}
\begin{proof}
	This corollary follows directly from Lemmas~\ref{lem:eig-A} and~\ref{lem:B-contraction} and using the fact that $1-x\leq e^{-x}$
\end{proof}

The following lemma bounds the total error of $\thetavb^{\textrm{bias}}$.
\begin{lemma}\label{lem:bound-bias}
	\begin{align*}
		&\iprod{\begin{bmatrix}
			\Cov & \zero \\ \zero & \zero
			\end{bmatrix} }{\E{\thetavb^{\textrm{bias}} \otimes \thetavb^{\text{bias}}}} \leq		\UC\cdot\frac{(\cnH\cnS)^{9/4}d\cnH}{(n-t)^2}\cdot\exp\bigg(-(t+1)\frac{\ctwo\cthree\sqrt{2\cone-\cone^2}}{\sqrt{\cnH\cnS}}\bigg)\cdot \big(P(\x_0)-P(\xs)\big) \nonumber\\&\qquad\qquad\qquad+  \UC\cdot(\cnH\cnS)^{5/4}d\cnH\cdot\exp\left(\frac{-n \ctwo \cthree \sqrt{2\cone-\cone^2} }{\sqrt{\cnH\cnS}}\right) \cdot \big(P(\x_0)-P(\xs)\big)
			\end{align*}	
			Where, $\UC$ is a universal constant.
\end{lemma}
\begin{proof}
	Lemma~\ref{lem:average-covar-bias} tells us that
	\begin{align}
		\E{\thetavb^{\textrm{bias}} \otimes \thetavb^{\text{bias}}} &= \frac{1}{(n-t)^2} \bigg( \eyeT + (\eyeT-\AL)^{-1}\AL + (\eyeT-\AR\T)^{-1}\AR\T\bigg) (\eyeT-\BT)^{-1}(\BT^{t+1}-\BT^{n+1}) \left(\thetat[0]\otimes \thetat[0]\right) \nonumber \\	 &\quad -\frac{1}{(n-t)^2}\sum_{j=t+1}^n\bigg( (\eyeT-\AL)^{-1}\AL^{n+1-j} + (\eyeT-\AR\T)^{-1}(\AR\T)^{n+1-j} \bigg)\BT^j \thetat[0] \otimes \thetat[0] . \label{eqn:bias-main-1}
	\end{align}
	We now use lemmas in this section to bound inner product of the two terms in the above expression with $\begin{bmatrix}\H&0\\0&0\end{bmatrix}$, i.e. we seek to bound,
	\begin{align}
		&\iprod{\begin{bmatrix}\H&0\\0&0\end{bmatrix}}{\E{\thetavb^{\textrm{bias}} \otimes \thetavb^{\text{bias}}}}\nonumber\\ &=\iprod{\begin{bmatrix}\H&0\\0&0\end{bmatrix}}{\frac{1}{(n-t)^2} \bigg( \eyeT + (\eyeT-\AL)^{-1}\AL + (\eyeT-\AR\T)^{-1}\AR\T\bigg) (\eyeT-\BT)^{-1}(\BT^{t+1}-\BT^{n+1}) \left(\thetat[0]\otimes \thetat[0]\right)}\nonumber\\
&+\iprod{\begin{bmatrix}\H&0\\0&0\end{bmatrix}}{-\frac{1}{(n-t)^2}\sum_{j=t+1}^n\bigg( (\eyeT-\AL)^{-1}\AL^{n+1-j} + (\eyeT-\AR\T)^{-1}(\AR\T)^{n+1-j} \bigg)\BT^j \thetat[0] \otimes \thetat[0]}\label{eqn:bias-main}
	\end{align}
	For the first term of equation~\ref{eqn:bias-main}, we have
	\begin{align}
		&\iprod{\begin{bmatrix}
			\Cov & \zero \\ \zero & \zero
			\end{bmatrix}}{\bigg( \eyeT + (\eyeT-\AL)^{-1}\AL + (\eyeT-\AR\T)^{-1}\AR\T\bigg) (\eyeT-\BT)^{-1}(\BT^{t+1}-\BT^{n+1}) \left(\thetat[0]\otimes \thetat[0]\right)} \nonumber \\ 
		&=\left\langle{\begin{bmatrix}
			\Cov & \zero \\ \zero & \zero
			\end{bmatrix}},\bigg( \eyeT + (\eyeT-\AL)^{-1}\AL + (\eyeT-\AR\T)^{-1}\AR\T\bigg) \inv{\left(\eyeT-\AL\AR\T\right)} \left(\eyeT-\AL\AR\T\right)\right. \nonumber \\ &\qquad \qquad \qquad \qquad \qquad \qquad \qquad \qquad \qquad \qquad \qquad \qquad \left.  (\eyeT-\BT)^{-1}(\BT^{t+1}-\BT^{n+1}) \left(\thetat[0]\otimes \thetat[0]\right)\right \rangle \nonumber \\ 
		&=\left\langle{\begin{bmatrix}
			\Cov & \zero \\ \zero & \zero
			\end{bmatrix}},(\eyeT-\AL)^{-1} (\eyeT-\AR\T)^{-1} \left(\eyeT-\AL\AR\T\right) (\eyeT-\BT)^{-1}(\BT^{t+1}-\BT^{n+1}) \left(\thetat[0]\otimes \thetat[0]\right)\right \rangle\nonumber \\ &\qquad \qquad \qquad \qquad \qquad \qquad \qquad \qquad \qquad \qquad \qquad \qquad \qquad\qquad\qquad\qquad\qquad \left(\mbox{using Lemma~\ref{lem:lhs-psd-lemma}}\right)\nonumber \\ 
		&= \iprod{(\eye-\A\T)^{-1} \begin{bmatrix}
			\Cov & \zero \\ \zero & \zero
			\end{bmatrix} (\eye-\A)^{-1} }{\left(\eyeT-\DT\right) (\eyeT-\BT)^{-1}(\BT^{t+1}-\BT^{n+1}) \left(\thetat[0]\otimes \thetat[0]\right)} \nonumber \\
		&\leq \frac{1}{(\g-c\delta)^2}\frac{4\cnH}{\sqrt{1-\alpha^2}} \exp\left(-(t+1) \ctwo \cthree \sqrt{2\cone - \cone^2}/\sqrt{\cnH\cnS}\right) \norm{\thetat[0]}^2\nonumber\\&\qquad\qquad\qquad\qquad\qquad\qquad \iprod{\bigg(\otimes_2\begin{bmatrix} -(c\eye-\g\Cov)\Cov^{-1/2}\\(\eye-\delta\Cov)\Cov^{-1/2}\end{bmatrix}\bigg)}{\eye + 2 \sqrt{\cnH\cnS} (\infbound/\sigma^2) \Sighat}.\nonumber
	\end{align}
	The two terms above can be bounded as
	\begin{align*}
		\iprod{\bigg(\otimes_2\begin{bmatrix} -(c\eye-\g\Cov)\Cov^{-1/2}\\(\eye-\delta\Cov)\Cov^{-1/2}\end{bmatrix}\bigg)}{\eye} &\leq 7\cdot \trace{\inv{\Cov}} \leq \frac{7d}{\mu} \mbox{ and,} \\
		2 \sqrt{\cnH\cnS} (\infbound/\sigma^2) \iprod{\bigg(\otimes_2\begin{bmatrix} -(c\eye-\g\Cov)\Cov^{-1/2}\\(\eye-\delta\Cov)\Cov^{-1/2}\end{bmatrix}\bigg)}{\Sighat} &=  2 \sqrt{\cnH\cnS} \infbound (\g - c\delta)^2 d.
	\end{align*}
	Combining the above and noting the fact that $2 \sqrt{\cnH\cnS} \infbound (\g - c\delta)^2 d<\frac{7 d}{\mu}$, we have
	\begin{align}
		&\iprod{\begin{bmatrix}
			\Cov & \zero \\ \zero & \zero
			\end{bmatrix}}{\bigg( \eyeT + (\eyeT-\AL)^{-1}\AL + (\eyeT-\AR\T)^{-1}\AR\T\bigg) (\eyeT-\BT)^{-1}(\BT^{t+1}-\BT^{n+1}) \left(\thetat[0]\otimes \thetat[0]\right)} \nonumber\\ 
		&\qquad \leq \frac{56 \cnH d }{\sqrt{1-\alpha^2}} \cdot \frac{\norm{\thetat[0]}^2}{\mu \left(\g - c \delta\right)^2} \cdot \exp\left(-(t+1) \ctwo \cthree \sqrt{2\cone - \cone^2}/\sqrt{\cnH\cnS}\right).\label{eqn:bias-11}
	\end{align}
	We now note the following facts:
	\begin{align*}
		&\frac{1}{1-\alpha}=\frac{\ctwo\sqrt{2\cone-\cone^2}}{\sqrt{\cnH\cnS}+\ctwo\sqrt{2\cone-\cone^2}}\leq\frac{2}{\sqrt{\cone\cfour}}\cdot\sqrt{\cnH\cnS}\\
		&\frac{1}{\g-c\delta}\leq\frac{1}{\gamma(1-\alpha)}\leq\frac{\mu}{(1-\alpha)^2}\leq\frac{4\cnS}{\cfour\delta}
	\end{align*}
	This implies, equation~\ref{eqn:bias-11} can be bounded as:
	\begin{align}
		&\iprod{\begin{bmatrix}
			\Cov & \zero \\ \zero & \zero
			\end{bmatrix}}{\bigg( \eyeT + (\eyeT-\AL)^{-1}\AL + (\eyeT-\AR\T)^{-1}\AR\T\bigg) (\eyeT-\BT)^{-1}(\BT^{t+1}-\BT^{n+1}) \left(\thetat[0]\otimes \thetat[0]\right)} \nonumber\\ 
&\qquad\leq\frac{1792}{(\cone\cfour)^{5/4}}\cdot\frac{(\cnH\cnS)^{9/4}d}{\delta\cfour}\cdot\exp\bigg(-(t+1)\frac{\ctwo\cthree\sqrt{2\cone-\cone^2}}{\sqrt{\cnH\cnS}}\bigg)\norm{\thetat[0]}^2\nonumber\\
&\qquad\leq\frac{1792}{(\cone\cfour)^{5/4}}\cdot\frac{(\cnH\cnS)^{9/4}d\cnH}{\cone\cfour}\cdot\exp\bigg(-(t+1)\frac{\ctwo\cthree\sqrt{2\cone-\cone^2}}{\sqrt{\cnH\cnS}}\bigg)\mu\norm{\thetat[0]}^2\nonumber\\
&\qquad\leq\frac{3584}{(\cone\cfour)^{5/4}}\cdot\frac{(\cnH\cnS)^{9/4}d\cnH}{\cone\cfour}\cdot\exp\bigg(-(t+1)\frac{\ctwo\cthree\sqrt{2\cone-\cone^2}}{\sqrt{\cnH\cnS}}\bigg)\cdot\big(P(\x_0)-P(\xs)\big)\nonumber\\
&\qquad\leq\UC\cdot(\cnH\cnS)^{9/4}d\cnH\cdot\exp\bigg(-(t+1)\frac{\ctwo\cthree\sqrt{2\cone-\cone^2}}{\sqrt{\cnH\cnS}}\bigg)\cdot\big(P(\x_0)-P(\xs)\big).\label{eqn:bias-1}
	\end{align}
	Where, $\UC$ is a universal constant.
	
	Consider now a term in the summation in the second term of~\eqref{eqn:bias-main}.
	\begin{align}
		&\iprod{\begin{bmatrix}
			\Cov & \zero \\ \zero & \zero
			\end{bmatrix} }{\bigg( (\eyeT-\AL)^{-1}\AL^{n+1-j} + (\eyeT-\AR\T)^{-1}(\AR\T)^{n+1-j} \bigg)\BT^j \left({\thetat[0] \otimes \thetat[0]}\right)}  \nonumber \\
		&= \iprod{(\eye-\A\T)^{-1}\begin{bmatrix}
			\Cov & \zero \\ \zero & \zero
			\end{bmatrix} }{\A^{n+1-j} \BT^j \left({\thetat[0] \otimes \thetat[0]}\right)} \nonumber\\&+ \iprod{\begin{bmatrix}
			\Cov & \zero \\ \zero & \zero
			\end{bmatrix} (\eye-\A)^{-1}}{\bigg(\BT^j \left({\thetat[0] \otimes \thetat[0]}\right)\bigg) (\A\T)^{n+1-j} } \nonumber \\
		&\leq 4 d \norm{(\eye-\A\T)^{-1}\begin{bmatrix}
			\Cov & \zero \\ \zero & \zero
			\end{bmatrix} } \norm{\A^{n+1-j} \BT^j \left({\thetat[0] \otimes \thetat[0]}\right)} \nonumber \\
		&\leq \frac{4d}{\g-c\delta} \norm{\begin{bmatrix}-(c\eye-\g\Cov)&0\\(\eye-\delta\Cov)&0\end{bmatrix}} \cdot \frac{12\sqrt{2}(n+1-j)\cnH}{\sqrt{1-\alpha^2}} \alpha^{\frac{n-j}{2}} \exp\left(\frac{-j \ctwo \cthree \sqrt{2\cone-\cone^2} }{\sqrt{\cnH\cnS}}\right) \norm{\thetat[0]}^2  \nonumber \\ & \qquad \qquad \qquad \qquad \qquad \left(\mbox{Lemma~\ref{lem:com1} and Corollary~\ref{cor:bias-tail1}}\right) \nonumber \\
		&\leq \frac{672 (n-t) d \cnH}{(\g - c\delta)\sqrt{1-\alpha^2}} \cdot \exp\left(\frac{-n \ctwo \cthree \sqrt{2\cone-\cone^2} }{\sqrt{\cnH\cnS}}\right) \cdot \norm{\thetat[0]}^2\nonumber\\
		&\leq\frac{5376}{(\cone\cfour)^{1/4}}\frac{(\cnH\cnS)^{5/4}d}{\delta\cfour}(n-t)\exp\left(\frac{-n \ctwo \cthree \sqrt{2\cone-\cone^2} }{\sqrt{\cnH\cnS}}\right) \cdot \norm{\thetat[0]}^2\nonumber\\		
		&\leq\frac{5376}{(\cone\cfour)^{1/4}}\frac{(\cnH\cnS)^{5/4}d\cnH}{\cone\cfour}(n-t)\exp\left(\frac{-n \ctwo \cthree \sqrt{2\cone-\cone^2} }{\sqrt{\cnH\cnS}}\right) \cdot \mu\norm{\thetat[0]}^2\nonumber\\		
		&\leq\frac{10752}{(\cone\cfour)^{1/4}}\frac{(\cnH\cnS)^{5/4}d\cnH}{\cone\cfour}(n-t)\exp\left(\frac{-n \ctwo \cthree \sqrt{2\cone-\cone^2} }{\sqrt{\cnH\cnS}}\right) \cdot \big(P(\x_0)-P(\xs)\big)\nonumber\\		
		&\leq\UC\cdot(\cnH\cnS)^{5/4}d\cnH\cdot(n-t)\exp\left(\frac{-n \ctwo \cthree \sqrt{2\cone-\cone^2} }{\sqrt{\cnH\cnS}}\right) \cdot \big(P(\x_0)-P(\xs)\big).\label{eqn:bias-2}
	\end{align}
	Where, $\UC$ is a universal constant.
	Plugging~\eqref{eqn:bias-1} and~\eqref{eqn:bias-2} into~\eqref{eqn:bias-main}, we obtain
	\begin{align*}
		&\iprod{\begin{bmatrix}
			\Cov & \zero \\ \zero & \zero
			\end{bmatrix} }{\E{\thetavb^{\textrm{bias}} \otimes \thetavb^{\text{bias}}}}\nonumber\\ &\leq			\UC\cdot\frac{(\cnH\cnS)^{9/4}d\cnH}{(n-t)^2}\cdot\exp\bigg(-(t+1)\frac{\ctwo\cthree\sqrt{2\cone-\cone^2}}{\sqrt{\cnH\cnS}}\bigg)\cdot \big(P(\x_0)-P(\xs)\big) \nonumber\\&\qquad\qquad\qquad+  \UC\cdot(\cnH\cnS)^{5/4}d\cnH\cdot\exp\left(\frac{-n \ctwo \cthree \sqrt{2\cone-\cone^2} }{\sqrt{\cnH\cnS}}\right) \cdot \big(P(\x_0)-P(\xs)\big)
			\end{align*}
	This proves the lemma.
\end{proof}

\section{Lemmas and proofs for Bounding variance error}
\label{sec:varianceContraction}
Before we prove lemma~\ref{lem:main-variance}, we recall old notation and introduce new notations that will be employed in these proofs.
\subsection{Notations}
We begin with by recalling that we track $\thetav_k=\begin{bmatrix}\x_k-\xs\\\y_k-\xs\end{bmatrix}$. Given $\thetav_k$, we recall the recursion governing the evolution of $\thetav_k$:
\begin{align}
\label{eq:simpleXYRec}
\thetav_{k+1}&=\begin{bmatrix}0&\eye-\delta\widehat{\H}_{k+1}\\-c\cdot\eye&(1+c)\eye-\g\cdot\widehat{\H}_{k+1}\end{bmatrix}\thetav_k+\begin{bmatrix}\delta\cdot\epsilon_{k+1}\av_{k+1}\\\g\cdot\epsilon_{k+1}\av_{k+1}\end{bmatrix}\nonumber\\
&=\widehat{\A}_{k+1}\thetav_{k}+\zetav_{k+1}
\end{align}
where, recall, $c=\alpha(1-\beta),\  \g=\alpha\delta+(1-\alpha)\gamma$, and $\widehat{\H}_{k+1}=\a_{k+1}\a_{k+1}\T$. Furthermore, we recall the following definitions, which will be heavily used in the following proofs:
\begin{align*}
\A&=\E{\Ah_{k+1}|\mathcal{F}_{k}}\\
\BT&=\E{\Ah_{k+1}\otimes\Ah_{k+1}|\mathcal{F}_{k}}\\
\Sigh&=\E{\zetav_{k+1}\otimes\zetav_{k+1}|\mathcal{F}_{k}}=\begin{bmatrix}\delta^2&\delta\cdot \g\\\delta\cdot \g&\g^2\end{bmatrix}\otimes\Sig\preceq\sigma^2\cdot\begin{bmatrix}\delta^2&\delta\cdot \g\\\delta\cdot \g&\g^2\end{bmatrix}\otimes\Cov
\end{align*}
We recall:
\begin{align*}
\RT&=\E{(\A-\Ah_{k+1})\otimes(\A-\Ah_{k+1})|\mathcal{F}_k}\\
\DT&=\A\otimes\A
\end{align*}
And the operators $\BT,\DT,\RT$ being related by:
\begin{align*}
\BT=\DT+\RT
\end{align*}
Furthermore, in order to compute the steady state distribution with the fourth moment quantities in the mix, we need to rely on the following re-parameterization of the update matrix $\Ah$:
\begin{align*}
\Ah&=\begin{bmatrix}0&\eye-\delta\widehat{\H}\\-c\cdot\eye&(1+c)\cdot\eye-\g\cdot\widehat{\H}\end{bmatrix}\\
&=\begin{bmatrix}0&\eye\\-c\cdot\eye&(1+c)\cdot\eye\end{bmatrix}+\begin{bmatrix}0&-\delta\cdot\widehat{\H}\\0&-\g\cdot\widehat{\H}\end{bmatrix}\\
&\defeq \V_1+\Vh_2
\end{align*}
This implies in particular:
\begin{align*}
\Ah\otimes\Ah&=(\V_1+\Vh_2)\otimes(\V_1+\Vh_2)\\
&=\V_1\otimes\V_1 + \V_1\otimes\Vh_2 + \Vh_2\otimes\V_1 + \Vh_2\otimes\Vh_2
\end{align*}
Note in particular, the fourth moment part resides in the operator $\Vh_2\otimes\Vh_2$. Terms such as $\V_1\otimes\V_1$ are deterministic, or terms such as $\V_1\otimes\Vh_2$ or $\Vh_2\otimes\V_1$ contain second moment quantities. Furthermore, note that the operator $\BT=\E{\Ah\otimes\Ah}$ where the expectation is taken with respect to a single random draw from the distribution $\mathcal{D}$.

Considering the expectation of $\Ah\otimes\Ah$ with respect to a single draw from the distribution $\mathcal{D}$, we have:
\begin{align}
\BT=\E{\Ah\otimes\Ah}&=\V_1\otimes\V_1 + \E{\V_1\otimes\Vh_2} + \E{\Vh_2\otimes\V_1} + \E{\Vh_2\otimes\Vh_2}\nonumber\\
&=\V_1\otimes\V_1 + \V_1\otimes\V_2 + \V_2\otimes\V_1 + \E{\Vh_2\otimes\Vh_2},\nonumber
\end{align}
where $\V_2\defeq\E{\Vh_2} =\begin{bmatrix}0&-\delta\cdot\H\\0&-\g\cdot\H\end{bmatrix}$. 

Finally, we let $\text{nr}$ and $\text{dr}$ to denote the numerator and denominator respectively.

\subsection{An exact expression for the stationary distribution}
Note that a key term appearing in the expression for covariance of the variance equation~\eqref{eq:varianceTA} is $\inv{\left(\eyeT - \BT\right)}\Sighat$. This is in fact nothing but the covariance of the error when we run accelerated SGD forever starting at $\xs$ (i.e., at steady state). This can be seen by
considering the base variance recursion using equation~\eqref{eq:simpleXYRec}:
\begin{align*}
\thetav_{k}&=\Ah_{k}\thetav_{k-1}+\zetav_k\nonumber\\
\implies \phiv_k&\defeq\E{\thetav_k\otimes\thetav_k}\nonumber\\
&=\E{\E{\bigg(\Ah_{k}\thetav_{k-1}\otimes\thetav_{k-1}\Ah_{k}\T+\zetav_k\otimes\zetav_k\bigg)|\mathcal{F}_{k-1}}}\nonumber\\
&=\E{\E{\bigg(\Ah_{k}\thetav_{k-1}\otimes\thetav_{k-1}\Ah_{k}\T\bigg)|\mathcal{F}_{k-1}}}+\Sigh\nonumber\\
&=\BT\cdot\E{\thetav_{k-1}\otimes\thetav_{k-1}}+\Sigh\nonumber\\
&=\BT\cdot\phiv_{k-1}+\Sigh
\end{align*}
This recursion on the covariance operator $\phiv_{k}$ can be unrolled until the start i.e. $k=0$ to yield:
\begin{align}
\label{eq:steadyStateExp}
\phiv_k&=\BT^k\phiv_0 + \sum_{l=0}^{k-1} \BT^l\cdot\Sigh\nonumber\\
&=(\eyeT-\BT)^{-1}(\eyeT-\BT^{k})\Sigh\quad\quad\quad(\because\ \phiv_0=0)\nonumber\\
\implies\phiv_{\infty}&=\lim_{k\to\infty}\phiv_k=(\eyeT-\BT)^{-1}\Sigh
\end{align}

\subsection{Computing the steady state distribution}

We now proceed to compute the stationary distribution.
Recall that
\begin{align*}
\BT &= \V_1\otimes\V_1 + \V_1\otimes\V_2 + \V_2\otimes\V_1 + \E{\Vh_2\otimes\Vh_2}\nonumber\\
\implies \eyeT-\BT &= \big(\eyeT-\V_1\otimes\V_1 - \V_1\otimes\V_2 - \V_2\otimes\V_1\big)-\E{\Vh_2\otimes\Vh_2}
\end{align*}
Where the expectation is over a single sample drawn from the distribution $\mathcal{D}$.
This implies in particular,
\begin{align}
\label{eq:ibinv}
&(\eyeT-\BT)^{-1}=\bigg(\big(\eyeT-\V_1\otimes\V_1 - \V_1\otimes\V_2 - \V_2\otimes\V_1\big)-\E{\Vh_2\otimes\Vh_2}\bigg)^{-1}\nonumber\\
&=\sum_{k=0}^{\infty}\bigg(\big(\eyeT-\V_1\otimes\V_1 - \V_1\otimes\V_2 - \V_2\otimes\V_1\big)^{-1}\E{\Vh_2\otimes\Vh_2}\bigg)^k\nonumber\\&\qquad\qquad\qquad\qquad\qquad\qquad\cdot\big(\eyeT-\V_1\otimes\V_1 - \V_1\otimes\V_2 - \V_2\otimes\V_1\big)^{-1}
\end{align}
Since $\Sighat \preceq \sigma^2\cdot\begin{bmatrix}\delta^2&\delta\cdot \g\\\delta\cdot \g&\g^2\end{bmatrix}\otimes\Cov$, and $\inv{\left(\eyeT - \BT\right)}$ is a PSD operator, the steady state distribution $\phiv_\infty$ is bounded by:
\begin{align}
\label{eq:phivInftyBound}
\phiv_\infty&=(\eyeT-\BT)^{-1}\Sigh \preceq \sigma^2 (\eyeT-\BT)^{-1} \left(\begin{bmatrix}\delta^2&\delta\cdot \g\\\delta\cdot \g&\g^2\end{bmatrix}\otimes\Cov\right) \nonumber\\
&=\sigma^2 \sum_{k=0}^{\infty}\bigg(\big(\eyeT-\V_1\otimes\V_1 - \V_1\otimes\V_2 - \V_2\otimes\V_1\big)^{-1}\E{\Vh_2\otimes\Vh_2}\bigg)^k \cdot \nonumber \\
&\qquad \qquad \quad \big(\eyeT-\V_1\otimes\V_1 - \V_1\otimes\V_2 - \V_2\otimes\V_1\big)^{-1} \left(\begin{bmatrix}\delta^2&\delta\cdot \g\\\delta\cdot \g&\g^2\end{bmatrix}\otimes\Cov\right).
\end{align}
Note that the Taylor expansion above is guaranteed to be correct if the right hand side is finite. We will understand bounds on the steady state distribution by splitting the analysis into the following parts:
\begin{itemize}
\item Obtain $\U\eqdef\big(\eyeT-\V_1\otimes\V_1 - \V_1\otimes\V_2 - \V_2\otimes\V_1\big)^{-1} \left(\begin{bmatrix}\delta^2&\delta\cdot \g\\\delta\cdot \g&\g^2\end{bmatrix}\otimes\Cov\right)$ (in section~\ref{ssec:secMomentEffects}).
\item Obtain bounds on $\E{\Vh_2\otimes\Vh_2}\U$ (in section~\ref{ssec:fourthMomentEffects})
\item Combine the above to obtain bounds on $\phiv_{\infty}$ (lemma~\ref{lem:main-variance}).
\end{itemize}
Before deriving these bounds, we will present some reasoning behind the validity of the upper bounds that we derive on the stationary distribution $\phiv_{\infty}$:
\begin{align}
\label{eq:phivInftyUpperBound}
\phiv_\infty&=(\eyeT-\BT)^{-1}\Sigh\nonumber\\
&\preceq\sigma^2 \sum_{k=0}^{\infty}\bigg(\big(\eyeT-\V_1\otimes\V_1 - \V_1\otimes\V_2 - \V_2\otimes\V_1\big)^{-1}\E{\Vh_2\otimes\Vh_2}\bigg)^k\U\quad(***)\nonumber\\
&=\sigma^2\U+\sigma^2\sum_{k=1}^{\infty}\bigg(\big(\eyeT-\V_1\otimes\V_1 - \V_1\otimes\V_2 - \V_2\otimes\V_1\big)^{-1}\E{\Vh_2\otimes\Vh_2}\bigg)^k\U\nonumber\\
&=\sigma^2\U+\sigma^2\sum_{k=0}^{\infty}\bigg(\big(\eyeT-\V_1\otimes\V_1 - \V_1\otimes\V_2 - \V_2\otimes\V_1\big)^{-1}\E{\Vh_2\otimes\Vh_2}\bigg)^k\nonumber\\
&\qquad\qquad\qquad\qquad\qquad\cdot\big(\eyeT-\V_1\otimes\V_1 - \V_1\otimes\V_2 - \V_2\otimes\V_1\big)^{-1}\E{\Vh_2\otimes\Vh_2}\U\nonumber\\
&=\sigma^2\U+\sigma^2(\eyeT-\BT)^{-1}\cdot\E{\Vh_2\otimes\Vh_2}\U\qquad\qquad\qquad\qquad\qquad\qquad(\text{using equation}~\ref{eq:ibinv}),
\end{align}
with $(***)$ following through using equation~\ref{eq:phivInftyBound} and through the definition of $\U$.
Now, with this in place, we clearly see that since $(\eyeT-\BT)^{-1}$ and $\E{\Vh_2\otimes\Vh_2}$ are PSD operators, we can upper bound right hand side to create valid PSD upper bounds on $\phiv_{\infty}$. In particular, in section~\ref{ssec:secMomentEffects}, we derive with equality what $\U$ is, and follow that up with computation of an upper bound on $\E{\Vh_2\otimes\Vh_2}\U$ in section~\ref{ssec:fourthMomentEffects}. Combining this will enable us to present a valid PSD upper bound on $\phiv_{\infty}$ owing to equation~\ref{eq:phivInftyUpperBound}.

\subsubsection{Understanding the second moment effects}\label{ssec:secMomentEffects}
This part of the proof deals with deriving the solution to:
\begin{align*}
\U&=\big(\eyeT-\V_1\otimes\V_1 - \V_1\otimes\V_2 - \V_2\otimes\V_1\big)^{-1} \left(\begin{bmatrix}\delta^2&\delta\cdot \g \\\delta\cdot \g &\g ^2\end{bmatrix}\otimes\Cov\right)
\end{align*}
This is equivalent to solving the (linear) equation:
\begin{align}
\label{eq:linEq}
\big(\eyeT-\V_1\otimes\V_1 - \V_1\otimes\V_2 - \V_2\otimes\V_1\big)\cdot\U&= \left(\begin{bmatrix}\delta^2&\delta\cdot \g \\\delta\cdot \g &\g ^2\end{bmatrix}\otimes\Cov\right) \nonumber\\
\implies\U-\V_1\U\V_1\T-\V_1\U\V_2\T-\V_2\U\V_1\T&= \left(\begin{bmatrix}\delta^2&\delta\cdot \g \\\delta\cdot \g &\g ^2\end{bmatrix}\otimes\Cov\right)
\end{align}
Note that all the known matrices above i.e., $\V_1, \V_2$ and $\Cov$ are all diagonalizable with respect to $\Cov$, and thus, the solution of this system can be computed in each of the eigenspaces $(\lambda_j,\u_j)$ of $\Cov$. This implies, in reality, we deal with matrices $\U^{(j)}$, one corresponding to each eigenspace. However, for this section, we will neglect the superscript on $\U$, since it is clear from context for the purpose of this section.
\begin{align*}
\V_1\U\V_1\T&=\begin{bmatrix}0&1\\-c&1+c\end{bmatrix}\begin{bmatrix}u_{11}&u_{12}\\u_{12}&u_{22}\end{bmatrix}\begin{bmatrix}0&-c\\1&1+c\end{bmatrix}\\
&=\begin{bmatrix}u_{22}&-c u_{12}+(1+c)u_{22}\\-c u_{12}+(1+c)u_{22}&c^2u_{11}-2c(1+c)u_{12}+(1+c)^2u_{22}\end{bmatrix}
\end{align*}
Next,
\begin{align*}
\V_1\U\V_2\T&=\begin{bmatrix}0&1\\-c&1+c\end{bmatrix}\begin{bmatrix}u_{11}&u_{12}\\u_{12}&u_{22}\end{bmatrix}\begin{bmatrix}0 &0\\-\delta&-\g \end{bmatrix}\lambda_j\\
&=\begin{bmatrix}u_{12}&u_{22}\\-cu_{11}+(1+c)u_{12}&-cu_{12}+(1+c)u_{22}\end{bmatrix}\begin{bmatrix}0 &0\\-\delta&-\g \end{bmatrix}\lambda_j\\
&=\begin{bmatrix}-\delta u_{22}&-\g  u_{22}\\-\delta(-c u_{12}+(1+c)u_{22})&-\g (-c u_{12}+(1+c)u_{22})\end{bmatrix}\lambda_j
\end{align*}
It follows that:
\begin{align*}
\V_2\U\V_1\T&=(\V_1\U\V_2\T)\T\\
&=\begin{bmatrix}-\delta u_{22}&-\delta(-c u_{12}+(1+c)u_{22})\\-\g  u_{22}&-\g (-c u_{12}+(1+c)u_{22})\end{bmatrix}\lambda_j
\end{align*}
Given all these computations, comparing the $(1,1)$ term on both sides of equation~\ref{eq:linEq}, we get:
\begin{align}
\label{eq:t11}
u_{11}&-u_{22}+2\delta\lambda_j u_{22}=\delta^2\lambda_j\nonumber\\
u_{11}&=u_{22}(1-2\delta\lambda_j)+\delta^2\lambda_j
\end{align}
Next, comparing $(1,2)$ term on both sides of equation~\ref{eq:linEq}, we get:
\begin{align}
\label{eq:t12}
&u_{12}-(-c u_{12}+ (1+c) u_{22}) +\g  \lambda_j u_{22} + \delta\lambda_j (-c u_{12} + (1+c) u_{22})=\delta\ \g \lambda_j\nonumber\\
&u_{12}-(1-\delta\lambda_j)(-c u_{12} + (1+c) u_{22})+\g \lambda_ju_{22}=\delta\ \g \lambda_j\nonumber\\
&(1+c(1-\delta\lambda_j))\cdot u_{12}+(\g \lambda_j-(1+c)(1-\delta\lambda_j))\cdot u_{22}=\delta\ \g \lambda_j
\end{align}
Finally, comparing the $(2,2)$ term on both sides of equation~\ref{eq:linEq}, we get:
\begin{align}
\label{eq:t22}
&u_{22}-(c^2u_{11}-2c(1+c)u_{12}+(1+c)^2u_{22})+2\g \lambda_j(-cu_{12}+(1+c)u_{22})=\g ^2\lambda_j\nonumber\\
&\implies-c^2u_{11}+(2c(1+c)-2c\g \lambda_j)u_{12}+(1-(1+c)^2+2(1+c)\g \lambda_j)u_{22}=\g ^2\lambda_j\quad(\text{from equation}~\ref{eq:t11})\nonumber\\
&\implies-c^2(u_{22}(1-2\delta\lambda_j)+\delta^2\lambda_j)+(2c(1+c)-2c\g \lambda_j)u_{12}+(1-(1+c)^2+2(1+c)\g \lambda_j)u_{22}=\g ^2\lambda_j\nonumber\\
&\implies(2c(1+c)-2c\g \lambda_j)u_{12}+(1-(1+c)^2-c^2(1-2\delta\lambda_j)+2(1+c)\g \lambda_j)u_{22}=(\g ^2+c^2\delta^2)\lambda_j\nonumber\\
&\implies2c((1+c)-\g \lambda_j)u_{12}+2((1+c)(\g \lambda_j-c)+\delta\lambda_jc^2)u_{22}=(\g ^2+c^2\delta^2)\lambda_j
\end{align}
Now, we note that equations~\ref{eq:t12},~\ref{eq:t22} are linear systems in two variables $u_{12}$ and $u_{22}$. Denoting the system in the following manner,
\begin{align*}
a_{11} u_{12} + a_{12} u_{22} = b_1\\
a_{21} u_{12} + a_{22} u_{22} = b_2
\end{align*}
For analyzing the variance error, we require $u_{22},u_{12}$:
\begin{align*}
u_{22}=\frac{b_1a_{21}-b_2a_{11}}{a_{12}a_{21}-a_{11}a_{22}}, \ u_{12}=\frac{b_1a_{22}-b_2a_{12}}{a_{11}a_{22}-a_{12}a_{21}}
\end{align*}
Substituting the values from equations~\ref{eq:t12} and~\ref{eq:t22}, we get:
\begin{align}
\label{eq:u22-1}
u_{22}&=\frac{2c\g \delta\bigg(1+c-\g \lambda_j\bigg)-(\g ^2+c^2\delta^2)\bigg(1+c(1-\delta\lambda_j)\bigg)}{2c\bigg( \big(1+c-\g \lambda_j\big)\cdot \big(\lambda_j \g -(1+c)(1-\delta\lambda_j)\big) \bigg)-2\cdot\bigg(\big(1+c-c\delta\lambda_j\big)\cdot\big((1+c)(\g \lambda_j-c)+\delta\lambda_jc^2\big) \bigg)}\cdot\lambda_j
\end{align}
\begin{align}
\label{eq:u12-1}
u_{12}=\frac{2\g \delta\bigg((1+c)(\g \lambda_j-c)+\delta\lambda_jc^2\bigg)-(\g ^2+c^2\delta^2)\bigg(\lambda_j\g -(1+c)(1-\delta\lambda_j)\bigg)}{2\bigg(\big(1+c-c\delta\lambda_j\big)\cdot\big((1+c)(\g \lambda_j-c)+\delta\lambda_jc^2\big) \bigg)-2c\bigg( \big(1+c-\g \lambda_j\big)\cdot \big(\lambda_j \g -(1+c)(1-\delta\lambda_j)\big) \bigg)}\cdot\lambda_j
\end{align}
\underline{\bf Denominator of $u_{22}$}:
Let us consider the denominator of $u_{22}$ (from equation~\ref{eq:u22-1}) to write it in a concise manner.
\begin{align*}
\text{dr}(u_{22})=2 \bigg(\ \big(1+c-\g \lambda_j\big)\cdot k_1\ -\ \big(1+c-c\delta\lambda_j\big)\cdot k_2 \bigg)
\end{align*}
with 
\begin{align*}
k_1&=c\cdot\big(\lambda_j \g -(1+c)(1-\delta\lambda_j)\big)\\
&=\big(c\lambda_j \g -(c+c^2)(1-\delta\lambda_j)\big)\\
&=\big(c\g \lambda_j-c-c^2+c\delta\lambda_j+c^2\delta\lambda_j\big)\\
k_2&=\big((1+c)(\g \lambda_j-c)+\delta\lambda_jc^2\big)\\
&=\big(\g \lambda_j-c+c\g \lambda_j-c^2+\delta\lambda_jc^2\big)
\end{align*}
Plugging in expressions for $\g =\alpha\delta+(1-\alpha)\gamma$ and $c=\alpha(1-\beta)$, in $\text{dr}(u_{22})$ we get:
\begin{align}
\label{eq:dr-u22-int1}
\text{dr}(u_{22})=2\cdot\bigg(\ \big(1+c-\alpha\delta\lambda_j\big)(k_1-k_2)-\lambda_j\cdot\big((1-\alpha)\gamma k_1 + \alpha\beta\delta k_2\big)\ \bigg)
\end{align}
Next, considering $k_1-k_2$, we have:
\begin{align}
\label{eq:dr-u22-p1}
k_1-k_2&=c\lambda_j \g -c-c^2+c\delta\lambda_j+c^2\delta\lambda_j-\g \lambda_j+c-c \g \lambda_j+c^2-c^2\delta\lambda_j\nonumber\\
&=(c\delta-\g )\lambda_j\nonumber\\
&=-(\alpha\beta\delta+\gamma(1-\alpha))\lambda_j
\end{align}
Next, considering $\gamma(1-\alpha)k_1+\alpha\beta\delta\ k_2$, we have:
\begin{align*}
&\gamma(1-\alpha)k_1+\alpha\beta\delta\ k_2\\
&=\gamma(1-\alpha)(c\lambda_j \g -c-c^2+c^2\delta\lambda_j+c\delta\lambda_j)\\
&+\alpha\beta\delta(c\lambda_j \g -c-c^2+c^2\delta\lambda_j+\g \lambda_j)\\
&=(\alpha\beta\delta+(1-\alpha)\gamma)(c\lambda_j \g -c-c^2+c^2\delta\lambda_j)+\lambda_j\delta(c\gamma(1-\alpha)+\alpha\beta \g )
\end{align*}
Consider $c\gamma(1-\alpha)+\alpha\beta \g $:
\begin{align*}
c\gamma(1-\alpha)+\alpha\beta \g &=\alpha(1-\beta)\gamma(1-\alpha)+\alpha\beta(\alpha\delta+(1-\alpha)\gamma)\\
&=\alpha(1-\beta)\gamma(1-\alpha)+\alpha\beta\gamma(1-\alpha)+\alpha^2\beta\delta\\
&=\alpha\gamma(1-\alpha)+\alpha^2\beta\delta\\
&=\alpha(\alpha\beta\delta+(1-\alpha)\gamma)
\end{align*}
Re-substituting this in the expression for $\gamma(1-\alpha)k_1+\alpha\beta\delta k_2$, we have:
\begin{align}
\label{eq:dr-u22-p2}
\gamma(1-\alpha)k_1+\alpha\beta\delta\ k_2&=(\alpha\beta\delta+(1-\alpha)\gamma)(c\lambda_j \g -c-c^2+c^2\delta\lambda_j)+\lambda_j\delta(c\gamma(1-\alpha)+\alpha\beta \g )\nonumber\\
&=(\alpha\beta\delta+(1-\alpha)\gamma)(c\lambda_j \g -c-c^2+c^2\delta\lambda_j)+\alpha\lambda_j\delta(\alpha\beta\delta+(1-\alpha)\gamma)\nonumber\\
&=(\alpha\beta\delta+(1-\alpha)\gamma)(c\lambda_j \g -c-c^2+c^2\delta\lambda_j+\alpha\lambda_j\delta)
\end{align}
Substituting equations~\ref{eq:dr-u22-p1},~\ref{eq:dr-u22-p2} into equation~\ref{eq:dr-u22-int1}, we have:
\begin{align}
\label{eq:dr-u22}
\text{dr}(u_{22})&=-2\lambda_j(\alpha\beta\delta+\gamma(1-\alpha))\cdot(1+c-\alpha\delta\lambda_j+c\lambda_j \g -c-c^2+c^2\delta\lambda_j+\alpha\delta\lambda_j)\nonumber\\
&=-2\lambda_j(\alpha\beta\delta+\gamma(1-\alpha))\cdot(1-c^2+c\lambda_j(\g +c\delta))
\end{align}
We note that the denominator of $u_{12}$ (in equation~\ref{eq:u12-1}) is just the negative of the denominator of $u_{22}$ as represented in equation~\ref{eq:dr-u22}.

\underline{\bf Numerator of $u_{22}$}:
We begin by writing out the numerator of $u_{22}$ (from equation~\ref{eq:u22-1}):
\begin{align}
\label{eq:nr-u22-int}
\text{nr}(u_{22})&=\lambda_j\cdot\bigg(2c\g \delta\big(1+c-\g \lambda_j\big)-(\g ^2+c^2\delta^2)\big(1+c(1-\delta\lambda_j)\big)\bigg)\nonumber\\
&=\lambda_j\cdot\bigg(2c\g \delta\big(1+c-\alpha\delta\lambda_j-\gamma(1-\alpha)\lambda_j\big)-(\g ^2+c^2\delta^2)\big(1+c-\alpha\delta\lambda_j+\alpha\beta\delta\lambda_j\big)\bigg)\nonumber\\
&=\lambda_j\cdot\bigg( -(1+c-\alpha\delta\lambda_j)(\g -c\delta)^2-\lambda_j\cdot\big(2c\g \delta\gamma(1-\alpha)+(\g ^2+(c\delta)^2)\alpha\beta\delta\big)\bigg)
\end{align}
We now consider $2c\g \delta\gamma(1-\alpha)+(\g ^2+(c\delta)^2)\alpha\beta\delta$:
\begin{align}
\label{eq:nr-u22-p1}
&2c\g \delta\gamma(1-\alpha)+(\g ^2+(c\delta)^2)\alpha\beta\delta\nonumber\\
&=2c\g \delta\cdot(\gamma(1-\alpha)+\alpha\beta\delta)+(\g ^2+(c\delta)^2-2c\g \delta)\alpha\beta\delta\nonumber\\
&=2c\g \delta(\g -c\delta)+(\g -c\delta)^2\alpha\beta\delta
\end{align}
Substituting equation~\ref{eq:nr-u22-p1} into equation~\ref{eq:nr-u22-int} and grouping common terms, we obtain:
\begin{align}
\label{eq:nr-u22}
\text{nr}(u_{22})&=\lambda_j\cdot\bigg( -(1+c-\alpha\delta\lambda_j)(\g -c\delta)^2-\lambda_j\cdot\big(2c\g \delta(\g -c\delta)+(\g -c\delta)^2\alpha\beta\delta\big)\bigg)\nonumber\\
&=\lambda_j\cdot\bigg( -(1+c-c\delta\lambda_j)(\g -c\delta)^2-\lambda_j\cdot\big(2c\g \delta(\g -c\delta)\big)\bigg)\nonumber\\
&=-\lambda_j\cdot\bigg( (1+c-c\delta\lambda_j)(\g -c\delta)^2+2c\g \delta\lambda_j(\g -c\delta)\bigg)
\end{align}
With this, we can write out the exact expression for $u_{22}$:
\begin{align}
\label{eq:u22}
u_{22}&=\frac{\big(1+c-c\delta\lambda_j\big)(\g -c\delta)+2c\g \delta\lambda_j}{2\cdot(1-c^2+c\lambda_j\cdot(\g +c\delta))}
\end{align}

\underline{\bf Numerator of $u_{12}$}:
We begin by rewriting the numerator of $u_{12}$ (from equation~\ref{eq:u12-1}):
\begin{align}
\label{eq:nr-u12-start}
\text{nr}(u_{12})=\lambda_j\cdot\bigg(2\g \delta\big((1+c)(\g \lambda_j-c)+\delta\lambda_jc^2\big)-(\g ^2+c^2\delta^2)\big(\lambda_j\g -(1+c)(1-\delta\lambda_j)\big)\bigg)
\end{align}
We split the simplification into two parts: one depending on $(1+c)$ and the other part representing terms that don't contain $(1+c)$. In particular, we consider the terms that do not carry a coefficient of $(1+c)$:
\begin{align}
\label{eq:nr-u12-p1}
&2\g \delta^2\lambda_j c^2-(\g ^2+c^2\delta^2)\cdot(\g \lambda_j)\nonumber\\
&=\g \lambda_j\cdot(2\delta^2c^2-\g ^2-\delta^2c^2)\nonumber\\
&=-\g \lambda_j\cdot(\g ^2-(c\delta)^2)
\end{align}
Next, we consider the other term containing the $(1+c)$ part:
\begin{align}
\label{eq:nr-u12-p2}
&(1+c)\cdot\bigg(2\g \delta\cdot(\g \lambda_j-c)\ +\ (\g ^2+(c\delta)^2)\cdot(1-\delta\lambda_j)\bigg)\nonumber\\
&=(1+c)\cdot\bigg(2\g ^2\delta\lambda_j-2\g \delta c+\g ^2+(c\delta)^2-\g ^2\delta\lambda_j-c^2\delta^3\lambda_j\bigg)\nonumber\\
&=(1+c)\cdot\bigg(\ (\g -c\delta)^2 + \delta\lambda_j\ (\g ^2-(c\delta)^2)\ \bigg)
\end{align}
Substituting equations~\ref{eq:nr-u12-p1},~\ref{eq:nr-u12-p2} into equation~\ref{eq:nr-u12-start}, we get:
\begin{align}
\label{eq:nr-u12}
\text{nr}(u_{12})&=\lambda_j\cdot\big((1+c)\delta\lambda_j(\g ^2-(c\delta)^2)+(1+c)(\g -c\delta)^2-\g \lambda_j(\g ^2-(c\delta)^2)\big)\nonumber\\
&=\lambda_j\cdot\big((1+c)(\g -c\delta)^2+\lambda_j\big((1+c)\delta-\g \big)\cdot(\g ^2-(c\delta)^2)\big)\nonumber\\
&=\lambda_j\cdot\big((1+c)(\g -c\delta)^2+\lambda_j\big(\delta-(\g -c\delta)\big)\cdot(\g ^2-(c\delta)^2)\big)\nonumber\\
&=\lambda_j\cdot\big((1+c)(\g -c\delta)^2+\delta\lambda_j\cdot(\g ^2-(c\delta)^2)-\lambda_j(\g +c\delta)(\g -c\delta)^2\big)\nonumber\\
&=\lambda_j\cdot\big((1+c-\lambda_j\cdot(\g +c\delta))\cdot(\g -c\delta)^2+\delta\lambda_j\cdot(\g ^2-(c\delta)^2)\big)
\end{align}
With which, we can now write out the expression for $u_{12}$:
\begin{align}
\label{eq:u12}
u_{12}&=\frac{\big(1+c-\lambda_j(\g +c\delta)\big)(\g -c\delta)+\delta\lambda_j(\g +c\delta)}{2\cdot(1-c^2+c\lambda_j\cdot(\g +c\delta))}
\end{align}
\underline{\bf Obtaining $u_{11}$:}
We revisit equation~\ref{eq:t11} and substitute $u_{22}$ from equation~\ref{eq:u22}:
\begin{align*}
u_{11}&=u_{22}(1-2\delta\lambda_j)+\delta^2\lambda_j\\
&=\frac{\big(1+c-c\delta\lambda_j\big)(\g -c\delta)+2c\g \delta\lambda_j}{2\cdot(1-c^2+c\lambda_j\cdot(\g +c\delta))}\cdot(1-2\delta\lambda_j)+\delta^2\lambda_j
\end{align*}
From which, we consider the numerator of $u_{11}$ and begin simplifying it:
\begin{align}
\label{eq:nr-u11}
\text{nr}(u_{11})&=(1+c-c\delta\lambda_j)(\g -c\delta)(1-2\delta\lambda_j)+2c\g \delta\lambda_j(1-2\delta\lambda_j)+2\delta^2\lambda_j(1-c^2+c\lambda_j(\g +c\delta))\nonumber\\
&=(1+c-c\delta\lambda_j)(\g -c\delta)(1-2\delta\lambda_j)+2\delta^2\lambda_j + 2c\delta\lambda_j(\g -c\delta)(1-\delta\lambda_j)\nonumber\\
&=(1+c+c\delta\lambda_j)(\g -c\delta)(1-\delta\lambda_j)+2\delta^2\lambda_j-\delta\lambda_j(1+c-c\delta\lambda_j)(\g -c\delta)\nonumber\\
&=(1+c+c\delta\lambda_j)(\g -c\delta)-2\delta\lambda_j(\g -c\delta)(1+c)+2\delta^2\lambda_j\nonumber\\
&=(1+c-c\delta\lambda_j)(\g -c\delta)-2\delta\lambda_j(\g -c\delta)+2\delta^2\lambda_j
\end{align}
This implies,
\begin{align}
\label{eq:u11}
u_{11} = \frac{(1+c-c\delta\lambda_j)(\g -c\delta)-2\delta\lambda_j(\g -c\delta)+2\delta^2\lambda_j}{2\cdot(1-c^2+c\lambda_j\cdot(\g +c\delta))}
\end{align}
\underline{\bf Obtaining a bound on $\U_{22}$}

For obtaining a PSD upper bound on $\U_{22}$, we will write out a sharp bound of $u_{22}$ in each eigen space:
\begin{align*}
u_{22}&=\frac{\big(1+c-c\lambda_j\delta\big)(\g -c\delta)+2c\g \delta\lambda_j}{2\cdot(1-c^2+c\lambda_j\cdot(\g +c\delta))}\nonumber\\
&=\frac{\big(1-c^2+c\lambda_j(\g +c\delta)+\g\lambda_j+(1+c)(c-\lambda_j(\g +c\delta))\big)(\g -c\delta)+2c\g \delta\lambda_j}{2\cdot(1-c^2+c\lambda_j\cdot(\g +c\delta))}\nonumber\\
&=\frac{\g -c\delta}{2}+\frac{\g\lambda_j(\g-c\delta)}{2\cdot(1-c^2+c\lambda_j\cdot(\g +c\delta))}+\frac{(1+c)(c-\lambda_j(\g +c\delta))(\g -c\delta)+2c\g \delta\lambda_j}{2\cdot(1-c^2+c\lambda_j\cdot(\g +c\delta))}\nonumber\\
&\leq\frac{\g -c\delta}{2}+\frac{\g\lambda_j(\g-c\delta)}{2\cdot(c\lambda_j\cdot(\g +c\delta))}+\frac{(1+c)(c-\lambda_j(\g +c\delta))(\g -c\delta)+2c\g \delta\lambda_j}{2\cdot(1-c^2+c\lambda_j\cdot(\g +c\delta))}\nonumber\\
&\leq\frac{\g -c\delta}{2}\cdot\frac{1+c}{c}+\frac{(1+c)(c-\lambda_j(\g +c\delta))(\g -c\delta)+2c\g \delta\lambda_j}{2\cdot(1-c^2+c\lambda_j\cdot(\g +c\delta))}\nonumber
\end{align*}
Let us consider bounding the numerator of the $2^{\text{nd}}$ term:
\begin{align*}
&(1+c)(c-\lambda_j(\g +c\delta))(\g -c\delta)+2c\g \delta\lambda_j\nonumber\\
&=c(1+c)(\g -c\delta)-(1+c)\lambda_j(\g +c\delta)(\g -c\delta)+2c\g \delta\lambda_j\nonumber\\
&=c(1+c)(\g -c\delta)-(1+c)\lambda_j(\g -c\delta)^2-2c\delta\lambda_j(1+c)(\g -c\delta)+2c\g \delta\lambda_j\nonumber\\
&=c(1+c)(\g -c\delta)-(1+c)\lambda_j(\g -c\delta)^2-2c\delta\lambda_j(1+c)(\g -c\delta)+2c(\g -c\delta)\delta\lambda_j+2c^2\delta^2\lambda_j\nonumber\\
&=c(1+c)(\g -c\delta)+2c^2\delta^2\lambda_j-(1+c)\lambda_j(\g -c\delta)^2-2c^2\delta\lambda_j(\g -c\delta)\nonumber\\
&\leq c(1+c)(\g -c\delta)+2c^2\delta^2\lambda_j
\end{align*}
Implying,
\begin{align*}
u_{22}&\leq\frac{\g -c\delta}{2}\cdot\frac{1+c}{c}+\frac{c(1+c)(\g -c\delta)+2c^2\delta^2\lambda_j}{2\cdot(1-c^2+c\lambda_j\cdot(\g +c\delta))}\nonumber\\
&\leq\frac{\g -c\delta}{2}\cdot\frac{1+c}{c}+\frac{c(1+c)(\g -c\delta)}{2\cdot(1-c^2+c\lambda_j\cdot(\g +c\delta))}+\frac{c^2\delta^2\lambda_j}{(1-c^2+c\lambda_j\cdot(\g +c\delta))}
\end{align*}
We will first upper bound the third term:
\begin{align*}
\frac{c^2\delta^2\lambda_j}{(1-c^2+c\lambda_j\cdot(\g +c\delta))}&\leq\frac{c\delta^2}{(\g +c\delta)}\nonumber\\
&=\frac{c\delta^2}{(\g -c\delta+2c\delta)}\nonumber\\
&\leq\frac{c\delta^2}{2c\delta}=\frac{\delta}{2}
\end{align*}
This implies,
\begin{align*}
u_{22}&\leq\frac{\g -c\delta}{2}\cdot\frac{1+c}{c}+\frac{\delta}{2}+\frac{c(1+c)(\g -c\delta)}{2\cdot(1-c^2+c\lambda_j\cdot(\g +c\delta))}\nonumber\\
&=\frac{\g -c\delta}{2}\cdot\frac{1+c}{c}+\frac{\delta}{2}+\frac{c^2(\g -c\delta)}{1-c^2+c\lambda_j\cdot(\g +c\delta)}+\frac{c(1-c)(\g -c\delta)}{2\cdot(1-c^2+c\lambda_j\cdot(\g +c\delta))}\nonumber\\
&\leq\frac{\g -c\delta}{2}\cdot\frac{1+c}{c}+\frac{\delta}{2}+\frac{c^2(\g -c\delta)}{1-c^2+c\lambda_j\cdot(\g +c\delta)}+\frac{c(1-c)(\g -c\delta)}{2\cdot(1-c^2)}\nonumber\\
&=\frac{\g -c\delta}{2}\cdot\frac{1+c}{c}+\frac{\delta}{2}+\frac{c^2(\g -c\delta)}{1-c^2+c\lambda_j\cdot(\g +c\delta)}+\frac{c(\g -c\delta)}{2\cdot(1+c)}\nonumber\\
&= \frac{\g -c\delta}{2}\cdot\bigg(\frac{1+c}{c}+\frac{c}{1+c}\bigg)+\frac{\delta}{2}+\frac{c^2(\g -c\delta)}{1-c^2+c\lambda_j\cdot(\g +c\delta)}\nonumber\\
&\leq \frac{\g -c\delta}{2}\cdot\frac{3}{c}+\frac{\delta}{2}+\frac{c^2(\g -c\delta)}{1-c^2+c\lambda_j\cdot(\g +c\delta)}\nonumber\\
&\leq \frac{\g -c\delta}{2}\cdot\frac{3}{c}+\frac{\delta}{2}+\frac{c(\g -c\delta)}{\lambda_j\cdot(\g +c\delta)}\nonumber\\
&= \frac{\g -c\delta}{2}\cdot\frac{3}{c}+\frac{\delta}{2}+\frac{c(\g -c\delta)}{\lambda_j\cdot(\g -c\delta+2c\delta)}\nonumber\\
&\leq \frac{\g -c\delta}{2}\cdot\frac{3}{c}+\frac{\delta}{2}+\frac{\g -c\delta}{2\lambda_j\delta}\nonumber\\
&\leq \frac{4}{c}\cdot\frac{\g -c\delta}{2\delta\lambda_j}+\frac{\delta}{2}
\end{align*}
Let us consider bounding $\frac{\g -c\delta}{2\delta\lambda_j}$ :
\begin{align*}
\frac{\g -c\delta}{2\delta\lambda_j}&=\frac{\alpha\beta\delta+\gamma(1-\alpha)}{2\delta\lambda_j}
\end{align*}
Substituting the values for $\alpha,\beta,\gamma,\delta$ applying $\frac{1}{1+\gamma\mu}\leq 1$, $c_3=\frac{c_2\sqrt{2c_1-c_1^2}}{c_1}$ and, $c_2^2=\frac{c_4}{2-c_1}$ with $0<c_4<1/6$ we get:
\begin{align*}
\frac{\g -c\delta}{2\delta\lambda_j}&\leq\bigg( \frac{c_3c_2\sqrt{2c_1-c_1^2}}{2}\sqrt{\frac{\cnS}{\cnH}} + \frac{c_2^2(2c_1-c_1^2)}{2c_1} \bigg)\cdot\frac{1}{\lambda_j\cnS}\nonumber\\
&\leq\bigg( \frac{c_3c_2\sqrt{2c_1-c_1^2}}{2} + \frac{c_2^2(2c_1-c_1^2)}{2c_1} \bigg)\cdot\frac{1}{\lambda_j\cnS}\nonumber\\
&=c_2^2(2-c_1)\cdot\frac{1}{\cnS\lambda_j}=c_4\cdot\frac{1}{\lambda_j\cnS}
\end{align*}
Which implies the bound on $u_{22}$:
\begin{align*}
u_{22}\leq\frac{4}{c}\cdot\frac{c_4}{\lambda_j\cnS}+\frac{\delta}{2}
\end{align*}
Now, consider the following bound on $1/c$:
\begin{align}
\label{eq:oneOverc}
\frac{1}{c}&=\frac{1}{\alpha(1-\beta)}\nonumber\\
&=1+\frac{(1+\cthree)\ctwo\sqrt{2\cone-\cone^2}}{\sqrt{\cnH\cnS}-\ctwo\cthree\sqrt{2\cone-\cone^2}}\nonumber\\
&\leq1+\frac{(1+\cthree)\ctwo\sqrt{2\cone-\cone^2}}{1-\ctwo\cthree\sqrt{2\cone-\cone^2}}\nonumber\\
&=1+\frac{\sqrt{\cone\cfour}+\cfour}{1-\cfour}\nonumber\\
&=\frac{1+\sqrt{\cone\cfour}}{1-\cfour}
\end{align}
Substituting values of $\cone$, $\cfour$ we have: $1/c\leq1.5$. This implies the following bound on $u_{22}$:
\begin{align}
\label{eq:u22b}
u_{22}\leq6\cdot\frac{c_4}{\lambda_j\cnS}+\frac{\delta}{2}
\end{align}
Alternatively, this implies that $\U_{22}$ can be upper bounded in a psd sense as:
\begin{align*}
\U_{22}\preceq \frac{6 c_4}{\cnS}\cdot\Hinv + \frac{\delta}{2}\cdot\eye
\end{align*}
\subsubsection{Understanding fourth moment effects}\label{ssec:fourthMomentEffects}
We wish to obtain a bound on:
\begin{align*}
\E{\Vh_2\otimes\Vh_2}\U&=\E{\Vh_2\U\Vh_2\T}\\
&=\begin{bmatrix}\delta^2&\delta\cdot \g \\\delta\cdot \g &\g ^2\end{bmatrix}\otimes\M\U_{22}
\end{align*}
We need to understand $\M\U_{22}$. 
\begin{align}
\label{eq:MU22}
\M\U_{22}&\preceq\frac{6c_4}{\cnHh}\cdot\M\Hinv+\frac{\delta}{2}\cdot\M\eye\nonumber\\
&\preceq(6c_4+\frac{\delta\infbound}{2})\cdot\H\nonumber\\
&=s\cdot\H
\end{align}
where, $s\eqdef (6c_4+\frac{\delta\infbound}{2})=23/30\leq \frac{4}{5}$. This implies (along with the fact that for any PSD matrices $\A,\mat{B},\C$, if $\A\preceq\mat{B}$, then, $\A\otimes\C\preceq\mat{B}\otimes\C$)),
\begin{align}
\label{eq:fourthMomentAccBound}
\E{\Vh_2\otimes\Vh_2}\U&\preceq s\cdot\begin{bmatrix}\delta^2&\delta\cdot \g \\\delta\cdot \g &\g ^2\end{bmatrix}\otimes\H \preceq \frac{4}{5}\cdot\begin{bmatrix}\delta^2&\delta\cdot \g \\\delta\cdot \g &\g ^2\end{bmatrix}\otimes\H.
\end{align}


This will lead us to obtaining a PSD upper bound on $\phiv_{\infty}$, i.e., the proof of lemma~\ref{lem:main-variance}

\begin{proof}
[Proof of lemma~\ref{lem:main-variance}]
We begin by recounting the expression for the steady state covariance operator $\phiv_{\infty}$ and applying results derived from previous subsections:
\begin{align}
\label{eq:stationaryDistBound}
\phiv_{\infty}&=(\eyeT-\BT)^{-1}\Sigh\nonumber\\
&\preceq\sigma^2\U+\sigma^2(\eyeT-\BT)^{-1}\cdot\E{\Vh_2\otimes\Vh_2}\U\quad(\text{from equation}~\ref{eq:phivInftyUpperBound})\nonumber\\
&\preceq\sigma^2\U+\frac{4}{5}\sigma^2(\eyeT-\BT)^{-1}\bigg(\begin{bmatrix}\delta^2&\delta\cdot \g \\\delta\cdot \g &\g ^2\end{bmatrix}\otimes\H\bigg)\quad(\text{from equation}~\ref{eq:fourthMomentAccBound})\nonumber\\
&=\sigma^2\U+\frac{4}{5}(\eyeT-\BT)^{-1}\Sigh\nonumber\\
&=\sigma^2\U+\frac{4}{5}\cdot\phiv_{\infty}\nonumber\\
\implies\phiv_{\infty}&\preceq5\sigma^2\U.
\end{align}
Now, given the upper bound provided by equation~\ref{eq:stationaryDistBound}, we can now obtain a (mildly) looser upper PSD bound on $\U$ that is more interpretable, and this is by providing an upper bound on $\U_{11}$ and $\U_{22}$ by considering their magnitude along each eigen direction of $\H$. In particular, let us consider the max of $u_{11}$ and $u_{22}$ along the $j^{th}$ eigen direction (as implied by equations~\ref{eq:u11},~\ref{eq:u22}):
\begin{align*}
\max(u_{11},u_{22}) &= \frac{(1+c-c\delta\lambda_j)(\g -c\delta)+2\delta^2\lambda_j}{2\cdot(1-c^2+c\lambda_j\cdot(\g +c\delta))}\\
&=\frac{(1+c-c\delta\lambda_j)(\g -c\delta)+2\delta^2\lambda_j}{2\cdot(1-c^2+c\lambda_j\cdot(\g +c\delta))}\\
&=\frac{(1+c-c\delta\lambda_j)(\g -c\delta)+2c\g\lambda_j-2c\g\lambda_j+2\delta^2\lambda_j}{2\cdot(1-c^2+c\lambda_j\cdot(\g +c\delta))}\\
&=u_{22}+\frac{-2c\g\lambda_j+2\delta^2\lambda_j}{2\cdot(1-c^2+c\lambda_j\cdot(\g +c\delta))}\\
&\leq \frac{6\cfour}{\cnS\lambda_j} + \frac{\delta}{2} + \frac{\delta^2\lambda_j-c\g\lambda_j}{(1-c^2+c\lambda_j\cdot(\g +c\delta))}\quad\text{(using equation~\ref{eq:u22b})}
\end{align*}
This implies, we can now consider upper bounding the term in the equation above and this will yield us the result:
\begin{align*}
\frac{\delta^2\lambda_j-c\g\lambda_j}{(1-c^2+c\lambda_j\cdot(\g +c\delta))}&\leq\frac{\delta^2\lambda_j-c\g\lambda_j}{c\lambda_j\cdot(\g +c\delta)}\\
&\leq\frac{\delta^2\lambda_j-c\g\lambda_j}{2c^2\delta\lambda_j}\\
&=\frac{\delta^2\lambda_j-c(\alpha\delta+\gamma(1-\alpha))\lambda_j}{2c^2\delta\lambda_j}\\
&\leq\frac{\delta^2\lambda_j-c\alpha\delta\lambda_j}{2c^2\delta\lambda_j}=\frac{1-c\alpha}{c^2}\cdot\frac{\delta}{2}\\
&=\big(\frac{1-c}{c^2}+\frac{1-\alpha}{c}\big)\cdot\frac{\delta}{2}\\
&=\big(\frac{(1+\cthree)(1-\alpha)}{c^2}+\frac{1-\alpha}{c}\big)\cdot\frac{\delta}{2}\\
&=\frac{1-\alpha}{c}\big(\frac{(1+\cthree)}{c}+1\big)\cdot\frac{\delta}{2}\\
&\leq3\frac{1-\alpha}{c}\cdot\frac{1}{c}\cdot\frac{\delta}{2}\\
&\leq3\frac{1-\alpha}{c}\cdot\frac{1+\sqrt{\cone\cfour}}{1-\cfour}\cdot\frac{\delta}{2}\\
&=3\cdot\frac{\cone\cthree}{\sqrt{\cnH\cnS}-\cone\cthree^2}\cdot\frac{1+\sqrt{\cone\cfour}}{1-\cfour}\cdot\frac{\delta}{2}\\
&\leq3\cdot\frac{\cone\cthree}{1-\cone\cthree^2}\cdot\frac{1+\sqrt{\cone\cfour}}{1-\cfour}\cdot\frac{\delta}{2}\\
&\leq (2/3) \frac{\delta}{2}
\end{align*}
Plugging this into the bound for $\max{u_{11},u_{22}}$, we get:
\begin{align*}
\max(u_{11},u_{22}) &\leq \frac{6\cfour}{\cnS\lambda_j} + (5/3)\frac{\delta}{2}=(2/3)\frac{1}{\cnS\lambda_j}+(5/3)\frac{\delta}{2}
\end{align*}
This implies the bound written out in the lemma, that is,
\begin{align*}
\U\preceq\begin{bmatrix}1&0\\0&1\end{bmatrix}\otimes\bigg(\frac{2}{3}\big(\frac{1}{\cnS}\Hinv\big)+\frac{5}{6}\cdot\big(\delta\ \eye\big)\bigg)
\end{align*}
\end{proof}

\begin{lemma}\label{lem:var-main-1}
\begin{align*}
&\iprod{\begin{bmatrix}\H&0\\0&0\end{bmatrix}}{\bigg(\eyeT+(\eyeT-\AL)^{-1}\AL+(\eyeT-\AR\T)^{-1}\AR\T\bigg)\cdot\E{\thetav_{l}\otimes\thetav_{l}}}\leq\nonumber\\&\iprod{\begin{bmatrix}\H&0\\0&0\end{bmatrix}}{\bigg(\eyeT+(\eyeT-\AL)^{-1}\AL+(\eyeT-\AR\T)^{-1}\AR\T\bigg)\cdot\E{\thetav_{\infty}\otimes\thetav_{\infty}}}\leq 5  \sigma^2 d.
\end{align*}
Where, $d$ is the dimension of the problem.
\label{lem:leadingOrderVar}
\end{lemma}

Before proving Lemma~\ref{lem:var-main-1}, we note that the sequence of expected covariances of the centered parameters $\E{\thetav_l\otimes\thetav_l}$ when initialized at the zero covariance (as in the case of variance analysis) only grows (in a psd sense) as a function of time and settles at the steady state covariance.
\begin{lemma}
Let $\thetav_0=0$. Then, by running the stochastic process defined using the recursion as in equation~\ref{eq:simpleXYRec}, the covariance of the resulting process is monotonically increasing until reaching the stationary covariance $\E{\thetav_{\infty}\otimes\thetav_{\infty}}$.
\end{lemma}
\begin{proof}
As long as the process does not diverge (as defined by spectral norm bounds of the expected update $\BT=\E{\hat{\A}\otimes\hat{\A}}$ being less than $1$), the first-order Markovian process converges geometrically to its unique stationary distribution $\thetav_{\infty}\otimes\thetav_{\infty}$.
In particular,
\begin{align*}
\E{\thetav_l\otimes\thetav_l}&=\BT\E{\thetav_{l-1}\otimes\thetav_{l-1}}+\Sigh\\
&=(\sum_{k=0}^{l-1}\BT^k)\Sigh
\end{align*}
Thus implying the fact that
\begin{align*}
\E{\thetav_l\otimes\thetav_l} = \E{\thetav_{l-1}\otimes\thetav_{l-1}}+\BT^{l-1}\Sigh
\end{align*}
Owing to the PSD'ness of the operators in the equation above, the lemma concludes with the claim that $\E{\thetav_l\otimes\thetav_l}\succeq\E{\thetav_{l-1}\otimes\thetav_{l-1}}$
\end{proof}

Given these lemmas, we are now in a position to prove lemma~\ref{lem:leadingOrderVar}.
\begin{proof}[Proof of Lemma~\ref{lem:leadingOrderVar}]

\begin{align}
\label{eq:simpVarMain}
&\iprod{\begin{bmatrix}\H&0\\0&0\end{bmatrix}}{\bigg(\eyeT+(\eyeT-\AL)^{-1}\AL+(\eyeT-\AR\T)^{-1}\AR\T\bigg)
\cdot\E{\thetav_{l}\otimes\thetav_{l}}}\nonumber\\
&=\iprod{\begin{bmatrix}\H&0\\0&0\end{bmatrix}}{\bigg(\eyeT+(\eyeT-\AL)^{-1}\AL+(\eyeT-\AR\T)^{-1}\AR\T\bigg)(\eyeT-\AL\AR\T)^{-1}(\eyeT-\AL\AR\T)\cdot\E{\thetav_{l}\otimes\thetav_{l}}}\nonumber\\
&=\iprod{\begin{bmatrix}\H&0\\0&0\end{bmatrix}}{\bigg((\eyeT-\AL)^{-1}(\eyeT-\AR\T)^{-1}\bigg)(\eyeT-\AL\AR\T)\cdot\E{\thetav_{l}\otimes\thetav_{l}}}\quad \left(\mbox{using Lemma~\ref{lem:lhs-psd-lemma}}\right)\nonumber\\
&=\iprod{\bigg((\eyeT-\AL\T)^{-1}(\eyeT-\AR)^{-1}\bigg)\begin{bmatrix}\H&0\\0&0\end{bmatrix}}{(\eyeT-\AL\AR\T)\cdot\E{\thetav_{l}\otimes\thetav_{l}}}\nonumber\\
&=\iprod{(\eye-\A\T)^{-1}\begin{bmatrix}\H&0\\0&0\end{bmatrix}(\eye-\A)^{-1}}{(\eyeT-\AL\AR\T)\cdot\E{\thetav_{l}\otimes\thetav_{l}}}\nonumber\\
&=\iprod{(\eye-\A\T)^{-1}\begin{bmatrix}\H&0\\0&0\end{bmatrix}(\eye-\A)^{-1}}{(\eyeT-\DT)\cdot\E{\thetav_{l}\otimes\thetav_{l}}}\nonumber\\
&=\frac{1}{(\g -c\delta)^2}\iprod{\bigg(\otimes_2\begin{bmatrix} -(c\eye-\g \Cov)\Cov^{-1/2}\\(\eye-\delta\Cov)\Cov^{-1/2}\end{bmatrix}\bigg)}{(\eyeT-\DT)\cdot\E{\thetav_l\otimes\thetav_l}}\quad\text{(using lemma~\ref{lem:com1})}\nonumber\\
&=\frac{1}{(\g -c\delta)^2}\iprod{\bigg(\otimes_2\begin{bmatrix} -(c\eye-\g \Cov)\Cov^{-1/2}\\(\eye-\delta\Cov)\Cov^{-1/2}\end{bmatrix}\bigg)}{(\eyeT-\DT)(\eyeT-\BT)^{-1}(\eyeT-\BT^l)\Sigh}\nonumber\\
&=\frac{1}{(\g -c\delta)^2}\iprod{\bigg(\otimes_2\begin{bmatrix} -(c\eye-\g \Cov)\Cov^{-1/2}\\(\eye-\delta\Cov)\Cov^{-1/2}\end{bmatrix}\bigg)}{(\eyeT-\BT+\RT)(\eyeT-\BT)^{-1}(\eyeT-\BT^l)\Sigh}\nonumber\\
&=\frac{1}{(\g -c\delta)^2}\iprod{\bigg(\otimes_2\begin{bmatrix} -(c\eye-\g\Cov)\Cov^{-1/2}\\(\eye-\delta\Cov)\Cov^{-1/2}\end{bmatrix}\bigg)}{\Sigh-\BT^l\Sigh+\RT(\eyeT-\BT)^{-1}\Sigh-\RT(\eyeT-\BT)^{-1}\BT^l\Sigh}\nonumber\\
&\leq\frac{1}{(\g -c\delta)^2}\iprod{\bigg(\otimes_2\begin{bmatrix} -(c\eye-\g \Cov)\Cov^{-1/2}\\(\eye-\delta\Cov)\Cov^{-1/2}\end{bmatrix}\bigg)}{\Sigh+\sigma^2\RT\cdot(5 \U)}
\end{align}
So, we need to understand $\RT\U$:
\begin{align*}
\RT\U&=\mathbb{E}\bigg( \begin{bmatrix}0 & \delta\cdot(\H-\av\av\T)\\0 & \g \cdot(\H-\av\av\T)\end{bmatrix}\U\begin{bmatrix}0 & 0\\\delta\cdot(\H-\av\av\T) & \g \cdot(\H-\av\av\T)\end{bmatrix} \bigg)\\
&=\begin{bmatrix}\delta^2&\delta\cdot \g \\\delta\cdot \g &\g ^2\end{bmatrix}\otimes \E{(\H-\av\av\T)\U_{22}(\H-\av\av\T)}\\
&=\begin{bmatrix}\delta^2&\delta\cdot \g \\\delta\cdot \g &\g ^2\end{bmatrix}\otimes \big(\M-\HL\HR\big)\U_{22}\\
&\preceq \begin{bmatrix}\delta^2&\delta\cdot \g \\\delta\cdot \g &\g ^2\end{bmatrix}\otimes \M\U_{22}\\
&\preceq \frac{4}{5}\cdot\begin{bmatrix}\delta^2&\delta\cdot \g \\\delta\cdot \g &\g ^2\end{bmatrix}\otimes \H\quad(\text{from equation}~\ref{eq:MU22}).
\end{align*}
Then,
\begin{align}
&\iprod{\begin{bmatrix}\H&0\\0&0\end{bmatrix}}{\bigg(\eyeT+(\eyeT-\AL)^{-1}\AL+(\eyeT-\AR\T)^{-1}\AR\T\bigg)\cdot\thetav_{l}\otimes\thetav_{l}}\nonumber\\
&\leq\frac{1}{(\g -c\delta)^2}\iprod{\bigg(\otimes_2\begin{bmatrix} -(c\eye-\g \Cov)\Cov^{-1/2}\\(\eye-\delta\Cov)\Cov^{-1/2}\end{bmatrix}\bigg)}{\Sigh+\sigma^2\RT\cdot(5\U)}\qquad\qquad(\text{from equation}~\ref{eq:simpVarMain})\nonumber\\
&\leq\frac{5\sigma^2}{(\g -c\delta)^2} \cdot \iprod{\bigg(\otimes_2\begin{bmatrix} -(c\eye-\g \Cov)\Cov^{-1/2}\\(\eye-\delta\Cov)\Cov^{-1/2}\end{bmatrix}\bigg)}{\begin{bmatrix}\delta^2&\delta\cdot \g \\\delta\cdot \g &\g ^2\end{bmatrix}\otimes \H}\nonumber\\
&=\frac{5}{(\g -c\delta)^2} \cdot d\ \sigma^2\cdot (\g-c\delta)^2\nonumber\\
&=5\sigma^2d.
\end{align}
\end{proof}

\begin{lemma}
\label{lem:var1N2bound}
\begin{align*}
&\bigg\vert\iprod{\begin{bmatrix}\Cov&0\\0&0\end{bmatrix}}{ \bigg((\eyeT-\AL)^{-2}\AL+(\eyeT-\AR\T)^{-2}\AR\T\bigg)\phiv_{\infty}}\bigg\vert\leq\UC\cdot\sigma^2 d\sqrt{\cnH\cnS}
\end{align*}
Where, $\UC$ is a universal constant.
\end{lemma}
\begin{proof}
We begin by noting the following while considering the left side of the above expression:
\begin{align*}
&\iprod{\begin{bmatrix}\Cov&0\\0&0\end{bmatrix}}{ \bigg((\eyeT-\AR\T)^{-2}\AR\T+(\eyeT-\AL)^{-2}\AL\bigg)\phiv_{\infty}}\\
&=\iprod{\begin{bmatrix}\Cov&0\\0&0\end{bmatrix}\A(\eye-\A)^{-2}+(\eye-\A\T)^{-2}\A\T\begin{bmatrix}\Cov&0\\0&0\end{bmatrix}}{\phiv_{\infty}}
\end{align*}
The inner product above is a sum of two terms, so let us consider the first of the terms:
\begin{align*}
&\iprod{\begin{bmatrix}\Cov&0\\0&0\end{bmatrix}\A(\eye-\A)^{-2}}{\phiv_{\infty}}\\
&=\text{Tr}\bigg((\eye-\A\T)^{-2}\A\T\begin{bmatrix}\H^{1/2}\\0\end{bmatrix}\begin{bmatrix}\H^{1/2}&0\end{bmatrix}\phivi \bigg)\\
&=\text{Tr}\bigg(\bigg(\begin{bmatrix}\H^{1/2}\\0\end{bmatrix}\T\phivi (\eye-\A\T)^{-2}\A\T\begin{bmatrix}\H^{1/2}\\0\end{bmatrix}\bigg)\bigg) \\
&=\sum_{j=1}^{d}\text{Tr}\bigg(\bigg(\begin{bmatrix}\lambda_j^{1/2}\\0\end{bmatrix}\T (\phivi)_j (\eye-\A_j\T)^{-2}\A_j\T\begin{bmatrix}\lambda_j^{1/2}\\0\end{bmatrix}\bigg)\bigg) \\
&=\sum_{j=1}^{d}\text{Tr}\bigg(\bigg(\begin{bmatrix}\lambda_j^{1/2}\\0\end{bmatrix}\T(\phivih)_j\bigg)\cdot\bigg((\phivih)_j\T(\eye-\A_j\T)^{-2}\A_j\T\begin{bmatrix}\lambda_j^{1/2}\\0\end{bmatrix}\bigg)\bigg),
\end{align*}
where $(\phivi)_j$ is the $2\times 2$ block of $\phivi$ corresponding to the $j^{\textrm{th}}$ eigensubspace of $\Cov$,  $(\phivih)_j$ denotes the $2\times 2d$ submatrix (i.e., $2$ rows) of $\phivih$ corresponding to the $j^{\textrm{th}}$ eigensubspace and $\A_j$ denotes the $j^{\textrm{th}}$ diagonal block of $\A$. Note that $(\phivih)_j (\phivih)_j \T = (\phivi)_j$.
It is very easy to observe that the second term in the dot product can be written in a similar manner, i.e.:
\begin{align*}
&\iprod{(\eye-\A\T)^{-2}\A\T\begin{bmatrix}\Cov&0\\0&0\end{bmatrix}}{\phiv_{\infty}}\\
&=\sum_{j=1}^{d}\text{Tr}\bigg(\bigg((\phivih)_j\T\begin{bmatrix}\lambda_j^{1/2}\\0\end{bmatrix}\bigg)\cdot\bigg(\begin{bmatrix}\lambda_j^{1/2}\\0\end{bmatrix}\T\A_j(\eye-\A_j)^{-2} (\phivih)_j \bigg)\bigg)
\end{align*}
So, essentially, the expression in the left side of the lemma can be upper bounded by using Cauchy-Shwartz inequality:
\begin{align}
\label{eq:lotpmain1}
&\text{Tr}\bigg(\bigg(\begin{bmatrix}\lambda_j^{1/2}\\0\end{bmatrix}\T (\phivih)_j \bigg)\cdot\bigg((\phivih)_j \T (\eye-\A_j\T)^{-2}\A_j\T\begin{bmatrix}\lambda_j^{1/2}\\0\end{bmatrix}\bigg)\bigg)\nonumber\\&+\text{Tr}\bigg(\bigg((\phivih)_j \T \begin{bmatrix}\lambda_j^{1/2}\\0\end{bmatrix}\bigg)\cdot\bigg(\begin{bmatrix}\lambda_j^{1/2}\\0\end{bmatrix}\T\A_j(\eye-\A_j)^{-2}(\phivih)_j \bigg)\bigg)\nonumber\\
&\quad\quad\quad\quad\leq2\adanorm{\begin{bmatrix}\lambda_j^{1/2}\\0\end{bmatrix}}_{(\phivi)_j}\cdot\adanorm{(\eye-\A_j\T)^{-2}\A_j\T\begin{bmatrix}\lambda_j^{1/2}\\0\end{bmatrix}}_{(\phivi)_j}
\end{align}
The advantage with the above expression is that we can now begin to employ psd upper bounds on the covariance of the steady state distribution $\phivi$ and provide upper bounds on the expression on the right hand side. In particular, we employ the following bound provided by the taylor expansion that gives us an upper bound on $\phivi$:
\begin{align*}
\phivi\defeq\begin{bmatrix}\hat{\U}_{11}&\hat{\U}_{12} \\ \hat{\U}_{12}\T&\hat{\U}_{22}\end{bmatrix} \preceq 5\sigma^2 \U = 5\sigma^2 \begin{bmatrix} \U_{11} & \U_{12}\\\U_{12}\T&\U_{22}\end{bmatrix}\quad(\text{using equation~\ref{eq:stationaryDistBound}})
\end{align*}
This implies in particular that $(\phivi)_j \preceq 5 \sigma^2 \U_j$ for every $j\in[d]$ and hence, for any vector $\adanorm{\a}_{(\phivi)_j}\leq \sqrt{5\sigma^2}\adanorm{\a}_{\U_j}$. The important property of the matrix $\U$ that serves as a PSD upper bound is that it is diagonalizable using the basis of $\Cov$, thus allowing us to bound the computations in each of the eigen directions of $\Cov$. 
\begin{align}
\label{eq:p1}
&\adanorm{(\eye-\A_j\T)^{-2}\A_j\T\begin{bmatrix}\lambda_j^{1/2}\\0\end{bmatrix}}_{(\phivi)_j}\nonumber\\
=&\sqrt{\begin{bmatrix}\lambda_j^{1/2}&0\end{bmatrix}\A_j(\eye-\A_j)^{-2}(\phivi)_j(\eye-\A_j\T)^{-2}\A_j\T\begin{bmatrix}\lambda_j^{1/2}\\0\end{bmatrix}}\nonumber\\
\leq&\sqrt{5\sigma^2\begin{bmatrix}\lambda_j^{1/2}&0\end{bmatrix}\A_j(\eye-\A_j)^{-2}\U_j(\eye-\A_j\T)^{-2}\A_j\T\begin{bmatrix}\lambda_j^{1/2}\\0\end{bmatrix}}\nonumber\\
=&\sqrt{5\sigma^2}\adanorm{(\eye-\A_j\T)^{-2}\A_j\T\begin{bmatrix}\lambda_j^{1/2}\\0\end{bmatrix}}_{\U_j}
\end{align}
So, let us consider $\begin{bmatrix}\lambda_j^{1/2}&0\end{bmatrix}\A_j(\eye-\A_j)^{-2}$ and write out the following series of equations:
\begin{align*}
\begin{bmatrix}\lambda_j^{1/2}&0\end{bmatrix}\A_j&=\begin{bmatrix}0&\sqrt{\lambda_j}(1-\delta\lambda_j)\end{bmatrix}\\
\eye-\A_j&=\begin{bmatrix}1&-(1-\delta\lambda_j)\\c&-(c-\g \lambda_j)\end{bmatrix}\\
\det(\eye-\A_j)&=(\g -c\delta)\lambda_j\\
(\eye-\A_j)^{-1}&=\frac{1}{(\g -c\delta)\lambda_j}\begin{bmatrix}-(c-\g \lambda_j)&1-\delta\lambda_j\\-c&1\end{bmatrix}\\
\implies\begin{bmatrix}\lambda_j^{1/2}&0\end{bmatrix}\A_j(\eye-\A_j)^{-1}&=\frac{\sqrt{\lambda_j}(1-\delta\lambda_j)}{(\g -c\delta)\lambda_j}\begin{bmatrix}-c&1\end{bmatrix}\\
\implies\begin{bmatrix}\lambda_j^{1/2}&0\end{bmatrix}\A_j(\eye-\A_j)^{-2}&=\frac{\sqrt{\lambda_j}(1-\delta\lambda_j)}{((\g -c\delta)\lambda_j)^2}\begin{bmatrix}-c(1-c+\g \lambda_j)&1-c+c\delta\lambda_j\end{bmatrix}\\
&=\frac{\sqrt{\lambda_j}(1-\delta\lambda_j)}{((\g -c\delta)\lambda_j)^2}\cdot\bigg((1-c+c\delta\lambda_j)\begin{bmatrix}-c&1\end{bmatrix} - c\lambda_j(\g -c\delta)\begin{bmatrix}1&0\end{bmatrix}   \bigg)
\end{align*}
This implies,
\begin{align}
\label{eq:lotp2}
\adanorm{(\eye-\A_j\T)^{-2}\A_j\T\begin{bmatrix}\lambda_j^{1/2}\\0\end{bmatrix}}_{\U_j}\leq\frac{\sqrt{\lambda_j}(1-\delta\lambda_j)}{((\g -c\delta)\lambda_j)^2}\cdot(1-c+c\delta\lambda_j)\adanorm{\begin{bmatrix}-c\\1\end{bmatrix}}_{\U_j}+\frac{c\sqrt{\lambda_j}(1-\delta\lambda_j)}{((\g -c\delta)\lambda_j)}\adanorm{\begin{bmatrix}1\\0\end{bmatrix}}_{\U_j}
\end{align}
Next, let us consider $\adanorm{\begin{bmatrix}-c\\1\end{bmatrix}}^2_{\U_j}$:
\begin{align*}
\adanorm{\begin{bmatrix}-c\\1\end{bmatrix}}^2_{\U_j}&=c^2u_{11}+u_{22}-2c\cdot u_{12}
\end{align*}
Note that $u_{11},u_{12},u_{22}$ share the same denominator, so let us evaluate the numerator $\text{nr}(c^2 u_{11}-2cu_{12}+u_{22})$. For this, we have, from equations~\ref{eq:u11},~\ref{eq:u12},~\ref{eq:u22} respectively:
Furthermore, 
\begin{align*}
\text{nr}(u_{11})&=(1+c-c\delta\lambda_j)(\g -c\delta)-2\delta\lambda_j(\g -c\delta)+2\delta^2\lambda_j\\
\text{nr}(u_{12})&=(1+c-\lambda_j(\g +c\delta))(\g -c\delta)+\delta\lambda_j(\g +c\delta)\\
\text{nr}(u_{22})&=(1+c-c\delta\lambda_j)(\g -c\delta)+2c\g \delta\lambda_j
\end{align*}
Combining these, we have:
\begin{align*}
&c^2\text{nr}(u_{11})-2c\cdot\text{nr}(u_{12})+\text{nr}(u_{22})\\
=&\big((1+c-c\delta\lambda_j)(1-c)^2+2c\g \lambda_j\big)(\g -c\delta)-2c^2\delta\lambda_j(\g -c\delta)\\
=&\big((1+c-c\delta\lambda_j)(1-c)^2(\g -c\delta)\big)+2c\lambda_j(\g -c\delta)^2
\end{align*}
Implying,
\begin{align*}
\adanorm{\begin{bmatrix}-c\\1\end{bmatrix}}^2_{\U_j}&=\frac{(1+c-c\delta\lambda_j)(1-c)^2(\g -c\delta)+2c\lambda_j(\g -c\delta)^2}{1-c^2+c\lambda_j(\g +c\delta)}
\end{align*}
In a very similar manner,
\begin{align*}
\adanorm{\begin{bmatrix}1\\0\end{bmatrix}}^2_{\U_j}&=u_{11}\\
&=\frac{(1+c-c\delta\lambda_j)(\g -c\delta)-2\delta\lambda_j(\g -c\delta)+2\delta^2\lambda_j}{1-c^2+c\lambda_j(\g +c\delta)}
\end{align*}
This implies, plugging into equation~\ref{eq:lotp2}
\begin{align}
\label{eq:lotp3}
&\adanorm{(\eye-\A_j\T)^{-2}\A_j\T\begin{bmatrix}\lambda_j^{1/2}\\0\end{bmatrix}}_{\U_j}\nonumber\\&\leq\frac{\sqrt{\lambda_j}(1-\delta\lambda_j)}{((\g -c\delta)\lambda_j)^2}\cdot(1-c+c\delta\lambda_j)\sqrt{\frac{(1+c-c\delta\lambda_j)(1-c)^2(\g -c\delta)+2c\lambda_j(\g -c\delta)^2}{1-c^2+c\lambda_j(\g +c\delta)}}\nonumber\\&+\frac{c\sqrt{\lambda_j}(1-\delta\lambda_j)}{((\g -c\delta)\lambda_j)}\sqrt{\frac{(1+c-c\delta\lambda_j)(\g -c\delta)-2\delta\lambda_j(\g -c\delta)+2\delta^2\lambda_j}{1-c^2+c\lambda_j(\g +c\delta)}}
\end{align}
Finally, we need, 
\begin{align*}
\adanorm{\begin{bmatrix}\H^{1/2}\\0\end{bmatrix}}_{\phivi}\leq\sqrt{5\sigma^2}\adanorm{\begin{bmatrix}\H^{1/2}\\0\end{bmatrix}}_{\U}
\end{align*}
Again, this can be analyzed in each of the eigen directions $(\lambda_j,\u_j)$ of $\H$ to yield:
\begin{align}
\label{eq:lotp4}
\adanorm{\begin{bmatrix}\lambda_j^{1/2}\\0\end{bmatrix}}_{\U_j}&=\sqrt{\lambda_ju_{11}}\nonumber\\
&=\sqrt{\lambda_j\cdot\frac{(1+c-c\delta\lambda_j)(\g -c\delta)-2\delta\lambda_j(\g -c\delta)+2\delta^2\lambda_j}{1-c^2+c\lambda_j(\g +c\delta)}}
\end{align}
Now, we require to bound the product of equation~\ref{eq:lotp3} and~\ref{eq:lotp4}:
\begin{align}
\label{eq:lotp5}
\adanorm{(\eye-\A_j\T)^{-2}\A_j\T\begin{bmatrix}\lambda_j^{1/2}\\0\end{bmatrix}}_{\U_j}\cdot\adanorm{\begin{bmatrix}\lambda_j^{1/2}\\0\end{bmatrix}}_{\U_j}=T_1+T_2
\end{align}
Where, 
\begin{align*}
T_1&=\frac{\lambda_j(1-\delta\lambda_j)}{((\g -c\delta)\lambda_j)^2}\cdot(1-c+c\delta\lambda_j)\bigg(\sqrt{\frac{(1+c-c\delta\lambda_j)(1-c)^2(\g -c\delta)+2c\lambda_j(\g -c\delta)^2}{1-c^2+c\lambda_j(\g +c\delta)}}\bigg)\nonumber\\&\cdot\bigg(\sqrt{\frac{(1+c-c\delta\lambda_j)(\g -c\delta)-2\delta\lambda_j(\g -c\delta)+2\delta^2\lambda_j}{1-c^2+c\lambda_j(\g +c\delta)}}\bigg)\nonumber
\end{align*}
And,
\begin{align*}
T_2&=\frac{c(1-\delta\lambda_j)}{\g -c\delta}\cdot\bigg(\frac{(1+c-c\delta\lambda_j)(\g -c\delta)-2\delta\lambda_j(\g -c\delta)+2\delta^2\lambda_j}{1-c^2+c\lambda_j(\g +c\delta)}\bigg)
\end{align*}
We begin by considering $T_1$:
\begin{align}
\label{eq:lotp51}
T_1&=\frac{\lambda_j(1-\delta\lambda_j)}{((\g -c\delta)\lambda_j)^2}\cdot(1-c+c\delta\lambda_j)\bigg(\sqrt{\frac{(1+c-c\delta\lambda_j)(1-c)^2(\g -c\delta)+2c\lambda_j(\g -c\delta)^2}{1-c^2+c\lambda_j(\g +c\delta)}}\bigg)\nonumber\\&\cdot\bigg(\sqrt{\frac{(1+c-c\delta\lambda_j)(\g -c\delta)-2\delta\lambda_j(\g -c\delta)+2\delta^2\lambda_j}{1-c^2+c\lambda_j(\g +c\delta)}}\bigg)\nonumber\\
&=\bigg(\frac{\lambda_j(1-\delta\lambda_j)}{((\g -c\delta)\lambda_j)^2}\bigg)\cdot\bigg(\frac{1-c+c\delta\lambda_j}{1-c^2+c\lambda_j(\g +c\delta)}\bigg)\cdot\nonumber\\&\bigg(\sqrt{(1+c-c\delta\lambda_j)(\g -c\delta)-2\delta\lambda_j(\g -c\delta)+2\delta^2\lambda_j}\cdot\sqrt{(1+c-c\delta\lambda_j)(1-c)^2(\g -c\delta)+2c\lambda_j(\g -c\delta)^2}\bigg)\nonumber\\
&\leq\bigg(\frac{\lambda_j}{((\g -c\delta)\lambda_j)^2}\bigg)\cdot\bigg(\frac{1-c+c\delta\lambda_j}{1-c^2+c\lambda_j(\g +c\delta)}\bigg)\cdot\nonumber\\&\bigg(\sqrt{(1+c-c\delta\lambda_j)(\g -c\delta)+2\delta^2\lambda_j}\cdot\sqrt{(1+c-c\delta\lambda_j)(1-c)^2(\g -c\delta)+2c\lambda_j(\g -c\delta)^2}\bigg)
\end{align}
We will consider the four terms within the square root and bound them separately:
\begin{align*}
T_{1}^{11}&=\frac{(1+c-c\delta\lambda_j)(1-c)}{(\g -c\delta)\lambda_j}\\
&\leq\frac{2(1-c)}{\lambda_j\cdot(\g -c\delta)}\leq\frac{2(1+\cthree)}{\lambda_j\gamma}\\
&\leq\frac{2(1+\cthree)}{\ctwo\sqrt{2c_1-c_1^2}}\sqrt{\cnH\cnS}
\end{align*}
Next,
\begin{align*}
T_{1}^{21}&=\frac{\sqrt{2\delta^2\lambda_j}\sqrt{(1+c-c\delta\lambda_j)(1-c)^2(\g -c\delta)}}{(\g -c\delta)^2\lambda_j}\\
&\leq\frac{2\delta(1-c)}{\sqrt{(\g -c\delta)^3\lambda_j}}=\frac{2\delta}{\sqrt{(\g -c\delta)\lambda_j}}\frac{1-c}{\g -c\delta}\\
&=\frac{2(1+\cthree)\delta}{\gamma}\cdot\frac{1}{\sqrt{(\g -c\delta)\lambda_j}}\\
&\leq\frac{2(1+\cthree)\delta}{\gamma}\cdot\frac{1}{\sqrt{\gamma(1-\alpha)\mu}}\\
&\leq\frac{2\sqrt{2}(1+\cthree)}{\ctwo^2(2-c_1)}\cdot\cnS
\end{align*}
Next, 
\begin{align*}
T_1^{12}&=\frac{\sqrt{(1+c-c\delta\lambda_j)(\g -c\delta)^3\cdot 2c\lambda_j}}{(\g -c\delta)^2\lambda_j}\\
&\leq\frac{2\sqrt{2}}{\ctwo\sqrt{2c_1-c_1^2}}\cdot\sqrt{\cnH\cnS}
\end{align*}
Finally, 
\begin{align*}
T_1^{22}&=\frac{\sqrt{2\delta^2\lambda_j\cdot2\ c\lambda_j(\g -c\delta)^2}}{(\g -c\delta)^2\lambda_j}\\
&\leq\frac{2\delta}{\g -c\delta}\leq\frac{4}{\ctwo^2(2-c_1)}\cdot\cnS
\end{align*}
Implying,
\begin{align*}
T_1&\leq\bigg(\frac{1-c+c\delta\lambda_j}{1-c^2+c\lambda_j(\g +c\delta)}\bigg)\cdot(T_1^{11}+T_1^{12}+T_1^{21}+T_1^{22})\\
&\leq\bigg(\frac{1-c+c\delta\lambda_j}{1-c^2+c\lambda_j(\g +c\delta)}\bigg)\cdot2\cdot(1+\sqrt{2}+\cthree)\bigg(\frac{\sqrt{\cnH\cnS}}{\ctwo\sqrt{2\cone-\cone^2}}+\sqrt{2}\frac{\cnS}{\ctwo^2(2-\cone)}\bigg)\\
&\leq\bigg(\frac{1}{1+c}+\frac{1}{2c}\bigg)\cdot2\cdot(1+\sqrt{2}+\cthree)\bigg(\frac{\sqrt{\cnH\cnS}}{\ctwo\sqrt{2\cone-\cone^2}}+\sqrt{2}\frac{\cnS}{\ctwo^2(2-\cone)}\bigg)\\
&=\bigg(\frac{1}{1+c}+\frac{1}{2c}\bigg)\cdot2\cdot(1+\sqrt{2}+\cthree)\bigg(\frac{\sqrt{\cnH\cnS}}{\sqrt{\cone\cfour}}+\frac{\sqrt{2}\cnS}{\cfour}\bigg)\\
&\leq\frac{3}{c}\cdot(1+\sqrt{2}+\cthree)\bigg(\frac{\sqrt{\cnH\cnS}}{\sqrt{\cone\cfour}}+\frac{\sqrt{2}\cnS}{\cfour}\bigg)
\end{align*}
Recall the bound on 1/c from equation~\ref{eq:oneOverc}:
\begin{align*}
\frac{1}{c}\leq\frac{1+\sqrt{\cone\cfour}}{1-\cfour}
\end{align*}
Implying,
\begin{align}
\label{eq:t1final}
T_1&\leq\frac{3}{c}\cdot(1+\sqrt{2}+\cthree)\bigg(\frac{\sqrt{\cnH\cnS}}{\sqrt{\cone\cfour}}+\frac{\sqrt{2}\cnS}{\cfour}\bigg)\nonumber\\
&\leq\frac{3}{c}\cdot(1+\sqrt{2}+\cthree)\bigg(\frac{1}{\sqrt{\cone\cfour}}+\frac{\sqrt{2}}{\cfour}\bigg)\sqrt{\cnH\cnS}\nonumber\\
&\leq3(1+\sqrt{2}+\cthree)\bigg(\frac{1}{\sqrt{\cone\cfour}}+\frac{\sqrt{2}}{\cfour}\bigg)\cdot\frac{1+\sqrt{\cone\cfour}}{1-\cfour}\sqrt{\cnH\cnS}\nonumber\\
&\leq3(1+\sqrt{2}+\sqrt{(\cfour/\cone)})\bigg(\frac{1}{\sqrt{\cone\cfour}}+\frac{\sqrt{2}}{\cfour}\bigg)\cdot\frac{1+\sqrt{\cone\cfour}}{1-\cfour}\sqrt{\cnH\cnS}
\end{align}
Next, we consider $T_{2}$:
\begin{align*}
T_2&=\frac{c(1-\delta\lambda_j)}{\g -c\delta}\cdot\bigg(\frac{(1+c-c\delta\lambda_j)(\g -c\delta)-2\delta\lambda_j(\g -c\delta)+2\delta^2\lambda_j}{1-c^2+c\lambda_j(\g +c\delta)}\bigg)\\
&\leq\bigg(\frac{(1+c-c\delta\lambda_j)(\g -c\delta)-2\delta\lambda_j(\g -c\delta)+2\delta^2\lambda_j}{(\g -c\delta)\cdot(1-c^2+c\lambda_j(\g +c\delta))}\bigg)\\
&\leq\bigg(\frac{(1+c-c\delta\lambda_j)(\g -c\delta)+2\delta^2\lambda_j}{(\g -c\delta)\cdot(1-c^2+c\lambda_j(\g +c\delta))}\bigg)
\end{align*}
We split $T_2$ into two parts:
\begin{align*}
T_2^{1}&=\frac{(1+c-c\delta\lambda_j)}{(1-c^2+c\lambda_j(\g +c\delta))}\\
&\leq\frac{1}{1-c}=\frac{1}{1-\alpha+\alpha\beta}\\
&=\frac{1}{(1+\cthree)(1-\alpha)}\\
&\leq\frac{2\sqrt{\cnH\cnS}}{(1+\cthree)\ctwo\sqrt{2c_1-c_1^2}}\\
&\leq\frac{2\sqrt{\cnH\cnS}}{(1+\sqrt{\cfour/\cone})\sqrt{\cone\cfour}}\\
&=\frac{2\sqrt{\cnH\cnS}}{\sqrt{\cone\cfour}+\cfour}
\end{align*}
Then,
\begin{align*}
T_2^{2}&=\frac{2\delta^2\lambda_j}{(\g -c\delta)(1-c^2+c\lambda_j(\g +c\delta))}\\
&\leq\frac{\delta^2\lambda_j}{\gamma(1-\alpha)c^2\lambda_j\delta}=\frac{\delta}{c^2\gamma(1-\alpha)}\\
&=\frac{2\cnS}{\cfour}\cdot\frac{1}{c^2}
\end{align*}
Implying,
\begin{align}
\label{eq:t2final}
T_2&\leq2\cdot\bigg(\frac{\sqrt{\cnH\cnS}}{\cfour+\sqrt{\cone\cfour}}+\frac{\cnS}{c^2\cfour}\bigg)\nonumber\\
&\leq2\cdot\bigg(\frac{1}{\sqrt{\cone\cfour}+\cfour}+\big(\frac{1+\sqrt{\cone\cfour}}{1-\cfour}\big)^2\cdot\frac{1}{\cfour}\bigg)\sqrt{\cnH\cnS}\nonumber\\
&\leq\frac{2}{\cfour}\cdot\bigg(1+\big(\frac{1+\sqrt{\cone\cfour}}{1-\cfour}\big)^2\bigg)\sqrt{\cnH\cnS}
\end{align}
We add $T_1$ and $T_2$ and revisit equation~\ref{eq:lotp5}:
\begin{align}
\label{eq:perDirectionBound}
&\adanorm{(\eye-\A_j\T)^{-2}\A_j\T\begin{bmatrix}\lambda_j^{1/2}\\0\end{bmatrix}}_{\U_j}\cdot\adanorm{\begin{bmatrix}\lambda_j^{1/2}\\0\end{bmatrix}}_{\U_j}\nonumber\\
&=T_1+T_2\nonumber\\
&\leq\bigg(\ \frac{2}{\cfour}\cdot\bigg(1+\big(\frac{1+\sqrt{\cone\cfour}}{1-\cfour}\big)^2\bigg) + 3 \cdot \frac{1+\sqrt{\cone\cfour}}{1-\cfour} \cdot \frac{1+\sqrt{2}+\sqrt{\cfour/\cone}}{\cfour} \cdot (\sqrt{2}+\sqrt{\cfour/\cone}) \ \bigg)\sqrt{\cnH\cnS}
\end{align}
Then, we revisit equation~\ref{eq:lotpmain1}:
\begin{align}
\label{eq:lotpmain2}
&\bigg(\begin{bmatrix}\H^{1/2}\\0\end{bmatrix}\T\phivih\bigg)\cdot\bigg(\phivih(\eye-\A\T)^{-2}\A\T\begin{bmatrix}\H^{1/2}\\0\end{bmatrix}\bigg)+\bigg(\phivih\begin{bmatrix}\H^{1/2}\\0\end{bmatrix}\bigg)\cdot\bigg(\begin{bmatrix}\H^{1/2}\\0\end{bmatrix}\T\A(\eye-\A)^{-2}\phivih\bigg)\nonumber\\
&\leq2\sum_{j=1}^{d} \adanorm{\begin{bmatrix}\lambda_j^{1/2}\\0\end{bmatrix}}_{(\phivi)_j}\cdot\adanorm{(\eye-\A_j\T)^{-2}\A_j\T\begin{bmatrix}\lambda_j^{1/2}\\0\end{bmatrix}}_{(\phivi)_j} \nonumber\\
&\leq10\sigma^2 \sum_{j=1}^{d} \adanorm{\begin{bmatrix}\lambda_j^{1/2}\\0\end{bmatrix}}_{\U_j}\cdot\adanorm{(\eye-\A_j\T)^{-2}\A_j\T\begin{bmatrix}\lambda_j^{1/2}\\0\end{bmatrix}}_{\U_j}\quad(\text{using equation~\ref{eq:stationaryDistBound}})\nonumber\\
&\leq10\sigma^2 \cdot d \cdot\bigg(\ \frac{2}{\cfour}\cdot\bigg(1+\big(\frac{1+\sqrt{\cone\cfour}}{1-\cfour}\big)^2\bigg) + 3 \cdot \frac{1+\sqrt{\cone\cfour}}{1-\cfour} \cdot \frac{1+\sqrt{2}+\sqrt{\cfour/\cone}}{\cfour} \cdot (\sqrt{2}+\sqrt{\cfour/\cone}) \ \bigg)\sqrt{\cnH\cnS}\nonumber\\
&\leq \UC \sigma^2 d \sqrt{\cnH\cnS}
\end{align}
Where the equation in the penultimate line is obtained by summing over all eigen directions the bound implied by equation~\ref{eq:perDirectionBound}, and $\UC$ is a universal constant.
\end{proof}

\begin{lemma}\label{lem:bound-variance}
		\begin{align*}
	&\iprod{\begin{bmatrix}
		\Cov & \zero \\ \zero & \zero
		\end{bmatrix}}{\E{\thetavb^{\textrm{variance}} \otimes \thetavb^{\text{variance}}}}\leq 5\frac{\sigma^2d}{n-t} +\UC\cdot\frac{\sigma^2 d}{(n-t)^2} \cdot\sqrt{\cnH\cnS}  \\ &+  \UC\cdot\frac{\sigma^2d}{n-t}(\cnH\cnS)^{11/4}\exp\bigg(-\frac{(n-t-1)\ctwo\sqrt{2\cone-\cone^2}}{4\sqrt{\cnH\cnS}}\bigg) \\&+ \UC\cdot\frac{\sigma^2d}{(n-t)^2}\cdot\exp\bigg({-(n+1)\frac{\cone\cthree^2}{\sqrt{\cnH\cnS}}}\bigg)\cdot(\cnH\cnS)^{7/2}\cnS+\UC\cdot\sigma^2d\cdot(\cnH\cnS)^{7/4}\exp\bigg({-(n+1)\cdot\frac{\ctwo\cthree\sqrt{2\cone-\cone^2}}{\sqrt{\cnH\cnS}}}\bigg)
	\end{align*}
	where, $\UC$ is a universal constant.
\end{lemma}
\begin{proof}
We begin by recounting the expression for the covariance of the variance error of the tail-averaged iterate $\thetavb^{\textrm{variance}}$ from equation~\ref{eq:varianceTA}:
\begin{align*}
\E{\thetavb^{\text{variance}}\otimes\thetavb^{\text{variance}}}
&=\underbrace{\frac{1}{n-t}\big(\eyeT + (\eyeT-\AL)^{-1}\AL + (\eyeT-\AR\T)^{-1}\AR\T\big)(\eyeT-\BT)^{-1}\Sigh}_{\Y_1\eqdef}\nonumber\\&\underbrace{-\frac{1}{(n-t)^2}\big((\eyeT-\AL)^{-2}\AL+(\eyeT-\AR\T)^{-2}\AR\T\big)(\eyeT-\BT)^{-1}\Sigh}_{\Y_2\eqdef}\nonumber\\&\underbrace{+\frac{1}{(n-t)^2}\big((\eyeT-\AL)^{-2}\AL^{n+1-t}+(\eyeT-\AR\T)^{-2}(\AR\T)^{n+1-t}\big)(\eyeT-\BT)^{-1}\Sigh}_{\Y_3\eqdef}\nonumber\\&\underbrace{-\frac{1}{(n-t)^2}\big(\eyeT + (\eyeT-\AL)^{-1}\AL + (\eyeT-\AR\T)^{-1}\AR\T\big)(\eyeT-\BT)^{-2}(\BT^{t+1}-\BT^{n+1})\Sigh}_{\Y_4\eqdef}\nonumber\\&\underbrace{+\frac{1}{(n-t)^2}\sum_{j=t+1}^n\big((\eyeT-\AL)^{-1}\AL^{n+1-j}+(\eyeT-\AR\T)^{-1}(\AR\T)^{n+1-j}\big)(\eyeT-\BT)^{-1}\BT^j\Sigh}_{\Y_5\eqdef}
\end{align*}
The goal is to bound $\iprod{\begin{bmatrix}\H&0\\0&0\end{bmatrix}}{\Y_i}$, for $i=1,..,5$.

For the case of $\Y_1$, combining the fact that $\E{\thetav_{\infty}\otimes\thetav_{\infty}}=(\eyeT-\BT)^{-1}\Sigh$ and lemma~\ref{lem:var-main-1}, we get:
\begin{align}
\label{eq:e1}
\iprod{\begin{bmatrix}\H&0\\0&0\end{bmatrix}}{\Y_1}&=\frac{1}{n-t}\iprod{\begin{bmatrix}\H&0\\0&0\end{bmatrix}}{\big(\eyeT + (\eyeT-\AL)^{-1}\AL + (\eyeT-\AR\T)^{-1}\AR\T\big)(\eyeT-\BT)^{-1}\Sigh}\nonumber\\
&=\frac{1}{n-t}\iprod{\begin{bmatrix}\H&0\\0&0\end{bmatrix}}{\big(\eyeT + (\eyeT-\AL)^{-1}\AL + (\eyeT-\AR\T)^{-1}\AR\T\big)\E{\thetav_\infty\otimes\thetav_\infty}}\nonumber\\
&\leq 5\frac{\sigma^2d}{n-t}
\end{align}

For the case of $\Y_2$, we employ the result from lemma~\ref{lem:var1N2bound}, and this gives us:

\begin{align}
\label{eq:e2}
&\bigg\vert\iprod{\begin{bmatrix}\H&0\\0&0\end{bmatrix}}{\Y_2}\bigg\vert\leq\frac{\UC\cdot\sigma^2 d\sqrt{\cnH\cnS}}{(n-t)^2}
\end{align}

For $i=3$, we have:
\begin{align}
\label{eq:e31}
&\iprod{\begin{bmatrix}\H&0\\0&0\end{bmatrix}}{\Y_3}=\frac{1}{(n-t)^2}\iprod{\begin{bmatrix}\H&0\\0&0\end{bmatrix}}{\big((\eyeT-\AL)^{-2}\AL^{n+1-t}+(\eyeT-\AR\T)^{-2}(\AR\T)^{n+1-t}\big)(\eyeT-\BT)^{-1}\Sigh}\nonumber\\
&=\frac{1}{(n-t)^2}\bigg(\iprod{(\eye-\A\T)^{-2}\A\T\begin{bmatrix}\H&0\\0&0\end{bmatrix}}{\A^{n-t}(\eyeT-\BT)^{-1}\Sigh}+\iprod{\begin{bmatrix}\H&0\\0&0\end{bmatrix}\A(\eye-\A)^{-2}}{(\eyeT-\BT)^{-1}\Sigh\ (\A\T)^{n-t}}\bigg)\nonumber\\
&=\frac{4d}{(n-t)^2}\cdot\|(\eye-\A\T)^{-2}\A\T\begin{bmatrix}\H&0\\0&0\end{bmatrix}\|\cdot\|\A^{n-t}(\eyeT-\BT)^{-1}\Sigh\|
\end{align}
We will consider bounding $\|\A^{n-t}(\eyeT-\BT)^{-1}\Sigh\|$:
\begin{align}
\label{eq:e311}
\|\A^{n-t}(\eyeT-\BT)^{-1}\Sigh\|&\leq\sum_{i=0}^{\infty}\|\A^{n-t}\BT^i\Sigh\|\nonumber\\
&\leq\frac{12\sqrt{2}}{\sqrt{1-\alpha^2}}\cnH(n-t)\alpha^{(n-t-1)/2}\bigg(\sum_{i}\big(1-\frac{\ctwo\cthree\sqrt{2\cone-\cone^2}}{\sqrt{\cnH\cnS}}\big)^i\bigg)\|\Sigh\|\nonumber\\&\qquad\qquad\qquad\qquad\qquad\qquad\qquad\qquad\qquad\qquad\quad(\text{using corollary}~\ref{cor:bias-tail1})\nonumber\\
&=\frac{12\sqrt{2}}{\sqrt{1-\alpha^2}}\cnH(n-t)\alpha^{(n-t-1)/2}\cdot\frac{\sqrt{\cnH\cnS}}{\ctwo\cthree\sqrt{2\cone-\cone^2}}\cdot\|\Sigh\|\nonumber\\
&=\frac{12\sqrt{2}\sigma^2}{\sqrt{1-\alpha^2}}\cnH(n-t)\alpha^{(n-t-1)/2}\cdot\frac{\sqrt{\cnH\cnS}}{\ctwo\cthree\sqrt{2\cone-\cone^2}}\cdot(\g+c\delta)^2\|\H\|\nonumber\\
&\leq\frac{108\sqrt{2}\sigma^2}{\sqrt{1-\alpha^2}}\cnH(n-t)\alpha^{(n-t-1)/2}\cdot\frac{\sqrt{\cnH\cnS}}{\ctwo\cthree\sqrt{2\cone-\cone^2}}\cdot\delta^2\|\H\|
\end{align}
We also upper bound $\alpha$ as:
\begin{align}
\label{eq:alpBound}
\alpha&=1-\frac{\ctwo\sqrt{2\cone-\cone^2}}{\sqrt{\cnH\cnS}+\ctwo\sqrt{2\cone-\cone^2}}\nonumber\\
&\leq1-\frac{\ctwo\sqrt{2\cone-\cone^2}}{2\sqrt{\cnH\cnS}}\nonumber\\
&=e^{-\frac{\ctwo\sqrt{2\cone-\cone^2}}{2\sqrt{\cnH\cnS}}}
\end{align}
Furthermore, for $\|(\eye-\A\T)^{-2}\A\T\begin{bmatrix}\H&0\\0&0\end{bmatrix}\|$, we consider a bound in each eigendirection $j$ and accumulate the results subsequently:
\begin{align*}
&\|(\eye-\A_j\T)^{-2}\A_j\T\begin{bmatrix}\lambda_j&0\\0&0\end{bmatrix}\|\\&\leq\frac{1}{(\g-c\delta)^2}\cdot\frac{1-\delta\lambda_j}{\lambda_j}\cdot\sqrt{(1+c^2)(1-c)^2+c^2\lambda_j^2(\g^2+\delta^2)}\nonumber\\&\qquad\qquad\qquad\qquad\qquad\qquad\qquad\qquad\qquad\qquad\quad(\text{using lemma}~\ref{lem:com2})\\
&\leq\frac{\sqrt{7}}{(\g-c\delta)^2}\cdot\frac{1}{\lambda_j}\\
&\leq\frac{\sqrt{7}}{(\gamma(1-\alpha))^2}\cdot\frac{1}{\lambda_j}\\
&\leq\frac{48(\cnH\cnS)^2}{(\cone\cfour)^2}\frac{\mu^2}{\lambda_j}=\frac{48\cnS^2}{(\delta\cfour)^2}\frac{1}{\lambda_j}\\
\implies\|(\eye-\A\T)^{-2}\A\T\begin{bmatrix}\H&0\\0&0\end{bmatrix}\|&\leq\frac{48\cnS^2}{(\delta\cfour)^2}\cdot\frac{1}{\mu}
\end{align*}
Plugging this into equation~\ref{eq:e31}, we obtain:
\begin{align}
\label{eq:e3}
\iprod{\begin{bmatrix}\H&0\\0&0\end{bmatrix}}{\Y_3}&\leq41472\frac{\sigma^2d}{n-t}(\cnH\cnS)^{11/4}\alpha^{(n-t-1)/2}\frac{1}{\cthree\cfour^2(\cone\cthree)^{3/2}}\nonumber\\
&\leq\UC\frac{\sigma^2d}{n-t}(\cnH\cnS)^{11/4}\alpha^{(n-t-1)/2}\nonumber\\
&\leq\UC\frac{\sigma^2d}{n-t}(\cnH\cnS)^{11/4}\exp^{-\frac{(n-t-1)\ctwo\sqrt{2\cone-\cone^2}}{4\sqrt{\cnH\cnS}}}
\end{align}
Next, let us consider $\Y_4$:
\begin{align}
\label{eq:e41}
&\iprod{\begin{bmatrix}\H&0\\0&0\end{bmatrix}}{\Y_4}\nonumber\\&=-\frac{1}{(n-t)^2}\iprod{\begin{bmatrix}\H&0\\0&0\end{bmatrix}}{\big(\eyeT + (\eyeT-\AL)^{-1}\AL + (\eyeT-\AR\T)^{-1}\AR\T\big)(\eyeT-\BT)^{-2}(\BT^{t+1}-\BT^{n+1})\Sigh}\nonumber\\
&=-\frac{1}{(n-t)^2}\iprod{(\eye-\A\T)^{-1}\begin{bmatrix}\H&0\\0&0\end{bmatrix}(\eye-\A)^{-1}}{(\eyeT-\DT)(\eyeT-\BT)^{-2}(\BT^{t+1}-\BT^{n+1})\Sigh}\nonumber\\
&=-\frac{1}{(\g-c\delta)^2(n-t)^2}\iprod{\bigg(\otimes_{2}\begin{bmatrix}-(c\eye-\g\Cov)\Cov^{-1/2}\\(\eye-\delta\Cov)\Cov^{-1/2}\end{bmatrix}\bigg)}{(\eyeT-\BT+\RT)(\eyeT-\BT)^{-2}(\BT^{t+1}-\BT^{n+1})\Sigh}\nonumber\\&\qquad\qquad\qquad\qquad\qquad\qquad\qquad\qquad\qquad\qquad\qquad\qquad\qquad\qquad\qquad(\text{using lemma}~\ref{lem:com1})\nonumber\\
&\leq\frac{1}{(\g-c\delta)^2(n-t)^2}\iprod{\bigg(\otimes_{2}\begin{bmatrix}-(c\eye-\g\Cov)\Cov^{-1/2}\\(\eye-\delta\Cov)\Cov^{-1/2}\end{bmatrix}\bigg)}{(\eyeT-\BT+\RT)(\eyeT-\BT)^{-2}\BT^{n+1}\Sigh}\nonumber\\
&\leq\frac{1}{(\g-c\delta)^2(n-t)^2}\cdot\bigg(\iprod{\bigg(\otimes_{2}\begin{bmatrix}-(c\eye-\g\Cov)\Cov^{-1/2}\\(\eye-\delta\Cov)\Cov^{-1/2}\end{bmatrix}\bigg)}{(\eyeT-\BT)^{-1}\BT^{n+1}\Sigh}\nonumber\\&\qquad \qquad\qquad \qquad \qquad  +\iprod{\RT\T\bigg(\otimes_{2}\begin{bmatrix}-(c\eye-\g\Cov)\Cov^{-1/2}\\(\eye-\delta\Cov)\Cov^{-1/2}\end{bmatrix}\bigg)}{(\eyeT-\BT)^{-2}\BT^{n+1}\Sigh}\bigg)\nonumber\\
&=\frac{1}{(\g-c\delta)^2(n-t)^2}\cdot\bigg(\iprod{\bigg(\otimes_{2}\begin{bmatrix}-(c\eye-\g\Cov)\Cov^{-1/2}\\(\eye-\delta\Cov)\Cov^{-1/2}\end{bmatrix}\bigg)}{(\eyeT-\BT)^{-1}\BT^{n+1}\Sigh}\nonumber\\&\qquad \qquad\qquad \qquad \qquad  +\iprod{\otimes_2\begin{bmatrix}\delta\\\g\end{bmatrix}\otimes\bigg(\M-\HL\HR\bigg)(\eye-\delta\H)\H^{-1}(\eye-\delta\H)}{(\eyeT-\BT)^{-2}\BT^{n+1}\Sigh}\bigg)\nonumber\\
&\leq\frac{1}{(\g-c\delta)^2(n-t)^2}\cdot\bigg(\iprod{\bigg(\otimes_{2}\begin{bmatrix}-(c\eye-\g\Cov)\Cov^{-1/2}\\(\eye-\delta\Cov)\Cov^{-1/2}\end{bmatrix}\bigg)}{(\eyeT-\BT)^{-1}\BT^{n+1}\Sigh}\nonumber \nonumber\\&\qquad\qquad\qquad\qquad\qquad\qquad\qquad\qquad\qquad\qquad+\cnS\cdot\iprod{(\Sigh/\sigma^2)}{(\eyeT-\BT)^{-2}\BT^{n+1}\Sigh}\bigg)
\end{align}

To bound $\|\otimes_{2}\begin{bmatrix}-(c\eye-\g\Cov)\Cov^{-1/2}\\(\eye-\delta\Cov)\Cov^{-1/2}\end{bmatrix}
\|$, we will consider a bound along each eigendirection and accumulate the results:
\begin{align*}
\|\otimes_{2}\begin{bmatrix}-(c-\g\lambda_j)\lambda_j^{-1/2}\\(1-\delta\lambda_j)\lambda_j^{-1/2}\end{bmatrix}\|&\leq\frac{(c-\g\lambda_j)^2+(1-\delta\lambda_j)^2}{\lambda_j}\\
&\leq2\cdot\frac{(1+c^2)+(\g^2+\delta^2)\lambda_j^2}{\lambda_j}\\
&\leq2\cdot\frac{2+5\delta^2\lambda_j^2}{\lambda_j}\leq\frac{14}{\lambda_j}\\
\implies\|\otimes_{2}\begin{bmatrix}-(c\eye-\g\Cov)\Cov^{-1/2}\\(\eye-\delta\Cov)\Cov^{-1/2}\end{bmatrix}
\|&\leq\frac{14}{\mu}
\end{align*}

Next, we bound $\|\BT^k(\eyeT-\BT)^{-1}\Sigh\|$ (as a consequence of lemma~\ref{lem:B-contraction} with $\Q=\Sigh$):
\begin{align*}
\|\BT^k(\eyeT-\BT)^{-1}\Sigh\|&\leq\frac{1}{\lambda_{\min}(\G)}\|\G\T\BT^k(\eyeT-\BT)^{-1}\Sigh\|\\
&\leq\frac{1}{\lambda_{\min}(\G)}\sum_{l=k}^{\infty}\|\G\T\B^k\Sigh\|\\
&\leq\frac{\sqrt{\cnH\cnS}}{\ctwo\cthree\sqrt{2\cone-\cone^2}}\kappa(\G)\exp({-k\frac{\ctwo\cthree\sqrt{2\cone-\cone^2}}{\sqrt{\cnH\cnS}}})\|\Sigh\|\\
&\leq\frac{4\sigma^2\cnH}{\sqrt{1-\alpha^2}}\cdot\frac{\sqrt{\cnH\cnS}}{\cone\cthree^2}\exp({-k\frac{\cone\cthree^2}{\sqrt{\cnH\cnS}}})\cdot 9\delta^2\|\H\|_2\\
&\leq\frac{36\sigma^2\cnH}{\sqrt{1-\alpha^2}}\cdot\frac{\sqrt{\cnH\cnS}}{\cone\cthree^2}\exp({-k\frac{\cone\cthree^2}{\sqrt{\cnH\cnS}}})\cdot \delta
\end{align*}

This implies, 
\begin{align*}
\iprod{\bigg(\otimes_{2}\begin{bmatrix}-(c\eye-\g\Cov)\Cov^{-1/2}\\(\eye-\delta\Cov)\Cov^{-1/2}\end{bmatrix}\bigg)}{(\eyeT-\BT)^{-1}\BT^{n+1}\Sigh}\leq\nonumber\\504\cdot\frac{\cnH}{\sqrt{1-\alpha^2}}\cdot\frac{\sqrt{\cnH\cnS}}{\cone\cthree^2}\exp\bigg({-(n+1)\frac{\cone\cthree^2}{\sqrt{\cnH\cnS}}}\bigg)\cdot \frac{\delta}{\mu}\cdot \sigma^2d
\end{align*}

Furthermore, 
\begin{align*}
\frac{\cnS}{\sigma^2}\iprod{\Sigh}{(\eyeT-\BT)^{-2}\BT^{n+1}\Sigh}&=\frac{\cnS}{\sigma^2}\iprod{(\eyeT-\BT)^{-1}\BT^{(n+1)/2}\Sigh}{(\eyeT-\BT)^{-1}\BT^{(n+1)/2}\Sigh}\\
&\leq\frac{\cnS}{\sigma^2}\|(\eyeT-\BT)^{-1}\BT^{(n+1)/2}\Sigh\|^2\cdot d\\
&\leq1296\frac{\sigma^2d}{1-\alpha^2}\bigg(\cnH\frac{\sqrt{\cnH\cnS}}{\cone\cthree^2}\bigg)^2\delta^2\cnS\exp({-(n+1)\frac{\cone\cthree^2}{\sqrt{\cnH\cnS}}})
\end{align*}

This implies that,
\begin{align}
\label{eq:e42}
&\iprod{\bigg(\otimes_{2}\begin{bmatrix}-(c\eye-\g\Cov)\Cov^{-1/2}\\(\eye-\delta\Cov)\Cov^{-1/2}\end{bmatrix}\bigg)}{(\eyeT-\BT)^{-1}\BT^{n+1}\Sigh}+\frac{\cnS}{\sigma^2}\iprod{\Sigh}{(\eyeT-\BT)^{-2}\BT^{n+1}\Sigh}\nonumber\\&\qquad\qquad\leq2592\frac{\sigma^2d}{1-\alpha^2}\bigg(\cnH\frac{\sqrt{\cnH\cnS}}{\cone\cthree^2}\bigg)^2\delta^2\cnS\exp({-(n+1)\frac{\cone\cthree^2}{\sqrt{\cnH\cnS}}})\nonumber\\
&\qquad\qquad\leq2592\cdot\sigma^2d\cdot\bigg(\frac{\sqrt{\cnH\cnS}}{\cone\cthree^2}\bigg)^3\exp({-(n+1)\frac{\cone\cthree^2}{\sqrt{\cnH\cnS}}})\cdot\delta^2\cnH^2\cnS
\end{align}

Finally, we also note the following:
\begin{align*}
\frac{1}{(\g-c\delta)}\leq\frac{1}{(\gamma(1-\alpha))}\leq\frac{\mu}{(1-\alpha)^2}\leq\frac{4\cnS}{\delta\cfour}
\end{align*}

Plugging equation~\ref{eq:e42} into equation~\ref{eq:e41}, we get:
\begin{align}
\label{eq:e4}
\iprod{\begin{bmatrix}\H&0\\0&0\end{bmatrix}}{\Y_4}&=2592\cdot\frac{\sigma^2d}{(n-t)^2(\g-c\delta)^2}\cdot\bigg(\frac{\sqrt{\cnH\cnS}}{\cone\cthree^2}\bigg)^3\exp({-(n+1)\frac{\cone\cthree^2}{\sqrt{\cnH\cnS}}})\cdot\delta^2\cnH^2\cnS\nonumber\\
&\leq41472\cdot\frac{\sigma^2d}{(n-t)^2}\cdot\frac{1}{\cfour^2}\cdot\bigg(\frac{\sqrt{\cnH\cnS}}{\cone\cthree^2}\bigg)^3\exp({-(n+1)\frac{\cone\cthree^2}{\sqrt{\cnH\cnS}}})\cdot\cnH^2\cnS^3\nonumber\nonumber\\
&=41472\cdot\frac{\sigma^2d}{(n-t)^2}\cdot\frac{1}{\cfour^2(\cone\cthree^2)^3}\cdot\exp\bigg({-(n+1)\frac{\cone\cthree^2}{\sqrt{\cnH\cnS}}}\bigg)\cdot(\cnH\cnS)^{7/2}\cnS\nonumber\\
&\leq\UC\cdot\frac{\sigma^2d}{(n-t)^2}\cdot\exp\bigg({-(n+1)\frac{\cone\cthree^2}{\sqrt{\cnH\cnS}}}\bigg)\cdot(\cnH\cnS)^{7/2}\cnS
\end{align}

Next, we consider $\Y_5$:
\begin{align}
\label{eq:e51}
&\iprod{\begin{bmatrix}\H&0\\0&0\end{bmatrix}}{\Y_5}\nonumber\\&=\frac{1}{(n-t)^2}\sum_{j=t+1}^n\iprod{\begin{bmatrix}\H&0\\0&0\end{bmatrix}}{\big((\eyeT-\AL)^{-1}\AL^{n+1-j}+(\eyeT-\AR\T)^{-1}(\AR\T)^{n+1-j}\big)(\eyeT-\BT)^{-1}\BT^j\Sigh}\nonumber\\
&=\frac{1}{(n-t)^2}\sum_{j=t+1}^n\bigg(\iprod{(\eyeT-\A\T)^{-1}\A\T\begin{bmatrix}\H&0\\0&0\end{bmatrix}}{\A^{n-j}(\eyeT-\BT)^{-1}\BT^j\Sigh}\nonumber\\&\qquad\qquad\qquad\qquad\qquad+\iprod{\begin{bmatrix}\H&0\\0&0\end{bmatrix}\A(\eyeT-\A)^{-1}}{(\eyeT-\BT)^{-1}\BT^j\Sigh(\A\T)^{n-j}}\bigg)\nonumber\\
&\leq\frac{4d}{(n-t)^2}\sum_{j=t+1}^n\|(\eye-\A\T)^{-1}\A\T\begin{bmatrix}\H&0\\0&0\end{bmatrix}\|\cdot\|\A^{n-j}(\eyeT-\BT)^{-1}\BT^j\Sigh\|
\end{align}
In a manner similar to bounding $\|\A^{n-t}(\eyeT-\BT)^{-1}\Sigh\|$ as in equation~\ref{eq:e311}, we can bound $\|\A^{n-j}(\eyeT-\BT)^{-1}\BT^j\Sigh\|$ as:
\begin{align*}
\|\A^{n-j}(\eyeT-\BT)^{-1}\BT^j\Sigh\|&\leq\frac{108\sqrt{2}\sigma^2}{\sqrt{1-\alpha^2}}\cnH(n-j)\alpha^{(n-j-1)/2}\cdot\frac{\sqrt{\cnH\cnS}}{\ctwo\cthree\sqrt{2\cone-\cone^2}}\cdot\exp^{-(\frac{j\ctwo\cthree\sqrt{2\cone-\cone^2}}{\sqrt{\cnH\cnS}})}\cdot\delta^2\|\H\|
\end{align*}
Furthermore, we will consider the bound $\|(\eye-\A\T)^{-1}\A\T\begin{bmatrix}\H&0\\0&0\end{bmatrix}\|$ along one eigen direction (by employing equation~\ref{eq:intermediateEqn}) and collect the results:
\begin{align*}
\|(\eye-\A_j\T)^{-1}\A_j\T\begin{bmatrix}\lambda_j&0\\0&0\end{bmatrix}\|&\leq\frac{1+c^2}{\g-c\delta}\leq\frac{2}{\g-c\delta}\\
&\leq\frac{2}{\gamma(1-\alpha)}\leq\frac{4\cnS}{\delta\cfour}\\
\implies\|(\eye-\A\T)^{-1}\A\T\begin{bmatrix}\Cov&0\\0&0\end{bmatrix}\|&\leq\frac{4\cnS}{\delta\cfour}
\end{align*}
Plugging this into equation~\ref{eq:e51}, and upper bounding the sum by $(n-t)$ times the largest term of the series:
\begin{align}
\label{eq:e5}
\iprod{\begin{bmatrix}\H&0\\0&0\end{bmatrix}}{\Y_5}&\leq6912\cdot\sigma^2d\cdot\frac{(\cnH\cnS)^{7/4}}{\cthree\cfour(\cone\cthree)^{3/2}}\exp^{-(n+1)\cdot\frac{\ctwo\cthree\sqrt{2\cone-\cone^2}}{\sqrt{\cnH\cnS}}}\nonumber\\
&\leq\UC\cdot\sigma^2d\cdot(\cnH\cnS)^{7/4}\cdot\exp^{-(n+1)\cdot\frac{\ctwo\cthree\sqrt{2\cone-\cone^2}}{\sqrt{\cnH\cnS}}}
\end{align}
Summing up equations~\ref{eq:e1},~\ref{eq:e2},~\ref{eq:e3},~\ref{eq:e4},~\ref{eq:e5}, the statement of the lemma follows.
\end{proof}

\section{Proof of Theorem~\ref{thm:main}}\label{sec:proofMainTheorem}
\begin{proof}[Proof of Theorem~\ref{thm:main}]
	The proof of the theorem follows through various lemmas that have been proven in the appendix:
	\begin{itemize}
	\item Section~\ref{sec:tailAverageIterateCovariance} provides the bias-variance decomposition and provides an exact tensor expression governing the covariance of the bias error (through lemma~\ref{lem:average-covar-bias})and the variance error (lemma~\ref{lem:average-covar-var}).
	\item Section~\ref{sec:biasContraction} provides a scalar bound of the bias error through lemma~\ref{lem:bound-bias}. The technical contribution of this section (which introduces a new potential function) is in lemma~\ref{lem:main-bias}. 
	\item Section~\ref{sec:varianceContraction} provides a scalar bound of the variance error through lemma~\ref{lem:bound-variance}. The key technical contribution of this section is in the introduction of a stochastic process viewpoint of the proposed accelerated stochastic gradient method through lemmas~\ref{lem:main-variance},~\ref{lem:var-main-1}. These lemmas provide a tight characterization of the stationary distribution of the covariance of the iterates of the accelerated method. Lemma~\ref{lem:var1N2bound} is necessary to show the sharp burn-in (up to log factors), beyond which the leading order term of the error is up to constants the statistically optimal error rate $\mathcal{O}(\sigma^2 d/n)$.
	\end{itemize}
	Combining the results of these lemmas, we obtain the following guarantee of algorithm~\ref{algo:TAASGD}:

	\begin{align*}
		\E{P(\bar{\x}_{t,n})}-P(\xs) &\leq \UC\cdot\frac{(\cnH\cnS)^{9/4}d\cnH}{(n-t)^2}\cdot\exp\bigg(-\frac{t+1}{9\sqrt{\cnH\cnS}}\bigg)\cdot\big(P(\x_0)-P(\xs)\big) \\&+\UC\cdot(\cnH\cnS)^{5/4}d\cnH\cdot\exp\left(\frac{-n }{9\sqrt{\cnH\cnS}}\right) \cdot \big(P(\x_0)-P(\xs)\big) + 5\frac{\sigma^2d}{n-t}\\&+ \UC\cdot\frac{\sigma^2 d}{(n-t)^2} \sqrt{\cnH\cnS} + \UC\cdot\sigma^2d\cdot(\cnH\cnS)^{7/4}\cdot\exp\bigg(\frac{-(n+1)}{9\sqrt{\cnH\cnS}}\bigg) \\ &+ \UC\cdot\frac{\sigma^2d}{n-t}(\cnH\cnS)^{11/4}\exp\bigg(-\frac{(n-t-1)}{30\sqrt{\cnH\cnS}}\bigg) \\&+\UC\cdot\frac{\sigma^2d}{(n-t)^2}\cdot\exp\bigg({-\frac{(n+1)}{9\sqrt{\cnH\cnS}}}\bigg)\cdot(\cnH\cnS)^{7/2}\cnS
	\end{align*}	
	Where, $\UC$ is a universal constant.

\end{proof}


\begin{thebibliography}{51}
	\providecommand{\natexlab}[1]{#1}
	\providecommand{\url}[1]{\texttt{#1}}
	\expandafter\ifx\csname urlstyle\endcsname\relax
	\providecommand{\doi}[1]{doi: #1}\else
	\providecommand{\doi}{doi: \begingroup \urlstyle{rm}\Url}\fi
	
	\bibitem[Agarwal et~al.(2012)Agarwal, Bartlett, Ravikumar, and
	Wainwright]{AgarwalBRW12}
	A.~Agarwal, P.~L. Bartlett, P.~Ravikumar, and M.~J. Wainwright.
	\newblock Information-theoretic lower bounds on the oracle complexity of
	stochastic convex optimization.
	\newblock \emph{IEEE Transactions on Information Theory}, 2012.
	
	\bibitem[Allen-Zhu(2016)]{Zhu16}
	Z.~Allen-Zhu.
	\newblock Katyusha: The first direct acceleration of stochastic gradient
	methods.
	\newblock \emph{CoRR}, abs/1603.05953, 2016.
	
	\bibitem[Anbar(1971)]{anbar1971optimal}
	D.~Anbar.
	\newblock \emph{On Optimal Estimation Methods Using Stochastic Approximation
		Procedures}.
	\newblock University of California, 1971.
	\newblock URL \url{http://books.google.com/books?id=MmpHJwAACAAJ}.
	
	\bibitem[Bach(2014)]{Bach14}
	F.~R. Bach.
	\newblock Adaptivity of averaged stochastic gradient descent to local strong
	convexity for logistic regression.
	\newblock \emph{Journal of Machine Learning Research (JMLR)}, volume 15, 2014.
	
	\bibitem[Bach and Moulines(2011)]{BachM11}
	F.~R. Bach and E.~Moulines.
	\newblock Non-asymptotic analysis of stochastic approximation algorithms for
	machine learning.
	\newblock In \emph{NIPS 24}, 2011.
	
	\bibitem[Bach and Moulines(2013)]{BachM13}
	F.~R. Bach and E.~Moulines.
	\newblock Non-strongly-convex smooth stochastic approximation with convergence
	rate {O}(1/n).
	\newblock In \emph{NIPS 26}, 2013.
	
	\bibitem[Bottou and Bousquet(2007)]{BottouB07}
	L.~Bottou and O.~Bousquet.
	\newblock The tradeoffs of large scale learning.
	\newblock In \emph{NIPS 20}, 2007.
	
	\bibitem[Cauchy(1847)]{cauchy1847}
	L.~A. Cauchy.
	\newblock M\'ethode g\'en\'erale pour la r\'esolution des syst\'emes
	d'\'equations simultanees.
	\newblock \emph{C. R. Acad. Sci. Paris}, 1847.
	
	\bibitem[d'Aspremont(2008)]{dAspremont08}
	A.~d'Aspremont.
	\newblock Smooth optimization with approximate gradient.
	\newblock \emph{SIAM Journal on Optimization}, 19\penalty0 (3):\penalty0
	1171--1183, 2008.
	
	\bibitem[D{\'e}fossez and Bach(2015)]{DefossezB15}
	A.~D{\'e}fossez and F.~R. Bach.
	\newblock Averaged least-mean-squares: Bias-variance trade-offs and optimal
	sampling distributions.
	\newblock In \emph{AISTATS}, volume~38, 2015.
	
	\bibitem[Devolder et~al.(2013)Devolder, Glineur, and Nesterov]{DevolderGN13}
	O.~Devolder, F.~Glineur, and Y.~E. Nesterov.
	\newblock First-order methods with inexact oracle: the strongly convex case.
	\newblock \emph{CORE Discussion Papers 2013016}, 2013.
	
	\bibitem[Devolder et~al.(2014)Devolder, Glineur, and Nesterov]{DevolderGN14}
	O.~Devolder, F.~Glineur, and Y.~E. Nesterov.
	\newblock First-order methods of smooth convex optimization with inexact
	oracle.
	\newblock \emph{Mathematical Programming}, 146:\penalty0 37--75, 2014.
	
	\bibitem[Dieuleveut and Bach(2015)]{DieuleveutB15}
	A.~Dieuleveut and F.~R. Bach.
	\newblock Non-parametric stochastic approximation with large step sizes.
	\newblock \emph{The Annals of Statistics}, 2015.
	
	\bibitem[Dieuleveut et~al.(2016)Dieuleveut, Flammarion, and
	Bach]{DieuleveutFB16}
	A.~Dieuleveut, N.~Flammarion, and F.~R. Bach.
	\newblock Harder, better, faster, stronger convergence rates for least-squares
	regression.
	\newblock \emph{CoRR}, abs/1602.05419, 2016.
	
	\bibitem[Fabian(1973)]{Fabian:1973:AES}
	V.~Fabian.
	\newblock Asymptotically efficient stochastic approximation; the {RM} case.
	\newblock \emph{Annals of Statistics}, 1\penalty0 (3), 1973.
	
	\bibitem[Frostig et~al.(2015{\natexlab{a}})Frostig, Ge, Kakade, and
	Sidford]{FrostigGKS15b}
	R.~Frostig, R.~Ge, S.~Kakade, and A.~Sidford.
	\newblock Un-regularizing: approximate proximal point and faster stochastic
	algorithms for empirical risk minimization.
	\newblock In \emph{ICML}, 2015{\natexlab{a}}.
	
	\bibitem[Frostig et~al.(2015{\natexlab{b}})Frostig, Ge, Kakade, and
	Sidford]{FrostigGKS15}
	R.~Frostig, R.~Ge, S.~M. Kakade, and A.~Sidford.
	\newblock Competing with the empirical risk minimizer in a single pass.
	\newblock In \emph{COLT}, 2015{\natexlab{b}}.
	
	\bibitem[Ghadimi and Lan(2012)]{ghadimi2012optimal}
	S.~Ghadimi and G.~Lan.
	\newblock Optimal stochastic approximation algorithms for strongly convex
	stochastic composite optimization i: A generic algorithmic framework.
	\newblock \emph{SIAM Journal on Optimization}, 2012.
	
	\bibitem[Ghadimi and Lan(2013)]{ghadimi2013optimal}
	S.~Ghadimi and G.~Lan.
	\newblock Optimal stochastic approximation algorithms for strongly convex
	stochastic composite optimization, ii: shrinking procedures and optimal
	algorithms.
	\newblock \emph{SIAM Journal on Optimization}, 2013.
	
	\bibitem[Greenbaum(1989)]{Greenbaum89}
	A.~Greenbaum.
	\newblock Behavior of slightly perturbed lanczos and conjugate-gradient
	recurrences.
	\newblock \emph{Linear Algebra and its Applications}, 1989.
	
	\bibitem[Hestenes and Stiefel(1952)]{HestenesS52}
	M.~R. Hestenes and E.~Stiefel.
	\newblock Methods of conjuate gradients for solving linear systems.
	\newblock \emph{Journal of Research of the National Bureau of Standards}, 1952.
	
	\bibitem[Hsu et~al.(2014)Hsu, Kakade, and Zhang]{HsuKZ14}
	D.~J. Hsu, S.~M. Kakade, and T.~Zhang.
	\newblock Random design analysis of ridge regression.
	\newblock \emph{Foundations of Computational Mathematics}, 14\penalty0
	(3):\penalty0 569--600, 2014.
	
	\bibitem[Hu et~al.(2009)Hu, Kwok, and Pan]{HuKP09}
	C.~Hu, J.~T. Kwok, and W.~Pan.
	\newblock Accelerated gradient methods for stochastic optimization and online
	learning.
	\newblock In \emph{NIPS 22}, 2009.
	
	\bibitem[Jain et~al.(2016)Jain, Kakade, Kidambi, Netrapalli, and
	Sidford]{JainKKNS16}
	P.~Jain, S.~M. Kakade, R.~Kidambi, P.~Netrapalli, and A.~Sidford.
	\newblock Parallelizing stochastic approximation through mini-batching and
	tail-averaging.
	\newblock \emph{CoRR}, abs/1610.03774, 2016.
	
	\bibitem[Kushner and Clark(1978)]{KushnerClark}
	H.~J. Kushner and D.~S. Clark.
	\newblock \emph{Stochastic Approximation Methods for Constrained and
		Unconstrained Systems.}
	\newblock Springer-Verlag, 1978.
	
	\bibitem[Kushner and Yin(2003)]{KushnerY03}
	H.~J. Kushner and G.~Yin.
	\newblock Stochastic approximation and recursive algorithms and applications.
	\newblock \emph{Springer-Verlag}, 2003.
	
	\bibitem[Lan(2008)]{Lan08}
	G.~Lan.
	\newblock An optimal method for stochastic composite optimization.
	\newblock \emph{Tech. Report, IE, Georgia Tech.}, 2008.
	
	\bibitem[Lan and Zhou(2015)]{LanZ15}
	G.~Lan and Y.~Zhou.
	\newblock An optimal randomized incremental gradient method.
	\newblock \emph{CoRR}, abs/1507.02000, 2015.
	
	\bibitem[Lehmann and Casella(1998)]{lehmann1998theory}
	E.~L. Lehmann and G.~Casella.
	\newblock \emph{Theory of Point Estimation}.
	\newblock Springer Texts in Statistics. Springer, 1998.
	\newblock ISBN 9780387985022.
	
	\bibitem[Lin et~al.(2015)Lin, Mairal, and Harchaoui]{LinMH15}
	H.~Lin, J.~Mairal, and Z.~Harchaoui.
	\newblock A universal catalyst for first-order optimization.
	\newblock In \emph{NIPS}, 2015.
	
	\bibitem[Needell et~al.(2016)Needell, Srebro, and Ward]{NeedellSW16}
	D.~Needell, N.~Srebro, and R.~Ward.
	\newblock Stochastic gradient descent, weighted sampling, and the randomized
	kaczmarz algorithm.
	\newblock \emph{Mathematical Programming}, 2016.
	
	\bibitem[Nemirovsky and Yudin(1983)]{NemirovskyY83}
	A.~S. Nemirovsky and D.~B. Yudin.
	\newblock \emph{Problem Complexity and Method Efficiency in Optimization}.
	\newblock John Wiley, 1983.
	
	\bibitem[Nesterov(1983)]{Nesterov83}
	Y.~E. Nesterov.
	\newblock A method for unconstrained convex minimization problem with the rate
	of convergence ${O}(1/k^2)$.
	\newblock \emph{Doklady AN SSSR}, 269, 1983.
	
	\bibitem[Nesterov(2004)]{Nesterov04}
	Y.~E. Nesterov.
	\newblock \emph{Introductory lectures on convex optimization: A basic course},
	volume~87 of \emph{Applied Optimization}.
	\newblock Kluwer Academic Publishers, 2004.
	
	\bibitem[Nesterov(2012)]{Nesterov12}
	Y.~E. Nesterov.
	\newblock Efficiency of coordinate descent methods on huge-scale optimization
	problems.
	\newblock \emph{SIAM Journal on Optimization}, 22\penalty0 (2):\penalty0
	341--362, 2012.
	
	\bibitem[Paige(1971)]{Paige71}
	C.~C. Paige.
	\newblock The computation of eigenvalues and eigenvectors of very large sparse
	matrices.
	\newblock \emph{PhD Thesis, University of London}, 1971.
	
	\bibitem[Polyak(1964)]{Polyak64}
	B.~T. Polyak.
	\newblock Some methods of speeding up the convergence of iteration methods.
	\newblock \emph{USSR Computational Mathematics and Mathematical Physics}, 4,
	1964.
	
	\bibitem[Polyak(1987)]{Polyak87}
	B.~T. Polyak.
	\newblock \emph{Introduction to Optimization}.
	\newblock Optimization Software, 1987.
	
	\bibitem[Polyak and Juditsky(1992)]{PolyakJ92}
	B.~T. Polyak and A.~B. Juditsky.
	\newblock Acceleration of stochastic approximation by averaging.
	\newblock \emph{SIAM Journal on Control and Optimization}, volume 30, 1992.
	
	\bibitem[Proakis(1974)]{Proakis74}
	J.~G. Proakis.
	\newblock Channel identification for high speed digital communications.
	\newblock \emph{IEEE Transactions on Automatic Control}, 1974.
	
	\bibitem[Raginsky and Rakhlin(2011)]{RaginskyR11}
	M.~Raginsky and A.~Rakhlin.
	\newblock Information-based complexity, feedback and dynamics in convex
	programming.
	\newblock \emph{IEEE Transactions on Information Theory}, 2011.
	
	\bibitem[Robbins and Monro(1951)]{RobbinsM51}
	H.~Robbins and S.~Monro.
	\newblock A stochastic approximation method.
	\newblock \emph{The Annals of Mathematical Statistics}, vol. 22, 1951.
	
	\bibitem[Roy and Shynk(1990)]{RoyS90}
	S.~Roy and J.~J. Shynk.
	\newblock Analysis of the momentum lms algorithm.
	\newblock \emph{IEEE Transactions on Acoustics, Speech and Signal Processing},
	1990.
	
	\bibitem[Ruppert(1988)]{Ruppert88}
	D.~Ruppert.
	\newblock Efficient estimations from a slowly convergent robbins-monro process.
	\newblock \emph{Tech. Report, ORIE, Cornell University}, 1988.
	
	\bibitem[Shalev-Shwartz and Zhang(2014)]{ShwartzZ14}
	S.~Shalev-Shwartz and T.~Zhang.
	\newblock Accelerated proximal stochastic dual coordinate ascent for
	regularized loss minimization.
	\newblock In \emph{ICML}, 2014.
	
	\bibitem[Sharma et~al.(1998)Sharma, Sethares, and Bucklew]{SharmaSB98}
	R.~Sharma, W.~A. Sethares, and J.~A. Bucklew.
	\newblock Analysis of momentum adaptive filtering algorithms.
	\newblock \emph{IEEE Transactions on Signal Processing}, 1998.
	
	\bibitem[van~der Vaart(2000)]{Vaart00}
	A.~W. van~der Vaart.
	\newblock \emph{Asymptotic Statistics}.
	\newblock Cambridge University Publishers, 2000.
	
	\bibitem[Widrow and Stearns(1985)]{WidrowS85}
	B.~Widrow and S.~D. Stearns.
	\newblock \emph{Adaptive Signal Processing}.
	\newblock Englewood Cliffs, NJ: Prentice-Hall, 1985.
	
	\bibitem[Wilson et~al.(2016)Wilson, Recht, and Jordan]{WilsonRJ16}
	A.~C. Wilson, B.~Recht, and M.~I. Jordan.
	\newblock A lyapunov analysis of momentum methods in optimization.
	\newblock \emph{CoRR}, abs/1611.02635, 2016.
	
	\bibitem[Woodworth and Srebro(2016)]{WoodworthS16}
	B.~Woodworth and N.~Srebro.
	\newblock Tight complexity bounds for optimizing composite objectives.
	\newblock \emph{CoRR}, abs/1605.08003, 2016.
	
	\bibitem[Yuan et~al.(2016)Yuan, Ying, and Sayed]{YuanYS16}
	K.~Yuan, B.~Ying, and A.~H. Sayed.
	\newblock On the influence of momentum acceleration on online learning.
	\newblock \emph{Journal of Machine Learning Research (JMLR)}, volume 17, 2016.
	
\end{thebibliography}
\end{document}